\documentclass[12pt]{article}
\usepackage[normalem]{ulem}
\usepackage{setspace}
\usepackage{amsthm}
\usepackage{amsmath,amsfonts,amssymb}
\usepackage{natbib}
\usepackage{changepage}
\usepackage{graphicx}
\graphicspath{{./simulation/}}
\usepackage{booktabs}
\usepackage{subfig}
\usepackage{mwe}
\usepackage{rotating}
\usepackage{pdflscape}
\usepackage[utf8]{inputenc}
\usepackage[english]{babel}
\usepackage{multirow}
\usepackage{authblk}
\usepackage{algorithm,algcompatible}
\usepackage[noend]{algpseudocode}
\usepackage{makecell}
\usepackage{color}
\usepackage[table,xcdraw]{xcolor}
\usepackage{placeins}
\usepackage{float}
\usepackage{mathabx}
\usepackage{afterpage}
\usepackage[export]{adjustbox}
\usepackage{epstopdf}
\usepackage{tikz}
\usetikzlibrary{shapes}
\usetikzlibrary{positioning,arrows}
\usetikzlibrary{shapes,shapes.geometric,arrows,fit,calc,positioning,automata}
\usetikzlibrary{shapes,arrows,shadows}
\usepackage{pdfpages} 
\usepackage[toc,page,header]{appendix}

\usepackage{array}
\newcolumntype{L}[1]{>{\raggedright\let\newline\\\arraybackslash\hspace{0pt}}m{#1}}
\newcolumntype{C}[1]{>{\centering\let\newline\\\arraybackslash\hspace{0pt}}m{#1}}
\newcolumntype{R}[1]{>{\raggedleft\let\newline\\\arraybackslash\hspace{0pt}}m{#1}}

\hypersetup{
    linktocpage=true,
    breaklinks=true,
    }

\catcode`\^ = 13 \def^#1{\sp{#1}{}}

\definecolor{Pink}{rgb}{1.0, 0.5, 0.5}
\definecolor{Maroon}{rgb}{0.8, 0.0, 0.0}

\def\boxit#1{\vbox{\hrule\hbox{\vrule\kern6pt\vbox{\kern6pt#1\kern6pt}\kern6pt\vrule}\hrule}}

\newcommand{\B}{\mathbf}
\newcommand{\mB}{\mathcal{B}}
\newcommand{\mA}{\mathcal{A}}

\newcommand{\inner}[1]{\langle#1\rangle}
\newcommand{\norm}[1]{\|#1\|}
\newcommand{\abs}[1]{\vert #1 \vert}
\newcommand{\upperroman}[1]{\MakeUppercase{\romannumeral #1}}

\newtheorem{theorem}{Theorem}[section]
\newtheorem{lemma}[theorem]{Lemma}
\newtheorem{condition}{Condition}[section]

\newcommand{\bM}{\mbox{\bf M}}

\newcommand{\bu}{\mbox{\bf u}}
\newcommand{\bv}{\mbox{\bf v}}

\newcommand{\bX}{\mbox{\bf X}}

\newcommand{\bY}{\mbox{\bf Y}}

\newcommand{\bdelta}{\mbox{\boldmath $\delta$}}
\newcommand{\bDelta}{\mbox{\boldmath $\Delta$}}

\newcommand{\bGamma}{\mbox{\boldmath $\Gamma$}}

\newcommand{\bSig}{\mbox{\boldmath $\Sigma$}}

\newcommand{\what}{\widehat}

\newcommand{\sgn}{\mathrm{sgn}}

\def\t{^\top}

\def\beqn{\begin{eqnarray}}
\def\eeqn{\end{eqnarray}}
\def\beqns{\begin{eqnarray*}}
\def\eeqns{\end{eqnarray*}}

\def\0{{\bf 0}}

\def\C{{\bf C}}

\def\D{{\bf D}}
\def\d{{\bf d}}
\def\e{{\bf e}}
\def\E{{\bf E}}

\def\I{{\bf I}}

\def\t{{\bf t}}
\def\R{{\bf R}}

\def\Z{{\bf Z}}

\def\bP{{\bf P}}

\def\U{{\bf U}}
\def\V{{\bf V}}

\def\u{{\bf u}}
\def\v{{\bf v}}

\def\X{{\bf X}}
\def\x{{\bf x}}

\def\Y{{\bf Y}}
\def\y{{\bf y}}
\def\Z{{\bf Z}}

\def\1{{\bf 1}}

\def\trans{^{\rm T}}
\def\strans{^{*\rm T}}
\def\ztrans{^{0\rm T}}
\def\ttrans{^{t\rm T}}
\def\t1trans{^{t+1\rm T}}

\title{Statistically Guided Divide-and-Conquer for Sparse Factorization of Large Matrix\thanks{Authors are listed in alphabetical order.}}
\author[1]{Kun Chen\thanks{Correspondence: Kun Chen (kun.chen@uconn.edu) and Zemin Zheng (zhengzm@ustc.edu.cn)}}
\author[2]{Ruipeng Dong}
\author[1]{Wanwan Xu}
\author[2]{Zeming Zheng$^\dagger$}
\affil[1]{Department of Statistics, University of Connecticut, Storrs, CT}
\affil[2]{Department of Statistics \& Finance, University of Science and Technology of China, Anhui, China}
\date{}

\begin{document}

\maketitle
\begin{abstract}
\singlespacing
The sparse factorization of a large matrix is fundamental in modern statistical learning. In particular, the sparse singular value decomposition and its variants have been utilized in multivariate regression, factor analysis, biclustering, vector time series modeling, among others. The appeal of this factorization is owing to its power in discovering a highly-interpretable latent association network, either between samples and variables or between responses and predictors. However, many existing methods are either ad hoc without a general performance guarantee, or are computationally intensive, rendering them unsuitable for large-scale studies. We formulate the statistical problem as a sparse factor regression and tackle it with a divide-and-conquer approach. In the first stage of division, we consider both sequential and parallel approaches for simplifying the task into a set of co-sparse unit-rank estimation (CURE) problems, and establish the statistical underpinnings of these commonly-adopted and yet poorly understood deflation methods. In the second stage of division, we innovate a contended stagewise learning technique, consisting of a sequence of simple incremental updates, to efficiently trace out the whole solution paths of CURE. Our algorithm has a much lower computational complexity than alternating convex search, and the choice of the step size enables a flexible and principled tradeoff between statistical accuracy and computational efficiency. Our work is among the first to enable stagewise learning for non-convex problems, and the idea can be applicable in many multi-convex problems. Extensive simulation studies and an application in genetics demonstrate the effectiveness and scalability of our approach.

\end{abstract}

\clearpage
\doublespacing

\section{Introduction} \label{sec:intro}
Matrix factorization/decomposition is fundamental in many statistical learning techniques, including both unsupervised learning methods
such as principal component analysis and matrix completion, and supervised learning methods such as partial least squares and reduced-rank
regression. In particular, singular value decomposition (SVD) and its various generalizations and extensions have been applied or utilized
for matrix/tensor data dimensionality reduction, feature extraction, anomaly detection, data visualization, among others. Such tasks are
nowadays routinely encountered in various fields, and some modern applications include latent semantic analysis in natural language processing \citep{Landauer1998},
collaborative filtering in recommender systems \citep{Adoma2005}, bi-clustering analysis with high-throughput genomic data \citep{lee2010biclustering, vounou2010}, association network learning in
expression quantitative trait loci (eQTLs) mapping analysis \citep{uematsu2019sofar}, and efficient parameterization and stabilization in the training of deep neural
networks \citep{SunZheng2017}.


It is now well understood that classical SVD and its associated multivariate techniques have poor statistical properties in high dimensional
problems, as the noise accumulation due to a large number of irrelevant variables
and/or irrelevant latent subspace directions could disguise the signal of interest to the extent that the estimation becomes entirely distorted.
To break such ``curses of dimensionality'', structural low-dimensional assumptions such as low-rankness and sparsity are often imposed on the
matrix/tensor objects, and correspondingly regularization techniques are then called to the rescue. There is no free lunch though: matrix
factorization is often expensive in terms of computational complexity and memory usage, and for large-scale problems the computational burdens
emerging from regularized estimation become even more amplified. The required regularized estimation can be computationally intensive and
burdensome due to various reasons: the optimization often requires iterative SVD operations, and the data-driven estimation of the
regularization parameters often requires a spectrum of models with varying complexities being fitted.

To alleviate the computational burden of spare factorization of matrix, many related methods have been proposed as repeatedly fitting regularized rank-one models; see, e.g., \citet{witten2009}, \citet{lee2010biclustering}, and \citet{Ahn2015}. In particular, \citet{chen2012jrssb} proposed the exclusive extraction algorithm which isolates the estimation of each rank-one component by removing the signals of the other components from the response matrix based on an initial estimator of the matrix factorization. Alternatively, \citet{MishraDeyChen2017} suggested extracting unit-rank
factorization one by one in a sequential fashion, each time with the previously extracted components removed from the current response matrix. Both of them adapted the alternating convex search (ACS) method to solve the unit-rank estimation problems. Despite their proven effectiveness in various applications, the properties of these deflation strategies are not well understood in a general high-dimensional setup. Moreover, for each unit-rank problem, the optimization needs to be repeated for a grid of tuning parameters for locating the optimal factorization along the paths, and even with a warm-start strategy, it can still lead to skyrocketing computational costs especially for large-scale problems.

In parallel to the aforementioned learning scheme of ``regularization + optimization'' in solving the unit-rank problems, there has also been a revival of interest into the so-called \textit{stagewise learning} \citep{efron2004least,zhao2007,tibshirani2015stagewise}. Unlike regularized estimation, a stagewise procedure builds a model from scratch, and gradually increases the model complexity in a sequence of simple learning steps. For instance, in stagewise linear regression \citep{efron2004least,HastTibsFrie2008}, a \textit{forward step} searches for the best predictor to explain the current residual and updates its coefficient by a small increment (in magnitude), and a \textit{backward step} may decrease the coefficient (in magnitude) of a previously selected predictor to correct for greedy selection.  This process is repeated until a model with a desirable complexity level is reached or the model becomes excessively large and ceases to be identifiable. Thus far, fast stagewise learning algorithms and frameworks have been developed mainly for convex problems, which are not applicable for non-convex problems and are in general lacking of theoretical support; a comprehensive review of stagewise learning can be found in \citet{tibshirani2015stagewise}, in which its connections and differences with various optimization and machine learning approaches such as steepest descend, boosting, and path-following algorithms were discussed.

In this paper, we formulate the statistical problem of sparse factorization of large matrix as a sparse factor regression and tackle it with a novel two-stage divide-and-conquer approach. In the first stage of ``division'', we consider both sequential and parallel approaches for simplifying the task into a set of co-sparse unit-rank estimation (CURE) problems, and establish the statistical underpinnings of these commonly-adopted and yet poorly understood deflation methods. In the second stage of ``division'', we innovate a contended stagewise learning technique, consisting of a sequence of simple incremental updates, to efficiently trace out the whole solution paths of CURE. Our algorithm has a much lower computational complexity than ACS, and the choice of the step size enables a flexible and principled tradeoff between statistical accuracy and computational efficiency. Our work is among the first to enable stagewise learning for non-convex problems, and the idea can be applicable in many bi-convex problems including sparse factorization of a tensor. Extensive simulation studies and an application in genetics demonstrate the effectiveness and scalability of our approach.

The rest of the paper is organized as follows. In Section \ref{sec:ssvd}, we present a general statistical setup with a sparse SVD formulation, and present the deflation strategies for reducing the general problem into a set of simpler CURE problems with comprehensive theoretical properties. Section \ref{sec:unitrank} provides detailed derivations of the stagewise estimation algorithm for tracing out the solution paths of CURE. Algorithmic convergence and computational complexity analysis are also shown in Section \ref{sec:unitrank}. In Section \ref{sec:simulation}, extensive simulation studies demonstrate the effectiveness and scalability of our approach. An application in genetics is presented in Section \ref{sec:yeast}. Section \ref{sec6} concludes with possible future work. All technical details are relegated to the Supplementary Materials.

\section{A Statistical Framework for Sparse Factorization}\label{sec:ssvd}



Throughout the paper, we use bold uppercase letter to denote matrix and bold lowercase letter to denote column vector. For any matrix $\bM=(m_{ij})$, denote by $\|\bM\|_F=\bigl(\sum_{i,j}m_{ij}^2\bigr)^{1/2}$, $\|\bM\|_1=\sum_{i,j}|m_{ij}|$, and $\|\bM\|_{\infty}=\max_{i,j}|m_{ij}|$ the Frobenius norm, entrywise $\ell_1$-norm, and entrywise $\ell_{\infty}$-norm, respectively. We also denote by $\|\cdot\|_2$ the induced matrix norm (operator norm), which becomes the $\ell_2$ norm for vectors.






\subsection{Co-sparse Factor Regression}

We briefly review the co-sparse factor regression model proposed by \citet{MishraDeyChen2017}, with which arises the problem of sparse SVD recovery.
Consider the multivariate linear regression model,
\begin{align*}
\y = \C^*\trans\x + \e,
\end{align*}
where $\y\in \mathbb{R}^q$ is the multivariate response vector, $\x \in \mathbb{R}^p$ is the multivariate predictor vector,
$\C^* \in \mathbb{R}^{p\times q}$ is the coefficient matrix, and $\e\in\mathbb{R}^{q}$ is the random error vector of zero mean.
Without loss of generality, suppose $\mbox{E}(\x)=\0$ and hence $\mbox{E}(\y) = \0$.
As in reduced-rank regression \citep{anderson1951, reinsel1998}, suppose the coefficient matrix $\C^*$ is of low-rank,
i.e., $\mbox{rank}(\C^*) = r^*$, $r^*\leq \min(p,q)$.

With a proper parameterization, the model possesses a supervised factor analysis
interpretation. To see this, denote $\mbox{cov}(\x) = \bGamma$, and write $\C^*=\U^*\widetilde{\V}\trans$ for some $\U^*\in\mathbb{R}^{p\times r^*}$ and $\widetilde{\V}\in\mathbb{R}^{q\times r^*}$.
Then $\U^*\trans\x$ can be interpreted as a set of ``supervised latent factors'', as some linear combinations of the components
of $\x$; to avoid redundancy, it is desirable to make these factors uncorrelated, which leads to the constraint that
$\mbox{cov}(\U^*\trans\x)=\U^*\trans\bGamma\U^*=\I_{r^*}$. Similar to the factor analysis, we let $\widetilde{\V} = \V^*\D^*$ with
$\V^*\trans\V^*=\I$, to ensure parameter identifiability. This leads to
\begin{align}
\C^*=\U^*\D^*\V^*\trans,\qquad \mbox{ s.t. } \U^*\trans\bGamma\U^*=\I_{r^*} = \V^*\trans\V^* =\I_{r^*}.\label{eq:decomposition1}
\end{align}
When both $\U^*$ and $\V^*$ assumed to be sparse, we say that $\C^*$ admits a (generalized) \textit{sparse SVD structure}.
The interpretation is very appealing in statistical modeling: the $q$ outcomes in $\y$ are related to the $p$ features in $\x$ only
through $r^*$ pathways, each of which may only involve small subsets of the outcomes and the features.

With $n$ independent observations, let $\Y=[\widetilde{\y}_1,\ldots,\widetilde{\y}_q]=[\y_1,\ldots,\y_n]\trans\in\mathbb{R}^{n\times q}$
be the response matrix, and $\X=[\widetilde{\x}_1,\ldots,\widetilde{\x}_{p}]=[\x_1,\ldots,\x_n]\trans\in \mathbb{R}^{n\times p}$ be
the predictor matrix. The sample-version of the multivariate regression model is written as
\begin{equation}\label{eq:samplemodelC}
\Y=\X\C^*+\E,
\end{equation}
where $\E \in \mathbb{R}^{n\times q}$ is the random error matrix. In addition, we assume that each column of $\X$ is normalized to have $\ell_2$ norm $\sqrt{n}$. The sparse SVD representation of $\C^*$ becomes
\begin{align}
\C^*=\U^*\D^*\V^*\trans,\qquad \mbox{ s.t. } (\frac{1}{\sqrt{n}}\X\U^*)\trans(\frac{1}{\sqrt{n}}\X\U^*) = \V^*\trans\V^* =\I_{r^*},
\label{eq:decomposition}
\end{align}
where $\X\trans\X/n$ replaces its population counterpart $\bGamma$.

For convenience, for $k=1,\dots,r^*$, we denote $\C_k^* = d_k^*\u_k^*\v_k\strans$, where $\u_k^*,\v_k^*$ are the $k$th column of $\U^*,\V^*$ respectively and
$d_k^*$ is the $k$th diagonal element of $\D^*$. In addition, let ${\bf P} = n^{-1}\X\trans\X$, which is a sample-version covariance matrix of
$\x$. Then we call $\C_k^*$ as the $k$th layer $\bf P$-orthogonal SVD of $\C^*$ \citep{Zheng2019}. Notice that when setting $\X$ as the $n\times n$ identity matrix, the above general model subsumes the matrix approximation problem in the unsupervised learning scenario. Further, if $\Y$ is partially observed, the model connects to the matrix completion problem and the trace regression \citep{Hastie2015}.

Our main focus is then the recovery of the $\bf P$-orthogonal sparse SVD of $\C^*$ under the co-sparse factor model as specified in \eqref{eq:samplemodelC} and \eqref{eq:decomposition}. We stress that we desire the elements of the matrix factorization, e.g., the left and right singular vectors, to be entrywisely sparse, rather than only making the matrix itself to be sparse. A matrix with sparse singular vectors could be sparse itself, but not vice versa in general. The many appealing features of the sparse SVD factorization have been well demonstrated in various scientific applications \citep{vounou2010,ChenChanStenseth2014, Ma15675}.

\subsection{Divide-and-Conquer through Unit-Rank Deflation}\label{sec:deflation}

The simultaneous presence of the low rank and the co-sparsity structure makes the sparse SVD recovery challenging. A joint estimation of all the sparse singular vectors may necessarily involve identifiability constraints such as orthogonality \citep{vounou2010,uematsu2019sofar}, which makes the optimization computationally intensive. Motivated by the power method for computing SVD \citep{Golub1996}, we take a divide-and-conquer deflation approach to tackle the problem. The main idea is to estimate the unit-rank components of $\X\C^*$ one by one, thus reducing the problem into a set of much simpler unit-rank problems.


We first present the ``division'' that we desire for the estimation of the co-sparse factor model in \eqref{eq:samplemodelC} and \eqref{eq:decomposition}. That is, we aim to simplify the multi-rank factorization problem into a set of unit-rank problems of the following form:
\begin{align}
    \min_{\C \in \mathbb{R}^{p \times q}} \left\{L(\C; \mathcal{T}_n) + \rho(\C;\lambda)\right\},\quad \mbox{s.t. rank}(\C)\leq 1,\label{eq:unitoptC}
\end{align}
where $\mathcal{T}_n = \{(\y_i,\x_i); i = 1,\ldots, n\}$ denotes the observed data, $L(\C; \mathcal{T}_n)=(2n)^{-1}\norm{\Y - \X\C}_{F}^{2}$ is the sum of squares loss function, and $\rho(\C;\lambda) = \lambda \|\C\|_1$ is the $\ell_1$ penalty term with tuning parameter $\lambda \geq 0$. Intuitively, this criterion targets on the best sparse unit-rank approximation of $\Y$ in the column space of $\X$. Intriguingly, due to the unit-rank constraint, the problem can be equivalently expressed as a \textit{co-sparse unit-rank regression} (CURE) \citep{MishraDeyChen2017},
\begin{align}
\begin{split}
\min_{d,\u,\v} & \left\{L(d\u\v\trans; \mathcal{T}_n) + \lambda \|d\u\v\trans\|_1\right\},\qquad \mbox{s.t.}\,
d\geq 0, n^{-1/2}\|\X\u\|_2=1, \|\v\|_2 = 1. \label{eq:unitopt}
\end{split}
\end{align}
The penalty term is multiplicative in that $\|\C\|_1 = \|d\u\v\trans\|_1 = d \|\u\|_1\|\v\|_1$, conveniently producing a co-sparse factorization. This is not really surprising, because the sparsity of a unit-rank matrix directly leads to the sparsity in both its left and right singular vectors. Therefore, formulating the problem in term of unit-rank matrix enables CURE to use a single penalty term to achieve co-sparse factorization.

We now consider two general deflation approaches, in order to reach the desired simplification in \eqref{eq:unitoptC}, or equivalently, in \eqref{eq:unitopt}. The first approach is termed ``\textit{parallel pursuit}'', which is motivated by the exclusive extraction algorithm proposed in \citet{chen2012jrssb}. This approach requires an initial estimator of the elements of the matrix factorization of $\C^*$; the method then isolates the estimation of each rank-one component $\C_k^*$ by removing from the response matrix $\Y$ the signals of the other components based on the initial estimator. To be specific, let $\widetilde{\C} = \sum_{k=1}^{r}\widetilde{\C}_k = \sum_{k=1}^{r}\widetilde{d}_k \widetilde{\u}_k \widetilde{\v}_k\trans$ be the initial estimator of $\C^*$ and its $\bP$-orthogonal SVD as in \eqref{eq:decomposition}. Then parallel pursuit solves the problems
\begin{equation}\label{eq:parallel}
\widehat{\C}_k  = \arg\min_{\C} \left\{L(\C + \sum_{j\neq k}^{r}\widetilde{\C}_j;
\mathcal{T}_n) + \rho(\C;\lambda)\right\},\qquad \mbox{s.t. rank}(\C)\leq 1,
\end{equation}
for $k = 1,\ldots, r$, which can be implemented in parallel. See Figure \ref{figs:diagram1} for an illustration.

\tikzstyle{block} = [draw,minimum size=2em]
\def\radius{.7mm}
\tikzstyle{branch}=[fill,shape=circle,minimum size=3pt,inner sep=0pt]

\begin{figure}[htp]
\centering
\begin{tikzpicture}[remember picture,scale=1, every node/.style={scale=0.8},every block/.style={scale=0.8}]
    \foreach \y in {1,2,3} {
        \node at (0,-\y*3) (input\y) {$\widetilde{d}_\y\widetilde{\B{u}}_\y\widetilde{\B{v}}_\y\trans$};
        \node[block] at (5,-\y*3) (block\y) {$\B{Y}-\B{X}\sum_{k\neq\y}\widetilde{d}_k\widetilde{\B{u}}_k\widetilde{\B{v}}_k\trans$};

        \node at (10,-\y*3) (output\y){$\widehat{d}_\y\widehat{\B{u}}_\y\widehat{\B{v}}_\y\trans$};

        \draw [->] (block\y) -- node [name=u] [above] {CURE on $\B{X}$} (output\y);
    }

    \foreach \x in {2,3} {
            \draw[->] (input\x) -- (block1);
    }
    \foreach \x in {1,3} {
            \draw[->] (input\x) -- (block2);
    }
    \foreach \x in {1,2} {
            \draw[->] (input\x) -- (block3);
    }


    \path (input1) -- coordinate (branch) (block1);

\end{tikzpicture}
\caption{An illustration of the parallel pursuit.}\label{figs:diagram1}
\end{figure}
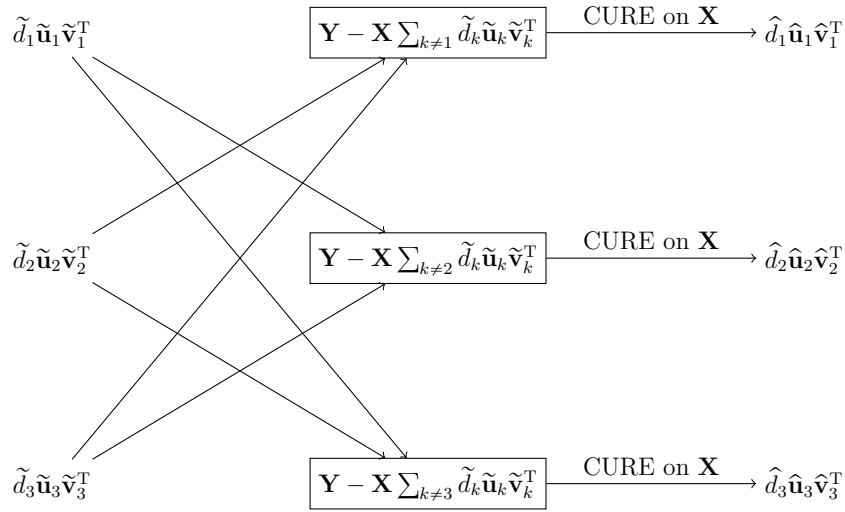

The second deflation approach is ``\textit{sequential pursuit}'', which was proposed by \citet{MishraDeyChen2017}. This method extracts unit-rank
factorization one by one in a sequential fashion, each time with the previously extracted components removed from the current response matrix. Specifically, the procedure sequentially solves
\begin{equation}\label{eq:sequential}
\widehat{\C}_k = \arg\min_{\C} \left\{L(\C + \sum_{j=0}^{k-1}\widehat{\C}_j; \mathcal{T}_n)+  \rho(\C;\lambda)\right\},
\qquad \mbox{s.t. rank}(\C)\leq 1,
\end{equation}
for $k = 1, \ldots r$, where $\widehat{\C}_0 = 0$, and $\widehat{\C}_k =\widehat{d}_k\widehat{\u}_k\widehat{\v}_k$ is the selected unit-rank solution (through tuning) in the $k$th step. See Figure \ref{figs:diagram2} for an illustration.

\tikzstyle{block} = [draw,minimum size=2em]
\def\radius{.7mm}
\tikzstyle{branch}=[fill,shape=circle,minimum size=3pt,inner sep=0pt]

\begin{figure}[htp]
\centering
\begin{tikzpicture}[remember picture,scale=1, every node/.style={scale=0.8},every block/.style={scale=0.8}]

        \node at (0,-3) (input1) {$\widehat{d}_0\widehat{\B{u}}_0\widehat{\B{v}}_0\trans=\0$};
        \node[block] at (5,-3) (block1) {$\bY_{1}=\B{Y}-\B{X}\widehat{d}_{0}\widehat{\B{u}}_{0}\widehat{\B{v}}_{0}\trans$};
        \node at (0,-6) (input2) {};
        \node[block] at (5,-6) (block2) {$\bY_{2}=\B{Y}_{1}-\B{X}\widehat{d}_{1}\widehat{\B{u}}_{1}\widehat{\B{v}}_{1}\trans$};
        \node at (0,-9) (input3) {};
        \node[block] at (5,-9) (block3) {$\bY_{3}=\B{Y}_{2}-\B{X}\widehat{d}_{2}\widehat{\B{u}}_{2}\widehat{\B{v}}_{2}\trans$};
    \draw[->] (input1) -- (block1);

    \foreach \y in {1,2,3} {

        \node at (10,-\y*3) (output\y){$\widehat{d}_\y\widehat{\B{u}}_\y\widehat{\B{v}}_\y\trans$};

        \draw [->] (block\y) -- node [name=u] [above] {CURE on $\B{X}$} (output\y);
    }

    \foreach \x in {1} {
            \draw[->] (output\x) -- (block2);
    }
    \foreach \x in {2} {
            \draw[->] (output\x) -- (block3);
    }



    \path (input1) -- coordinate (branch) (block1);

\end{tikzpicture}
\caption{An illustration of the sequential pursuit.}\label{figs:diagram2}
\end{figure}
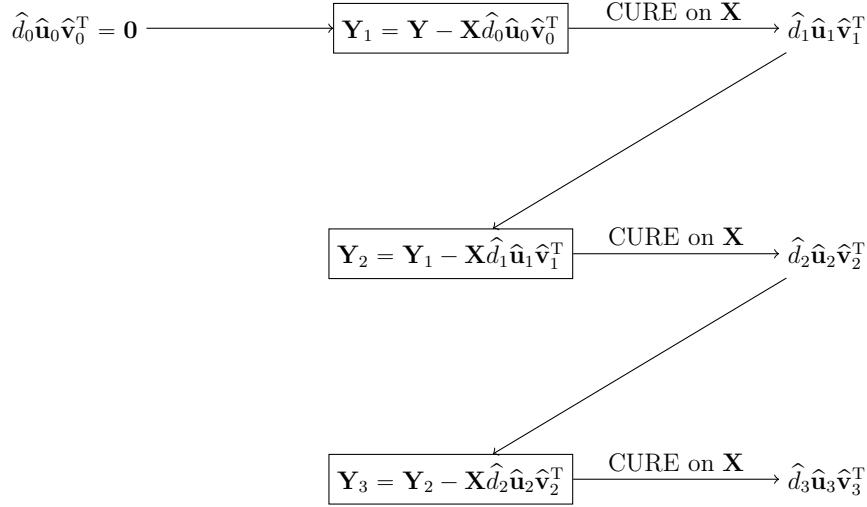

Both deflation approaches are intuitive. The parallel pursuit uses some initial estimator to isolate the estimation of the unit-rank components, and refines the estimation of each component around the vicinity of its initial estimator. Sequential pursuit, on the other hand, keeps extracting unit-rank components from the current residuals until no signal is left; the idea is related to a well-known fact derived from Eckart-Young Theorem \citep{eckart1936} in matrix approximation and reduced-rank regression, that is, the solution of $\min_{\C}\|\Y-\X\C\|_F^2$ s.t. $\mbox{rank}(\C)\leq r$ can be exactly recovered by sequentially fitting $r$ unit-rank problems with the residuals from the previous step. But apparently, such exact correspondence would break down with additional sparse regularization.


Indeed, many matrix decomposition related methods have been proposed as repeatedly fitting rank-one models; see, e.g., \citet{witten2009}, \citet{lee2010biclustering}, and \citet{Ahn2015}. Despite their proven effectiveness, many properties of these deflation strategies are yet to be explored. For example, what is the role of the initial estimator in parallel pursuit? Can parallel pursuit lead to substantially improved statistical performance (in terms of lowered error bound) comparing to that of the initial estimator? If Eckart-Young Theorem no longer applies, can sequential pursuit still work in the presence of regularization? What is the impact of ``noise accumulation'' since the subsequent estimation depends on the previous steps? How do the two deflation methods compare? We attempt to fill these gaps by performing a rigorous theoretical analysis, to be presented in the next section.


\subsection{Unveiling the Mystery of Deflation}

We now analyze the statistical properties of the deflation-based estimators of the sparse SVD components, i.e., $\widehat{\C}_k$, $k=1,\ldots, r$, from either the parallel pursuit in \eqref{eq:parallel} or the sequential pursuit in \eqref{eq:sequential}. Our treatment on the computation and the tuning of the CURE problem in \eqref{eq:unitoptC} will be deferred to Section \ref{sec:unitrank}, as the second part of our divide-and-conquer strategy. For now, we first concern theoretical analysis assuming CURE is solved globally, and then generalize our results to any computable local optimizer of CURE.




\subsubsection{Technical Conditions}

For any $p$ by $q$ nonzero matrix $\bDelta$, denote by $\bDelta_J$ the corresponding matrix of the same dimension that keeps the entries of $\bDelta$ with indices in set $J$ while sets the others to be zero. We need the following regularity conditions.

\begin{condition}\label{cond2}
The random noise vectors are independently and identically distributed as $\e_i \sim N(\boldsymbol{0}, \bSig)$, $i = 1,\ldots, n$.  Denote the $j$th diagonal entry of $\bSig$ as $\sigma_{j}^{2}$; we assume $\sigma_{\max}^{2} = \max_{1\le j \le p}\sigma_{j}^{2}$ is bounded from above. 
\end{condition}

\begin{condition}\label{cond3}
For $1 \leq k \leq r^{*}$, the gaps between the successive singular values are positive, i.e., $\delta_{k}^* = d_{k}^* - d_{k+1}^* > 0$, where $d_{k}^* > 0$ is the $k$th singular value of $n^{-1/2}\X\C^{*}$.
\end{condition}

\begin{condition}\label{cond1}
There exists certain sparsity level $s$ with a positive constant $\rho_{l}$ such that 
\begin{align*}
  \inf_{\bDelta} \left\{\frac{\norm{\X\bDelta}_{F}^{2}}{n(\norm{\bDelta_{J}}_{F}^{2}\vee\norm{\bDelta_{J_{s}^{c}}}_{F}^{2})}:
  \abs{J} \le s, ~
  \norm{\bDelta_{J^{c}}}_{1} \le 3 \norm{\bDelta_{J}}_{1}\right\} \geq \rho_{l},
\end{align*}
where $\bDelta_{J_{s}^{c}}$ is formed by keeping the $s$ entries of $\bDelta_{J^{c}}$ with largest absolute values and setting the others to be zero. 
\end{condition}

\begin{condition}\label{cond4}
There exists certain sparsity level $s$ with a positive constant $\phi_{u}$ such that
  \begin{equation*}
    \sup_{\bDelta}\left\{\frac{\abs{J}^{1/2}\norm{\X\trans\X\bDelta}_{\max}}{n\norm{\bDelta_{J}}_{F}}: \abs{J} \le s,~
    \norm{\bDelta_{J^{c}}}_{1} \le 3\norm{\bDelta_{J}}_{1}\right\} \le \phi_{u}. 
  \end{equation*}
\end{condition}



Condition \ref{cond2} assumes the random errors in model \eqref{eq:samplemodelC} are Gaussian for simplicity. In fact, our technical argument still applies as long as the tail probability bound of the noise decays exponentially; see the inequality \eqref{gaus} in the proof of Lemma \ref{se_lemma2}. Condition \ref{cond3} is imposed to ensure the identifiability of the $r^*$ latent factors (singular vectors). Otherwise, the targeted unit rank matrices would not be distinguishable. Similar assumptions can be found in \citet{MishraDeyChen2017} , \citet{Zheng2017}, and \citet{Zheng2019}, among others.



Condition \ref{cond1} is a matrix version of the restricted eigenvalue (RE) condition proposed in \citet{bickel2009re}, which is typically imposed in $\ell_1$-penalization to restrict the correlations between the columns of $\X$ within certain sparsity level, thus guaranteeing the identifiability of the true regression coefficients. The only difference is that we consider the estimation of matrices here instead of vectors. Since the Frobenius norm of a matrix can be regarded as the $\ell_2$-norm of the stacked vector consisting of the columns of the matrix, Condition \ref{cond1} is equivalent to the RE condition in the univariate response setting. Similarly, Condition \ref{cond4} is a matrix version of the cone invertibility factor \citep{ye2010rate} type of condition that allows us to control the entrywise estimation error. The integer $s$ in both conditions acts as a theoretical upper bound on the sparsity level of the true coefficient matrices, the requirement on which will be shown to be different in the two deflation approaches.

\subsubsection{Main Results}
Now we are ready to present the main results. Denote by $s_k = \norm{\C_k^*}_0$ for $1 \leq k \leq r^*$ and $s_0=\norm{\C^*}_0$. The following theorem characterizes the estimation accuracy in different layers of the sequential pursuit.

\begin{theorem}[Convergence rates of the sequential pursuit]\label{se_th}
  Suppose Conditions \ref{cond2}--\ref{cond4} hold with the sparsity level $s \geq \max_{1\leq k\leq r^*}s_k$. Choose $\lambda_{1} = 2\sigma_{\max} \sqrt{2\alpha\log(pq)/n}$ for some constant $\alpha > 1$ and $\lambda_k = \Pi_{\ell=1}^{k-1}(1+\eta_{\ell})\lambda_1$. 
  The following results hold uniformly over $1 \leq k \leq r^*$ with probability at least $1 - (pq)^{1-\alpha}$,
  \begin{align*}
   &\norm{\widehat{\bDelta}_{k}}_{F} = O(\theta_{k}\sqrt{s_k}\lambda_k) = O(\gamma_{k}\theta_k\sqrt{s_k\log(pq)/n}),\\ %
   &\norm{\widehat{\bDelta}_{k}}_{1} = O(\theta_{k}s_k \lambda_k)  = O(\gamma_{k}\theta_ks_k\sqrt{\log(pq)/n}), %
  \end{align*}
  where $\theta_k = d_{k}^*/\delta_{k}^*$, $\eta_k = c\theta_k$ with constant $c = 24\phi_u\rho_{l}^{-1}$, and $\gamma_{k}=\Pi_{\ell=1}^{k-1}(1+\eta_{\ell})$ with $\gamma_1 = 1$.
\end{theorem}

Theorem \ref{se_th} presents the estimation error bounds in terms of both Frobenius and $\ell_1$ norms for the co-sparse unit rank matrix estimators in sequential pursuit under mild and reasonable conditions. This is nontrivial since the subsequent layers play the role of extra noises in the estimation of each unit-rank matrix. We address this issue by showing that its impact is secondary due to the orthogonality between different layers, as demonstrated in Lemma \ref{se_lemma1}. The results hold uniformly over all true layers with a significant probability that approaches one in polynomial orders of the product of the two dimensions $p$ and $q$.

The regularization parameter $\lambda_k$ reflects the minimum penalization level needed to suppress the noise in the $k$th layer, which consists of two parts. One is from the random noise accompanied with the original response variables, while the other is due to the accumulation of the estimation errors from all previous layers. The latter is of a larger magnitude than the former one such that the penalization level increases almost exponentially with $k$. This can be inevitable in the sequential procedure since the $k$th layer estimator is based on the residual response matrix after extracting the previous layers. Hopefully, the estimation consistency is generally guaranteed for all the significant unit-rank matrices as long as the true regression coefficient matrix $\C^*$ is of sufficiently low rank. 

When the singular values are well separated, for the first few layers, the estimation accuracy is about $O(\sqrt{s_k\log(pq)/n})$, which is close to the optimal rate $O(\sqrt{(s_u+s_v)\log(p\vee q)/n})$ for the estimation of $\C^*$ established in \citet{ma2014adaptive}. When the true rank is one, the main difference between the two rates lies in the sparsity factors, where $(s_u + s_v)$ is a sum of the sparsity levels of the left and right singular vectors, while our sparsity factor is the product of them. However, the optimal rate is typically attained through some nonconvex algorithms \citep{ma2014adaptive,yu2018recovery,uematsu2019sofar} that search the optimal solution in a neighborhood of $\C^*$. In contrast, our CURE algorithm enjoys better computational efficiency and stability.

Last but not least, in view of the correlation constraint on the design matrix $\X$ ($s \geq \max_{1\leq k\leq r^*}s_k$), the sequential pursuit is valid as long as the number of nonzero entries in each layer $\C_k^*$ is within the sparsity level $s$ imposed in Conditions \ref{cond1} and \ref{cond4}. It means that in practice, the sequential deflation strategy can recover some complex networks layer by layer, which is not shared by either the parallel pursuit or other methods that directly target on estimating the multi-rank coefficient matrix.

\smallskip

We then turn our attention to the statistical properties of the parallel pursuit. The key point of this deflation strategy is to find a relatively accurate initial estimator such that the signals from the other layers/components except the targeted one can be approximately removed. The reduced-rank regression estimator may not be a good choice under high dimensions, since it does not guarantee the estimation consistency due to the lack of sparsity constraint. Similar to \citet{uematsu2019sofar}, we mainly adopt the following lasso initial estimator,
\begin{equation*}
  \widetilde{\C} = \underset{\C}{\arg\min}~ (2n)^{-1}\norm{\Y - \X\C}_{F}^{2} + \lambda_{0}\norm{\C}_{1},
\end{equation*}
which can be efficiently solved by various algorithms \citep{friedman2010regularization}. Under the same RE condition (Condition \ref{cond1}) as the sequential pursuit on the design matrix $\X$ with a tolerated sparsity level $s \geq s_0$, where $s_0$ indicates the number of nonzero entries in $\C^*$, we show that this lasso initial estimator is consistent with the convergence rate of $O(\sqrt{s_0\log(pq)/n})$ in Lemma \ref{para_th_init}.

Based on $\widetilde{\C}$, we can obtain the initial unit rank estimates $\widetilde{\C}_{k}^0$ for different layers through a $\bP$-orthogonal SVD. However, these unit rank estimates are not guaranteed to be sparse even if the lasso estimator $\widetilde{\C}$ is a sparse one, which causes additional difficulties in high dimensions. Thus, to facilitate the theoretical analysis, our initial $k$th layer estimate $\widetilde{\C}_{k}$ takes the $s$ largest components of $\widetilde{\C}_{k}^0$ in terms of absolute values while sets the others to be zero, where $s$ is the upper bound on the sparsity level defined in Conditions \ref{cond1} and \ref{cond4}. We will show that these $s$-sparse initial estimates $\widetilde{\C}_{k}$ can maintain the estimation accuracy of the initial unit-rank estimates $\widetilde{\C}_{k}^0$ without imposing any signal assumption, regardless of the thresholding procedure. In practice, we can use the SVD estimates directly as the impact of fairly small entries is negligible.



\begin{theorem}[Convergence rates of the parallel pursuit]\label{para_th}
Suppose Conditions \ref{cond2}--\ref{cond4} hold with the sparsity level $s \geq \max_{0\leq k\leq r^*}s_k$. Choose $\lambda_{0} = 2\sigma_{\max} \sqrt{2\alpha\log(pq)/n}$ with constant $\alpha > 1$ for the initial Lasso estimator $\widetilde{\C}$, and $\lambda_{k} = 2C\psi_k\sqrt{\log(pq)/n}$ for some positive constant $C$. The following results hold uniformly over $1 \leq k \leq r^*$ with probability at least $1 - (pq)^{1-\alpha}$,
  \begin{equation*}
  \begin{aligned}
    &\norm{\widehat{\bDelta}_{k}}_{F} = O(\sqrt{s_k}\lambda_{k}) = O(\psi_k\sqrt{s_{k}\log(pq)/n}), \\
    &\norm{\widehat{\bDelta}_{k}}_{1} = O(s_k \lambda_{k}) = O(\psi_k s_{k}\sqrt{\log(pq)/n}),
  \end{aligned}
  \end{equation*}
  where $\psi_k = d_1^*d_c^*/(d_k^*\min[\delta_{k-1}^*, \delta_{k}^*])$ with $d_c^*$ the largest singular value of $\C^*$.
\end{theorem}
Based on the lasso initial estimator, Theorem \ref{para_th} demonstrates that the parallel pursuit achieves accurate estimation of different layers and reduces the sparsity factor in the convergence rates from $s_0$ to $s_k$, where the sparsity level $s_0$ of the entire regression coefficient matrix can be about $r^*$ times of $s_k$. Moreover, as there is no accumulated noises in the parallel estimation of different layers, we do not see an accumulation factor $\gamma_{k}$ here such that the penalization parameters keep around a uniform magnitude. The eigen-factor $\psi_k$ plays a similar role as $\theta_k$ in the convergence rates for sequential pursuit in Theorem \ref{se_th}, both of which can be bounded from above under the low-rank and sparse structures. Thus, the estimation accuracy of all the significant unit-rank matrices is about the same as that of the first layer in the sequential pursuit, which reveals the potential superiority of the parallel pursuit. 

In view of the technical assumptions, the sequential and the parallel pursuits are very similar, while the main difference lies in the requirement on the sparsity level $s$. As discussed after Condition \ref{cond1}, the integer $s$ constrains the correlations between the columns of $\X$ to ensure the identifiability of the true supports, thus can be regarded as fixed for a given design matrix. Since the parallel pursuit requires not only the sparsity level $s_k$ of each individual layer but also the overall sparsity level $s_0$ to be no larger than $s$, it puts a stricter constraint on correlations among the predictors. In other words, when the true coefficient matrix $\C^*$ is not sufficiently sparse, the lasso initial estimator may not be accurate enough to facilitate the subsequent parallel estimation, in which case the sequential pursuit could enjoy some advantage.




The previous two theorems guarantee the statistical properties of the global optimizers of the CURE problems from the two different deflation strategies, respectively. But it is worth pointing out that although the CURE problem adopts a bi-convex form, it is generally not convex such that the computational solution is usually a local optimizer rather than a global one. Fortunately, it has been proved in univariate response settings that any two sparse local optimizers can be close to each other under mild regularity conditions \citep{zhang2012general,fan2013}. Therefore, when the global minimizer is also sparse, the computational solution will share similar asymptotic properties. We illustrate this phenomenon in multi-response settings through the following theorem.

\begin{theorem}\label{partial_vs_global_optimum_th}
  For any $k$, $1\le k \le r^{*}$, let $\widehat{\C}_{k}^L$ be a computable local optimizer which satisfies $\norm{\widehat{\C}_{k}^L}_{0}=O(s_{k})$
  and $n^{-1}\norm{\X\trans(\Y_{k}-\X\widehat{\C}_{k}^L)}_{\max}=O(\lambda_{k})$, $\min_{\norm{\boldsymbol{\gamma}}_{2}=1,\norm{\boldsymbol{\gamma}}_{0}\le Cs_k} n^{-1/2}\norm{\X\boldsymbol{\gamma}}_{2} \ge \kappa_0$ for some positive constant $\kappa_0$ and sufficiently large positive constant $C$, and $\norm{\widehat{\C}_{k}}_{0}=O(s_{k})$. Then under the same assumptions of Theorem \ref{se_th} (Theorem \ref{para_th}), $\widehat{\C}_{k}^L$ achieves the same estimation error bounds as the global optimizer. 
\end{theorem}


The regularity conditions on the computable local optimizer and the design matrix are basically the same as those in \citet{fan2013}. The assumption on the sparsity of the global optimizer is generally easy to satisfy as we adopt the $\ell_1$-penalization, which was shown in \citet{bickel2009re} to generate a sparse model of size $O(\phi^{\max} s)$ under certain regularity conditions, where $\phi^{\max}$ is the largest eigenvalue of the Gram matrix and $s$ is the true sparsity level. In view of Theorem \ref{partial_vs_global_optimum_th}, the sparse computational solutions we obtain in practice can also satisfy the desirable statistical properties.


\section{Divide-and-Conquer through Stagewise Learning} \label{sec:unitrank}



\subsection{Contended Stagewise Learning for CURE}\label{sec:stagecure}

It remains to solve the CURE problem in \eqref{eq:unitoptC} or \eqref{eq:unitopt}. The alternating convex
search (ACS) algorithm, i.e., block coordinate descent \citep{minasian2014energy},  is natural and commonly-used for solving (\ref{eq:unitopt}), in which
the objective function is alternately optimized with respect to a (overlapping) block of parameters, $(d,\u)$ or $(d, \v)$, with the rest
held fixed. Since the objective is a function of $(d, \u)$ or $(d, \v)$ only through the products $d\u$ or $d\v$, the norm constraints are
avoided in the sub-routines which then become exactly lasso problems. \citet{MishraDeyChen2017} showed that ACS converges to a coordinate-wise minimum point of \eqref{eq:unitopt} and the sequence of solutions along the iterations is uniformly bounded. With such a typical approach, however, the optimization needs to be repeated for a grid of $\lambda$ values, and even with a warm-start strategy, the computation can still be expensive.


 Inspired by the stagewise learning paradigm \citep{efron2004least,zhao2007,tibshirani2015stagewise} and as the second stage of our divide-and-conquer strategy, we innovate a \textit{contended stagewise learning strategy} (CostLes) to trace out the entire solution paths of CURE. Although CURE is non-convex, we come to realize that efficient and principled stagewise learning remains possible. Our idea is simple yet elegant: when determining the update at each stagewise step, all the proposals from the subproblems of the potential blockwise updates have to compete with each other, and only the winner gets executed. This is in contrast to the classical setup, where in each update there is only a single proposal that is determined in a global fashion \citep{zhao2007,tibshirani2015stagewise}. 

We now present in detail our proposed stagewise procedure for CURE. For improving the estimation stability and facilitate calculations in practice, we consider a more general CURE problem with an added $\ell_2$ penalty term and adopt an alternative set of normalization constraints,
\begin{equation}\label{elas_cure}
\begin{aligned}
    \min_{d,\u,\v}~ Q(d\u\v\trans) = \left\{L(d\u\v\trans) + \lambda \rho(d\u\v\trans) \right\}, \qquad \mbox{s.t.}\, d\geq 0, \|\u\|_1=1, \|\v\|_1 = 1,
\end{aligned}
\end{equation}
where $L(d\u\v\trans) = (2n)^{-1}\|\Y-d\X\u\v\trans\|_F^2 + (\mu/2)\norm{d\u\v\trans}_{F}^{2}$ and  $\rho(\d\u\v\trans) = \lambda d \|\u\|_1 \|\v\|_1$.
This is the same as using a strictly convex elastic net penalty on $\C = \d\u\v\trans$ whenever $\mu>0$ \citep{zouElas}. Here we consider $\mu$ as a pre-specified fixed constant, and our goal remains to trace out the solution paths with respect to varying $\lambda$. In Theorem \ref{th:algorithm:converge}, we show that the $\ell_2$ term can asymptotically vanish, so the CURE problems with or without the $\ell_2$ penalty are asymptotically equivalent. Since the objective is a function of $d\u\v\trans$, the normalization of $\u$ and $\v$ can be arbitrary and can always be absorbed to $d$; we thus use $\ell_1$ normalization here, and the solution can be easily re-normalized to satisfy the original set of constraints in \eqref{eq:unitopt}, corresponding to the $\bP$-orthogonal SVD.

In what follows, we first present the general structure of the contended stagewise learning algorithm, and then provide the resulting simple problems/solutions for the incremental updates. The detailed derivations are given in Section \ref{sec:app:cure} of the Supplementary Materials. 


\noindent \textbf{(I) Initialization}. 
Set $t=0$. This step is based on searching for the initial non-zero entry of $\C$ as follows:
\begin{equation}\label{eq:ini}
\begin{split}
(\widehat{j},\widehat{k},\widehat{s}) & = \arg\min_{(j,k);s=\pm\epsilon} L(d\u\v\trans) \mbox{ s.t. } du_jv_k=s, u_{j'}v_{k'}=0, \forall j'\neq j, k'\neq k;\\
& = \arg\min_{(j,k);s=\pm\epsilon} L(s\1_j\1_k\trans);\\
 \widehat{\u}^0 & = \1_{\widehat{j}}; \widehat{\v}^0 = \sgn(\widehat{s})\1_{\widehat{k}}; \widehat{d}^0 = \epsilon;\\
 \mathcal{A}^0 & = \{\widehat{j}\}, \mathcal{B}^0 = \{\widehat{k}\}; \lambda^0 = \frac{1}{\epsilon}\{L(\0)-L(\widehat{d}^0\widehat{\u}^0\widehat{\v}\ztrans)\}.
\end{split}
\end{equation}
Here $\1_j$ ($\1_k$) is an $p\times 1$ ($q \times 1$) standard basis vector with all zeros except for a one in its $j$th ($k$th) coordinate. The $\mathcal{A}^0$ and $\mathcal{B}^0$ are the initial row and column active sets of non-zero indices, respectively.


\noindent \textbf{(II) Backward Update}. The update is about searching for the best row or column index to move ``backward'' within the current active sets in a blockwise fashion. At the $(t+1)$th step, there are two options:
\begin{equation}\label{eq:back:du}
\begin{split}
\what{j} & = \arg\min_{j\in\mathcal{A}^t} L(d\u\v\trans) \mbox{ s.t. } (d\u) = (d\u)^t-\mbox{sgn}(u_j^t)\epsilon\1_j, \v = \v^{t};\\
(d\u)^{t+1} & = (d\u)^t - \mbox{sgn}(u_{\what{j}}^t)\epsilon\1_{\what{j}}; d^{t+1} = \|(d\u)^{t+1}\|_1; \u^{t+1} = (d\u)^{t+1}/d^{t+1}; \v^{t+1} = \v^{t}; \mbox{(scaling)}\\
\mA^{t+1}  & = \mA^{t} \mbox{ if } u_{\what{j}}^{t+1} \neq 0; \mA^{t+1}  = \mA^{t} - \{\what{j}\} \mbox{ if } u_{\what{j}}^{t+1} = 0; \mB^{t+1} = \mB^{t}; \\
\lambda^{t+1} & = \lambda^t;
\end{split}
\end{equation}
or
\begin{equation}\label{eq:back:dv}
\begin{split}
\what{k} & = \arg\min_{k\in\mathcal{B}^t} L(d\u\v\trans) \mbox{ s.t. } (d\v) = (d\v)^t-\mbox{sgn}(v_k^t)\epsilon\1_k, \u = \u^{t};\\
(d\v)^{t+1} & = (d\v)^t - \mbox{sgn}(v_{\what{k}}^t)\epsilon\1_{\what{k}}; d^{t+1} = \|(d\v)^{t+1}\|_1; \v^{t+1} = (d\v)^{t+1}/d^{t+1}; \u^{t+1} = \u^{t};\\
\mA^{t+1} & = \mA^{t}; \mB^{t+1}  = \mB^{t} \mbox{ if } v_{\what{k}}^{t+1} \neq 0; \mB^{t+1}  = \mB^{t} - \{\what{k}\} \mbox{ if } v_{\what{k}}^{t+1} = 0;\\
\lambda^{t+1} & = \lambda^t.
\end{split}
\end{equation}
The penalty term will be decreased by a fixed amount $\lambda^t\epsilon$ (as $\|\u^t\|_1=\|\v^t\|_1=1$); the option with a lower $L$ value wins the bid. To ensure the objective in \eqref{eq:unitopt} is reduced, the winning option is executed only if $L(d^{t+1}\u^{t+1}\v\t1trans) - L(d^t\u^t\v\ttrans) <  \lambda^t\epsilon - \xi$, 
where $\xi = o(\epsilon) >0$ is a tolerance level.

\noindent \textbf{(III) Forward Update.} When the backward update can no longer proceed, a forward update is carried out, which searches for the best row or column index to move ``forward'' over all indices, again, in a blockwise fashion. At the $(t+1)$th step, the update chooses between two proposals:
\begin{equation}\label{eq:forw:du}
\begin{split}
(\what{j}, \what{s})& = \arg\min_{j;s=\pm\epsilon} L(d\u\v\trans) \mbox{ s.t. } (d\u) = (d\u)^t+s\1_j, \v = \v^{t};\\
(d\u)^{t+1} & = (d\u)^t+\what{s}\1_{\what{j}};  d^{t+1} = \|(d\u)^{t+1}\|_1; \u^{t+1} = (d\u)^{t+1}/d^{t+1}; \v^{t+1} = \v^{t};\\
\mA^{t+1} & = \mA^{t} \cup {\what{j}}, \mB^{t+1} = \mB^{t};\\
\lambda^{t+1} & =  \min(\lambda^t,\{L(d^t\u^t\v\ttrans)-L(d^{t+1}\u^{t+1}\v\t1trans)-\xi\}/\epsilon),
\end{split}
\end{equation}
or
\begin{equation}\label{eq:forw:dv}
\begin{split}
(\what{k},\what{h}) & = \arg\min_{k;h=\pm\epsilon} L(d\u\v\trans) \mbox{ s.t. } (d\v) = (d\v)^t+h\1_k, \u = \u^{t};\\
(d\v)^{t+1} & = (d\v)^t+\what{h}\1_{\what{k}};  d^{t+1} = \|(d\v)^{t+1}\|_1; \v^{t+1} = (d\v)^{t+1}/d^{t+1}; \u^{t+1} = \u^{t};\\
\mA^{t+1} & = \mA^{t}, \mB^{t+1} = \mB^{t} \cup {\what{k}};\\
\lambda^{t+1} & = \min(\lambda^t,\{L(d^t\u^t\v\ttrans)-L(d^{t+1}\u^{t+1}\v\t1trans)-\xi\}/\epsilon).
\end{split}
\end{equation}
As $\rho(d^{t+1}\u^{t+1}\v\t1trans)-\rho(d^t\u^t\v\ttrans) = \epsilon$, the option that can decrease $L$ more is executed.

The structure of the proposed stagewise learning procedure is summarized in Algorithm \ref{alg:cure}. It is intuitive how this procedure works. A forward update is taken only when backward updates can no longer proceed, i.e., the regularized loss at $\lambda^t$ can no longer be reduced by searching over the ``backward directions'' in a blockwise fashion within the current active sets. When a forward step is taken, the search expends to all possible blockwise directions, and the parameter $\lambda^{t}$ gets reduced only when the corresponding regularized loss can not be further reduced by any incremental update on any block of parameters. Indeed, we show in Section \ref{sec:stage:conv} that when $\epsilon$ and $\xi = o(\epsilon)$ go to zero at the time $\lambda^t$ gets reduced to $\lambda^{t+1}$, a coordinatewise minimum point of the regularized loss with $\lambda^{t}$ can be reached. 

\begin{algorithm}
  \caption{Contended Stagewise Learning for CURE (Pseudo code)}\label{alg:cure}
  \begin{algorithmic}
  \State \textbf{Initialization step}: based on searching for the first nonzero entries in $\u$ and $\v$ by \eqref{eq:ini}. Set $\epsilon$ and $\xi$. Set $t = 0$. 
  \Repeat
  \State \textbf{Backward step}: compare the proposals of blockwise updates within the current active sets, and choose the one that induces the least increment in $L$ (or the most reduction):
  \begin{itemize}
    \setlength{\itemindent}{.25in}
    \item Proposal 1: update $(d,\u)$, $\mA^{t+1}$ and $\lambda^{t+1}$ by \eqref{eq:back:du}.
    \item Proposal 2: update $(d,\v)$, $\mB^{t+1}$ and $\lambda^{t+1}$ by \eqref{eq:back:dv}.
  \end{itemize}
  \If{the induced increment in $L$ is less than $\lambda^t\epsilon-\xi$,}
  \State Execute the chosen proposal of backward update.
  \Else
  \State \textbf{Forward step}: compare the proposals of blockwise updates over all row/column indices, choose the one that reduces the loss $L$ the most, and directly execute:
  \begin{itemize}
    \setlength{\itemindent}{.5in}
      \item Proposal 1: update $(d,\u)$, $\mA^{t+1}$ and $\lambda^{t+1}$ by \eqref{eq:forw:du}.
      \item Proposal 2: update $(d,\v)$, $\mB^{t+1}$ and $\lambda^{t+1}$ by \eqref{eq:forw:dv}.
      \end{itemize}      
  \EndIf
  \State $t \rightarrow t + 1$.
  \Until{$\lambda_{t}\le 0$.}
  \end{algorithmic}\label{stagewise_algorithm}
\end{algorithm}

The computations involved in the contended stagewise learning are straightforward. Define the current residual matrix as $\E^t = \Y - d^t\X\u^t\v^t\trans$ and denote its $k$th column as $\widetilde{\e}_k$. Recall that $\widetilde{\x}_{j}$ and $\widetilde{\y}_k$ are the $j$th and $k$th column of $\X$ and $\Y$, respectively, as defined in Section \ref{sec:ssvd}. Then, the initialization step boils down to the following,
\begin{equation}\label{implementation:init}
  \begin{aligned}
      (\what{j},\what{k})
      = \underset{j,k}{\arg\min}~\left\{(2n)^{-1}\epsilon\norm{\widetilde{\x}_j}_2^2 - n^{-1}\abs{\widetilde{\x}_{j}\trans\widetilde{\y}_{k}}\right\},\qquad
      \widehat{s}
      = \mbox{sgn}(\widetilde{\x}_{j}\trans\widetilde{\y}_{k})\epsilon. 
  \end{aligned}
\end{equation} 
The problem of the $(t+1)$th backward step is 
\begin{equation}\label{implementation:backward}
  \begin{aligned}
      \what{j} & = \underset{j\in\mathcal{A}^t}{\arg\min}~ (2n)^{-1}\epsilon\norm{\widetilde{\x}_j}_2^2\norm{\v^{t}}_{2}^{2} +
      n^{-1}\mbox{sgn}(u_{j}^t)\widetilde{\x}_{j}\trans\E^t\v^t - \mu d^t\abs{u_j^t}\norm{\v^t}_2^2,               \\
      \what{k} & = \underset{k\in\mathcal{B}^t}{\arg\min}~ n^{-1}\mbox{sgn}(v_k^t)\u^t\trans\X\trans\widetilde{\e}_k^t -
      \mu d^t\abs{v_k^t}\norm{\u^t}_2^2.
  \end{aligned}
\end{equation}
The problem of the $(t+1)$th forward step is 
\begin{equation}\label{implementation:forward}
  \begin{aligned}
      \what{j} & = \underset{j}{\arg\max}~  \abs{n^{-1}\widetilde{\x}_j\trans\E^t\v^t - \mu d^tu_j^t\norm{\v^t}_2^2} 
      -(2n)^{-1}\epsilon\norm{\widetilde{\x}_j}_2^2\norm{\v^t}_2^2;\\
      & \widehat{s} = \mbox{sgn}(n^{-1}\widetilde{\x}_{\widehat{j}}\trans\E^t\v^t - \mu d^tu_{\widehat{j}}^t\norm{\v^t}_2^2)\epsilon,\\                                           
      \what{k} & = \underset{k}{\arg\max}~ \abs{n^{-1}\u^t\trans\X\trans\widetilde{\e}_k^t - \mu d^t v_k^t\norm{\u^t}_2^2};\\
      & \widehat{h}  = \mbox{sgn}(n^{-1}\u^t\trans\X\trans\widetilde{\e}_{\widehat{k}}^t - \mu d^t v_{\widehat{k}}^t\norm{\u^t}_2^2)\epsilon.
  \end{aligned}
\end{equation}
Lastly, for choosing between updating $d\u$ or $d\v$ or between proceeding with backward or forward update, the change in the loss function are computed as  
\begin{equation*}
  \begin{aligned}
      &L\left(\left[(d\u)^t+\widehat{s}\1_{\widehat{j}}\right]\v^t\trans\right) - L(d^t\u^t\v^t\trans)
      = \frac{\epsilon^2}{2n}\norm{\widetilde{\x}_{\widehat{j}}}_2^2\norm{\v^t}_2^2 - \frac{\widehat{s}}{n}\widetilde{\x}_{\widehat{j}}\trans\E^t\v^t 
      + \frac{\mu}{2}\epsilon^2\norm{\v^t}_2^2 + \mu\widehat{s}d^tu_{\widehat{j}}^t\norm{\v^t}_2^2, \\
      &L\left(\u^t\left[(d\v)^t+\widehat{h}\1_{\widehat{k}}\right]\trans\right) - L(d^t\u^t\v^t\trans) 
      = \frac{\epsilon^2}{2n}\norm{\X\u^t}_2^2 - \frac{\widehat{h}}{n}\u^t\trans\X\trans\widetilde{\e}_{\widehat{k}}^t
      + \frac{\mu\epsilon^2}{2}\norm{\u^t}_2^2 + \mu\widehat{h}d^t v_{\widehat{k}}^t\norm{\u^t}_2^2.
  \end{aligned}
\end{equation*} 


To streamline the main ideas, we haven't discussed much on the handling of missing values. In fact, our stagewise methods can conveniently handle missing values in the response matrix, which makes it applicable for large-scale matrix completion \citep{Candes2009,Candes2010}. We shall briefly describe the setup. Let $\mathcal{H}$ be the index set of all observed values in $\Y$, i.e., $\mathcal{H}= \{(i,j); y_{ij} \ \mbox{is observed}, i=1,\ldots,n, j=1,\ldots, q\}$. Define 
$P_{\mathcal{H}}:\mathbb{R}^{n\times q}\rightarrow \mathbb{R}^{n\times q}$ be the projection operator onto $\mathcal{H}$, so that for any matrix $\Z \in \mathbb{R}^{n\times q}$, the entries of $P_{\mathcal{H}}(\Z)$ equal to those of $\Z$ on $\mathcal{H}$ and otherwise equal to zero. Then, corresponding to the complete data case in \eqref{elas_cure}, the CURE problem with incomplete data can be expressed as 
\begin{equation*}
\begin{aligned}    \underset{d,\u,\v}{\min}~&(2n)^{-1}\norm{P_{\mathcal{H}}(\Y)-P_{\mathcal{H}}(d\X\u\v\trans)}_{F}^{2} + 
    \frac{\mu}{2}\norm{d\u\v\trans}_{F}^{2} + \lambda d\norm{\u}_{1}\norm{\v}_{1} \\
    &\text{s.t.}~d\geq 0, \|\u\|_1 = 1,~\|\v\|_1= 1.
\end{aligned}
\end{equation*}
We can show that the proposed methods still work with slight modification and maintain computational efficiency; the only change is to replace the residual matrix $\E^t$ with its projection $P_{\mathcal{H}}(\E^t)$ during the iterations. We have implemented all the proposed computational methods in a user-friendly R package.

\subsection{Computational Complexity and Convergence}\label{sec:stage:conv}



Theorem \ref{th:algorithm:complexity} presents the computational complexity of the proposed stagewise procedure. 

\begin{theorem}\label{th:algorithm:complexity}
  Consider the $(t+1)$th step of Algorithm \ref{alg:cure} for contended stagewise learning of \eqref{elas_cure}. Let $\mA^t = a^t$ and $\mB^t=b^t$. The computational complexity of the update is $O(a^tnq + b^tnp)$.
\end{theorem}

The results show that the stagewise learning can be much more efficient than the ACS approach. ACS needs to be run for a grid of $\lambda$ values; for each fixed $\lambda$, the algorithm alternates between two lasso problems with computational complexity $O(npq)$  \citep{friedman2010regularization} until convergence. In contrast, our method costs $O(npq)$ operations in the initialization step and $O(a^tnq + b^tnp)$ operations in the subsequent steps, and the solution paths are traced out in a single run. 


We now attempt to quantify the proximity between the stagewise approximated solution and that from fully optimizing the regularized loss function through ACS.

\begin{lemma}\label{lemma:algorithm:converge}
  Consider Algorithm \ref{alg:cure} for the contended stagewise learning of \eqref{elas_cure}. When $\lambda^{t+1} < \lambda^{t}$, that is, whenever $\lambda^{t}$ gets reduced during the stagewise learning, we have 
  \begin{equation*}
    \max\left[\norm{(d\u)^{t*}\v^t\trans - d^t\u^t\v^t\trans}_F,~\norm{\u^t(d\v)^{t*}\trans - d^t\u^t\v^t\trans}_F\right] 
    \leq 2\sqrt{pq}\left(\frac{M\epsilon}{2\mu}+\frac{\xi}{\epsilon\mu}\right),
  \end{equation*}
  where $M$ is a constant, and 
  \begin{equation*}
    \begin{aligned}
      (d\u)^{t*} = \arg\min_{d\u}~Q(d\u\v^t\trans;\lambda^t),~  
      (d\v)^{t*} = \arg\min_{d\v}~Q(\u^t(d\v)\trans;\lambda^t).
    \end{aligned}
  \end{equation*}
\end{lemma} 

Let $\mathbb{B}(d^t\u^t\v^t\trans;2\sqrt{pq}[\frac{M\epsilon}{2\mu}+\frac{\xi}{\epsilon\mu}])$ be a $(pq-1)$-dimensional ball with center $d^t\u^t\v\ttrans$ and 
radius $2\sqrt{pq}[\frac{M\epsilon}{2\mu}+\frac{\xi}{\epsilon\mu}]$. By Lemma \ref{lemma:algorithm:converge}, the center will converge to the corresponding solution of ACS when the radius converges to zero, because $(d\u)^{t*}\v^t\trans$ and $\u^t(d\v)^{t*}\trans$ are always in the ball. In addition, the inequality in Lemma \ref{lemma:algorithm:converge} holds for any $t\geq 0$, so we can directly obtain the pathwise convergence as shown in the following theorem.


\begin{theorem}\label{th:algorithm:converge}
  Consider Algorithm \ref{alg:cure} for the contended stagewise learning of \eqref{elas_cure}. When $\epsilon\rightarrow 0$ and $\xi=o(\epsilon)$, the 
  solution paths of the stagewise learning converge to those of the ACS uniformly. Moreover, the results remain hold when $\mu\rightarrow 0$, $\epsilon=o(\mu)$ and $\xi=o(\epsilon\mu)$.
\end{theorem}


\subsection{Tuning \& Early Stopping}


In practice, it is often required to identify the optimal solution along the paths. The $K$-fold cross validation can be used to evaluate the models, for which the stagewise learning procedure has to be run multiple times. Alternatively, selecting tuning parameters by information criterion can be more efficient. In our implementation, the default choice is to use the generalized information criterion (GIC) \citep{fan2013tuning}:
\begin{equation*}
\mbox{GIC}(\lambda^t) = \mbox{log}\|\bY - d^{t}\bX\bu^{t}\bv\ttrans\|_F^2 + \frac{\mbox{loglog}(nq)\mbox{log}(pq)}{nq}\widehat{df}(\lambda^t)
\end{equation*}
where $\widehat{df}(\lambda^t) = \|\u^t\|_0 +  \|\v^t\|_0 -1$, is the estimated degrees of freedom for the stagewise model. GIC is shown to perform well in our numerical studies. Other implemented criteria including AIC, BIC, and out-of-sample prediction error (when additional testing samples are available) are optional. 

Due to the nature of stagewise learning, an early stopping mechanism can be implemented based on monitoring the information criterion during the learning process. As the model complexity gradually increases during the stagewise learning, the information criterion, which balances model fitting and model complexity, is expected to first decrease and then increase. As such, the stagewise learning can be terminated early if such a convex pattern of the information criterion is detected. In our implementation, the default is to stop the stagewise learning if the information criterion has not being decreasing for 300 consecutive steps. This strategy avoids the fitting of overly-complex models. 


\section{Simulation}\label{sec:simulation}

\subsection{Setups}\label{sim:setup}

We conduct simulation studies with data generated from the co-sparse factor regression model as specified in \eqref{eq:decomposition1} and \eqref{eq:samplemodelC}. We consider three simulation setups, which differ mainly on the generation of the true coefficient matrix $\C^* = \U^*\D^*\V\strans \in \mathbb{R}^{p\times q}$, where $\U^* = \left[\u_1^*,\dots, \u_{r^*}^*\right]$, $\V^*=\left[\v_1^*,\dots, \v_{r^*}^*\right]$, and $\D^* = \mbox{diag}\{d_1^*\ldots, d_{r^*}^*\}$.

Model \upperroman{1} is a unit-rank model mainly for analyzing the properties of CURE, in which we set
\begin{align*}
     \u_1^* & = \bar{\u}_1/\norm{\bar{\u}_1}_2 \mbox{ where } \bar{\u}_1  = \left[10, -10, 8, -8, 5, -5, \text{rep}(3, 5), \text{rep}(-3, 5), \text{rep}(0, p - 16)\right]\trans,\\
    \v_1^* & = \bar{\v}_1/\norm{\bar{\v}_1}_2 \mbox{ where } \bar{\v}_1  = \left[10, -9, 8, -7, 6, -5, 4, -3, \text{rep}(2, 17), \text{rep}(0, q - 25)\right]\trans,
\end{align*}
and $d_1^*=20$, where $\text{rep}(a,b)$ represent a $1\times b$ vector with all entries equaling to $a$.

Models II and III are multi-rank models. In Model II, the singular values and singular vectors are generated as follows,
\begin{align*}
\u_k^* & = \bar{\u}_k / \norm{\bar{\u}_k}_2 \mbox{ where }
\bar{\u}_k = [\text{rep}(0, k-1),\text{unif}(\mathcal{Q}_u, s_u),\text{rep}(0, p-s_u-k+1)]\trans,\\
\v_k^* & = \bar{\v}_k / \norm{\bar{\v}_k}_2 \mbox{ where }
\bar{\v}_k = [\text{rep}(0, k-1),\text{unif}(\mathcal{Q}_v, s_v),\text{rep}(0, q-s_v-k+1)]\trans,\\
d_k^* & = 5 + 5(r^* - k + 1),
\end{align*}
for $k = 1,\ldots, r^*$, where $s_u = 3$, $s_v=4$, and $\text{unif}(\mathcal{Q}, s)$ denotes a vector of length $s$ whose entries are i.i.d. uniformly distributed on the set $\mathcal{Q}$; here we set
$\mathcal{Q}_u = \left\{1, -1\right\}$ and $\mathcal{Q}_v = [-1,-0.3]\cup[0.3,1]$ and adopt an additional Gram-Schmidt orthogonalization on $\bar{\v}_{k}$ to ensure the orthogonality.
Model \upperroman{3} is similar to Model \upperroman{2}, and the only difference is that we generate $\bar{\u}_k$ and $\bar{\v}_k$ by
\begin{equation*}
\begin{aligned}
    &\bar{\u}_k = [\text{rep}(0, s_u(k-1)),\text{unif}(\mathcal{Q}_u, s_u),\text{rep}(0, p-ks_u)]\trans, \\
    &\bar{\v}_k = [\text{rep}(0, s_v(k-1)),\text{unif}(\mathcal{Q}_v, s_v),\text{rep}(0, q-ks_v)]\trans.
\end{aligned}
\end{equation*}
In view of $\C^* = \sum_{k=1}^{r^*}d_{k}^*\u_{k}^*\v_{k}\strans$, the coefficient matrix
in Model \upperroman{2} is more sparse than that in Model \upperroman{3} due to the overlap of nonzero components in different unit-rank matrices.

We then generate the design matrix $\X$, following the same procedure as in \citet{MishraDeyChen2017}. Specifically, let $\x \sim N(\0, \bGamma)$, where $\bGamma = (\gamma_{ij})_{p\times p}$ with $\gamma_{ij} = 0.5^{\abs{i-j}}$. Given $\U^* = \left[\u_1^*,\dots,\u_{r^*}^*\right]$, we can find $\U_{\bot}^* \in \mathbb{R}^{p\times (p-r^*)}$ such that ${\bf P} = [\U^*, \U_{\bot}^*]\in\mathbb{R}^{p\times p}$ and $\text{rank}({\bf P})=p$. Denote $\x_1=\U\strans\x$ and $\x_2 = \U_{\bot}\strans\x$. We first generate a matrix $\X_1\in\R^{n\times r^*}$ whose entries are from $N(\0,\I_{r^*})$ and then we generate $\X_2\in \mathbb{R}^{n\times (p-r^*)}$ by drawing $n$ random samples
from the conditional distribution of $\x_2$ given $\x_1$. The predictor matrix is then set as $\X=\left[\X_1,\X_2\right]{\bf P}^{-1}$ such that model assumption \eqref{eq:decomposition1} holds. The rows of the error matrix $\E$ are generated as i.i.d. samples from $N(0, \sigma^2\bDelta)$
where $\bDelta = (\delta_{ij})_{q\times q}$ with $\delta_{ij} = \rho^{\abs{i-j}}$. The response $\Y$ is then generate by $\Y = \X\C^* + \E$. We set $\sigma$ to control the signal-to-noise ratio (SNR), defined as
$\text{SNR} = \norm{d_{r^*}^*\X\u_{r^*}^*\v_{r^*}\strans}_{2}/\norm{\E}_F$.

\subsection{Path Convergence of Stagewise CURE}


We demonstrate that the stagewise paths of CURE closely mimic the solutions of ACS as the step size becomes sufficiently small. We use Model \upperroman{1}, with $n=p=q=200$, $\rho = 0.3$ and $\text{SNR}=0.25$. Figure \ref{fig:duv-path} shows the solution paths of $\widehat{d}\widehat{\u}$ and $\widehat{d}\widehat{\v}$ with varying $\epsilon$ values. Indeed, the stagewise paths approximately trace out the solution paths of ACS, and their discrepancies vanish as the step size gets smaller.

\begin{figure}[htp]
    \centering
    \subfloat[$\epsilon=2$]{
    \includegraphics[width=.23\columnwidth]{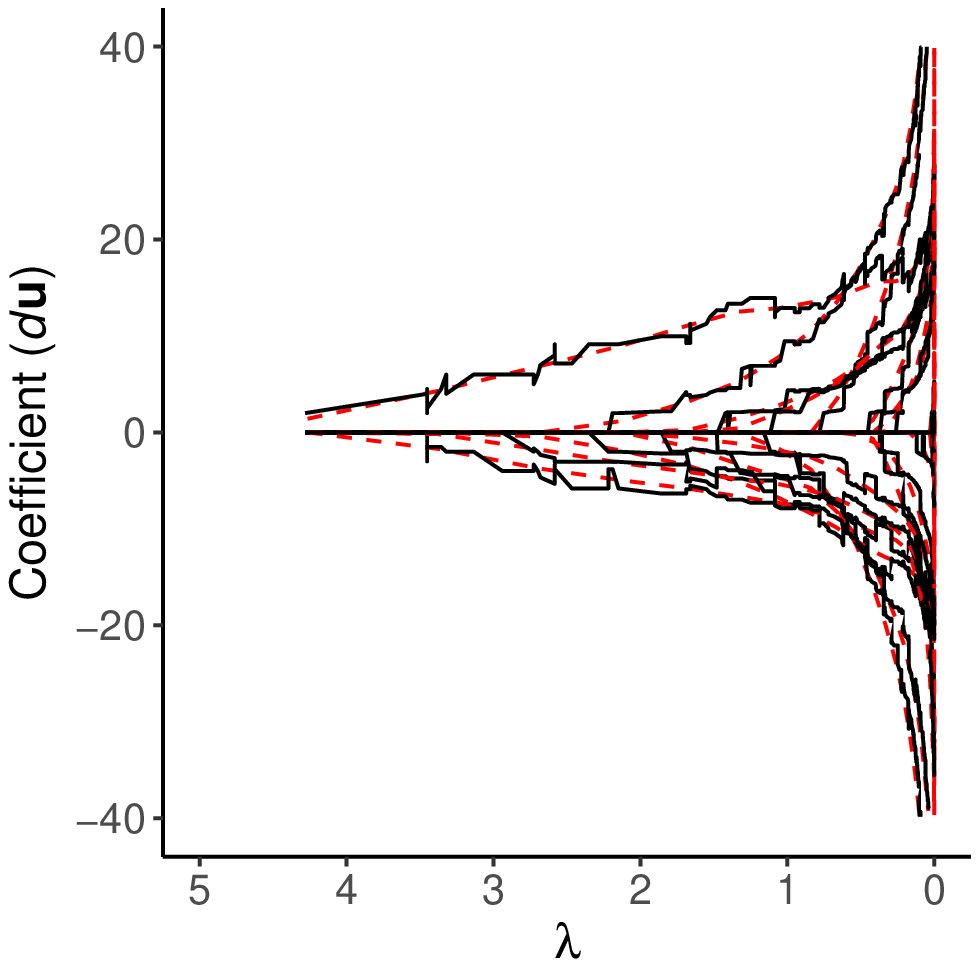}
    \includegraphics[width=.23\columnwidth]{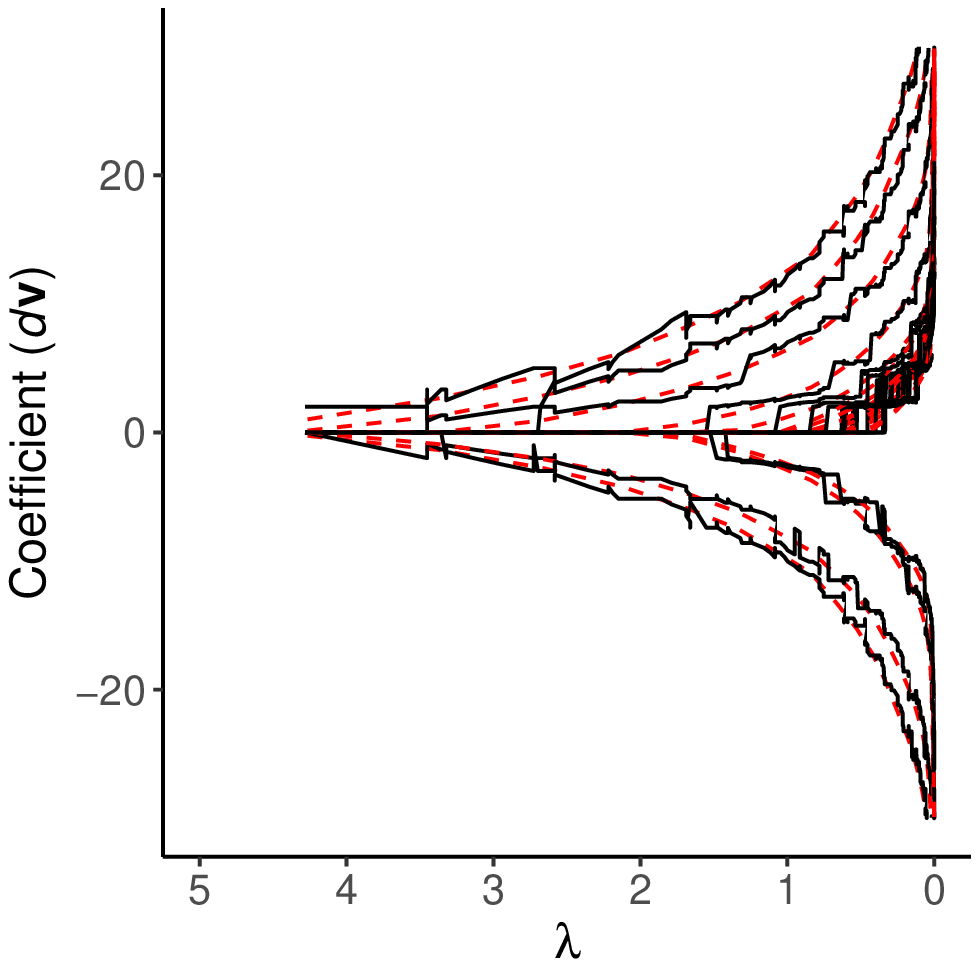}}\hfill
	\subfloat[$\epsilon=1.5$]{
    \includegraphics[width=.23\columnwidth]{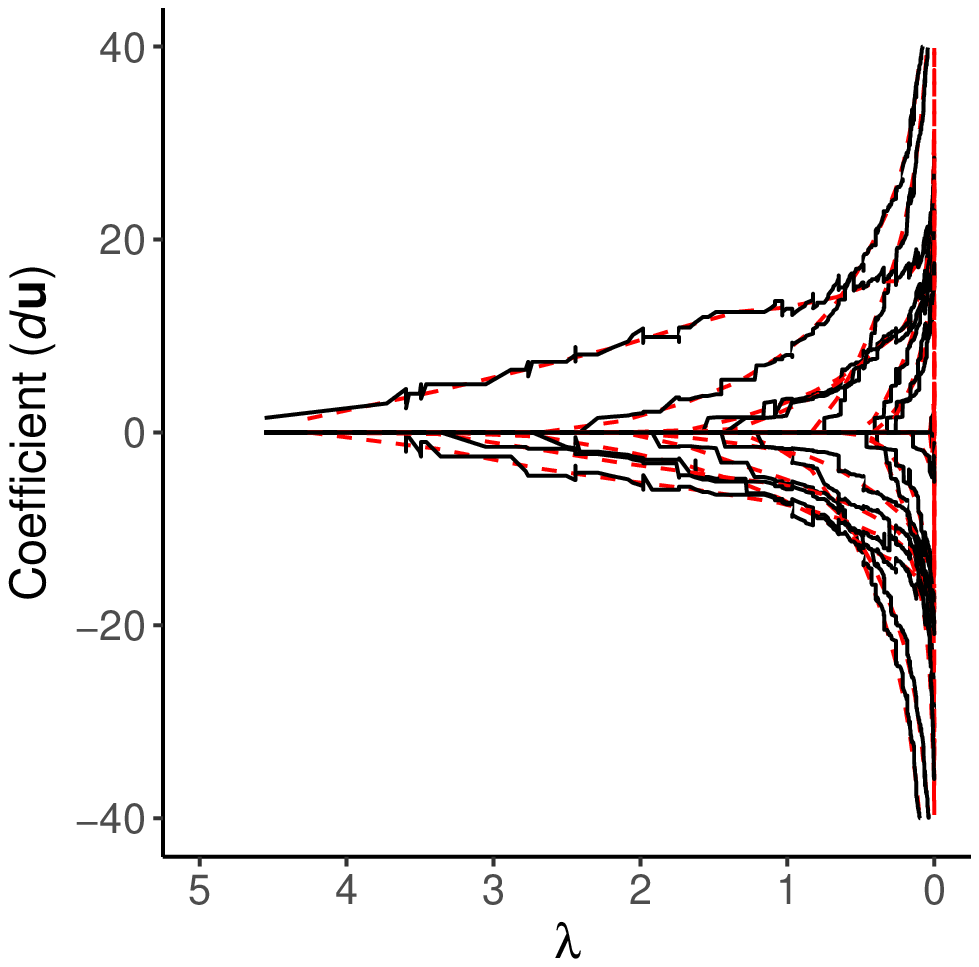}
    \includegraphics[width=.23\columnwidth]{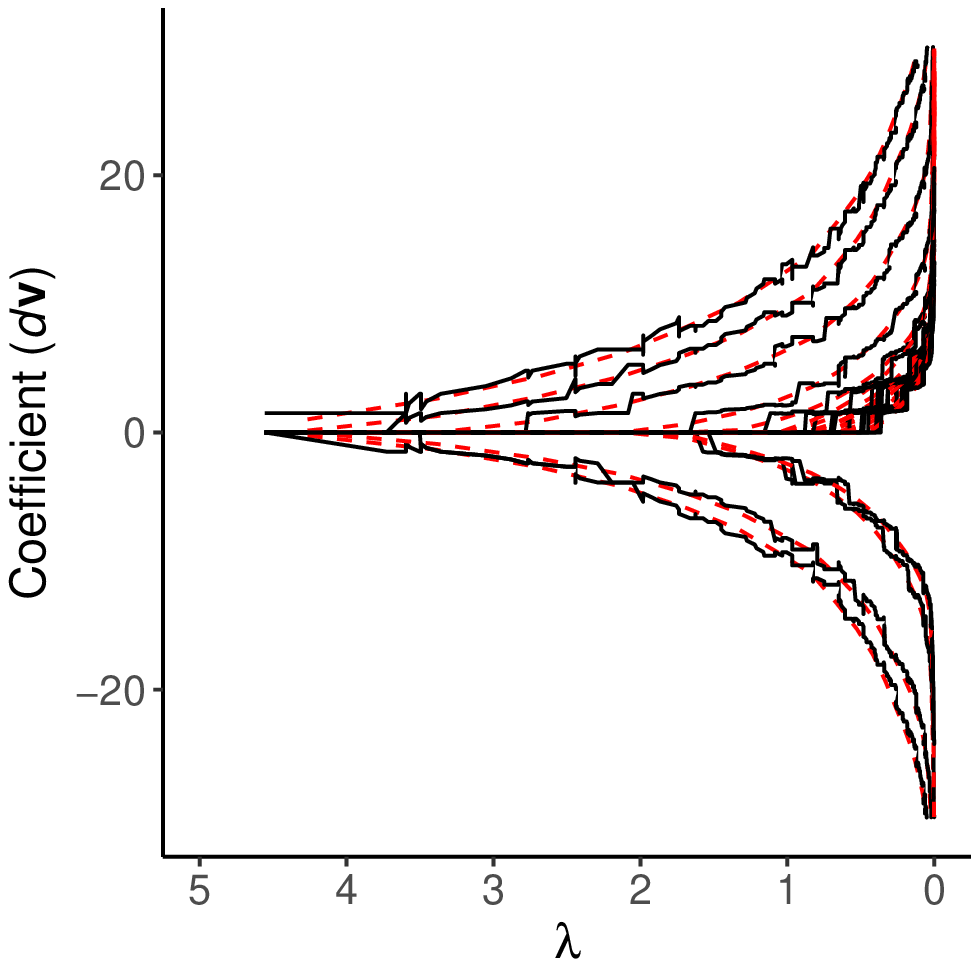}} \\
    \subfloat[$\epsilon=1$]{
    \includegraphics[width=.23\columnwidth]{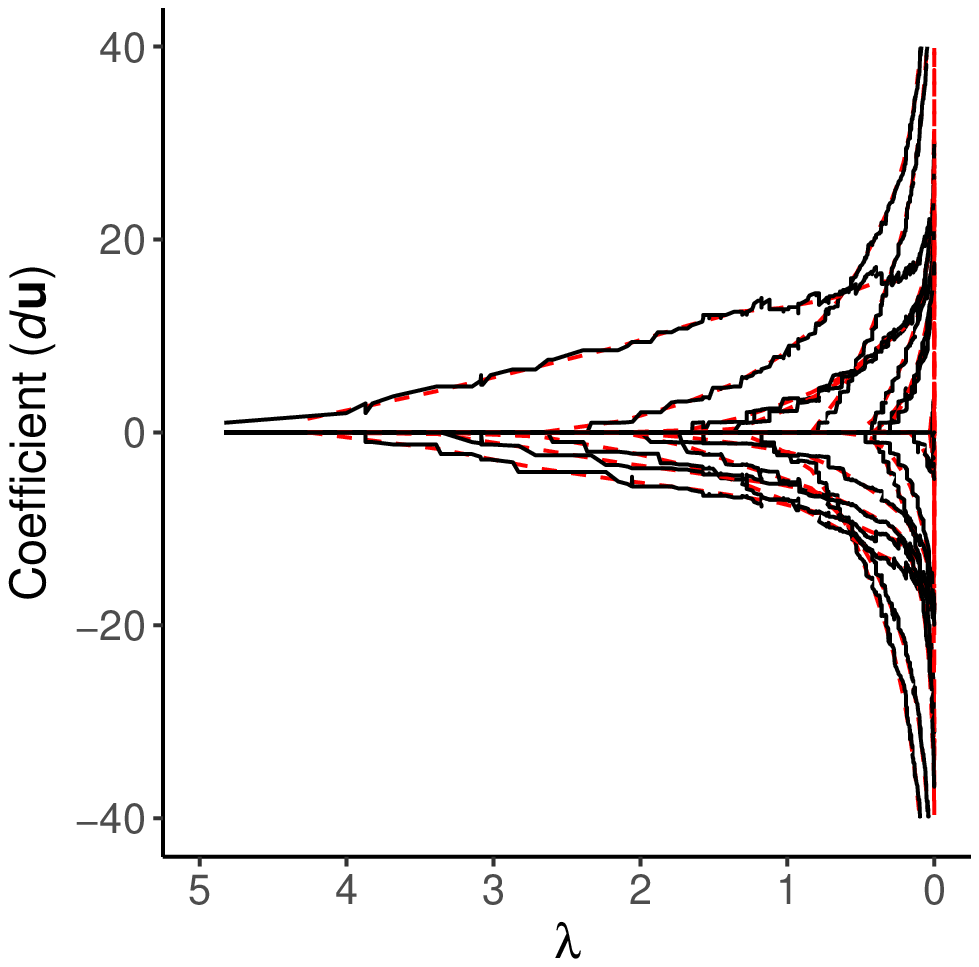}
    \includegraphics[width=.23\columnwidth]{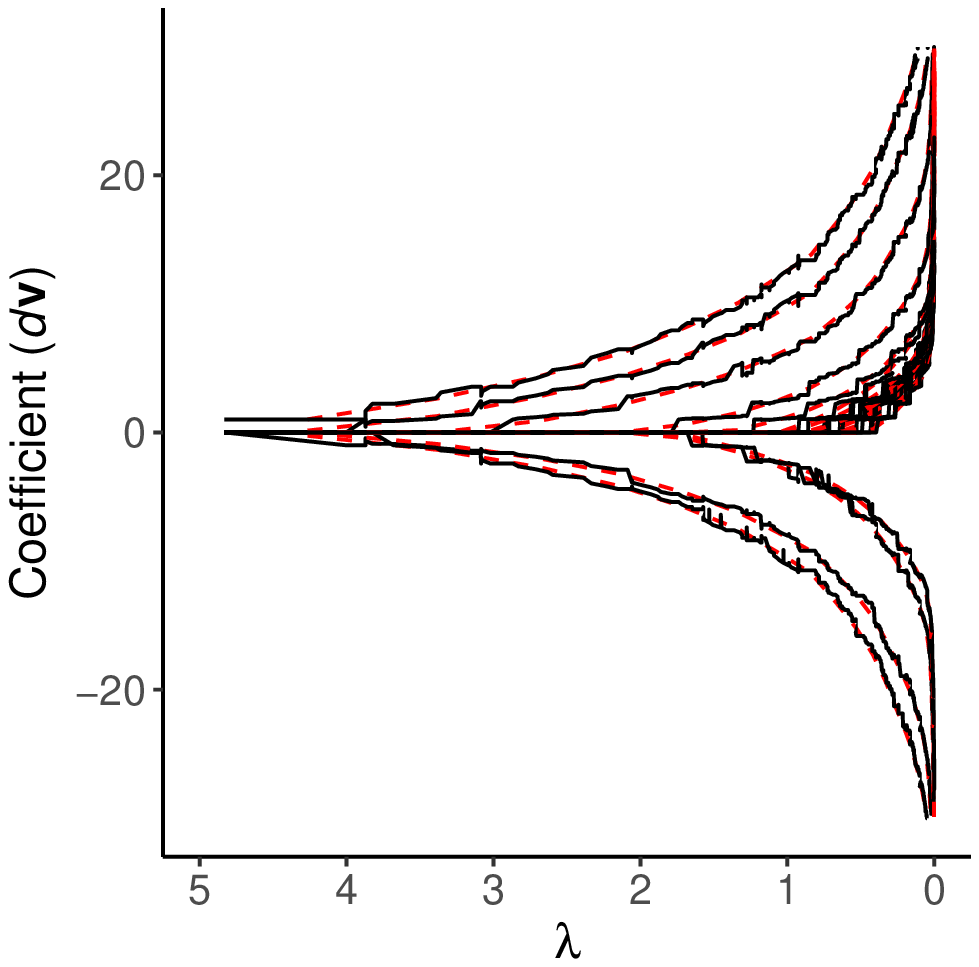}}\hfill
	\subfloat[$\epsilon=0.1$]{
    \includegraphics[width=.23\columnwidth]{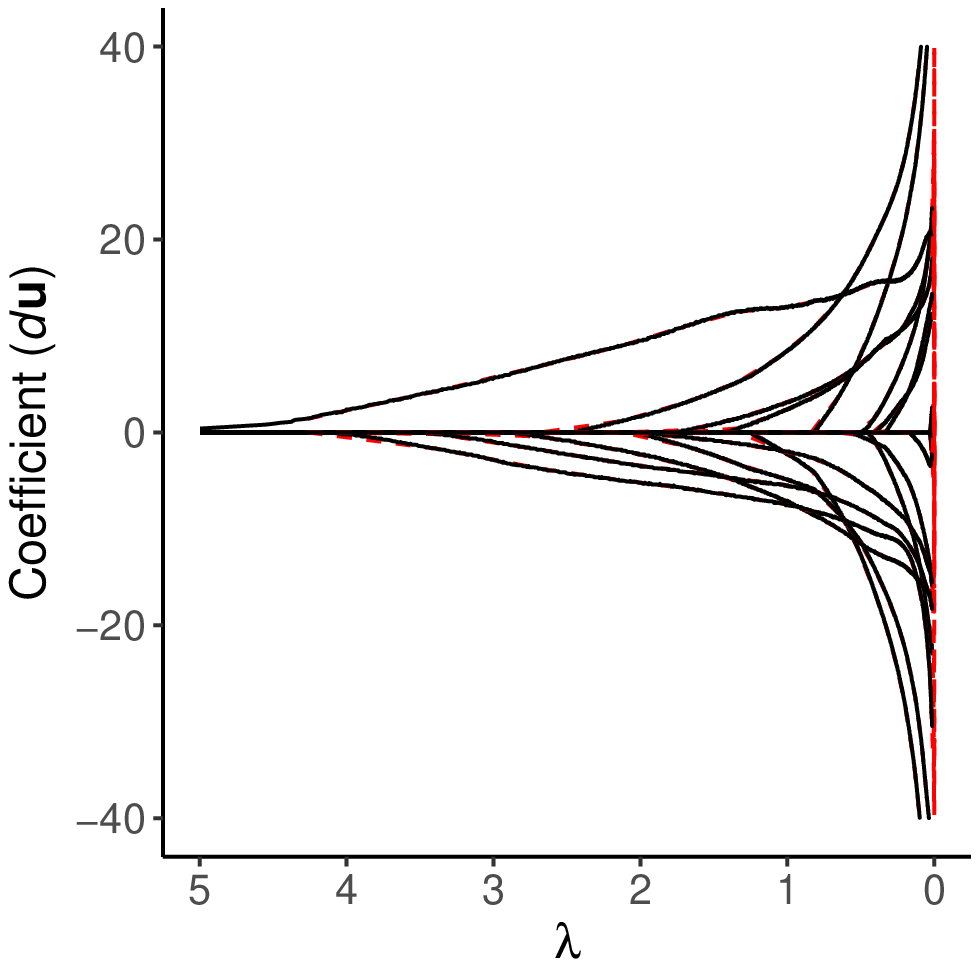}
    \includegraphics[width=.23\columnwidth]{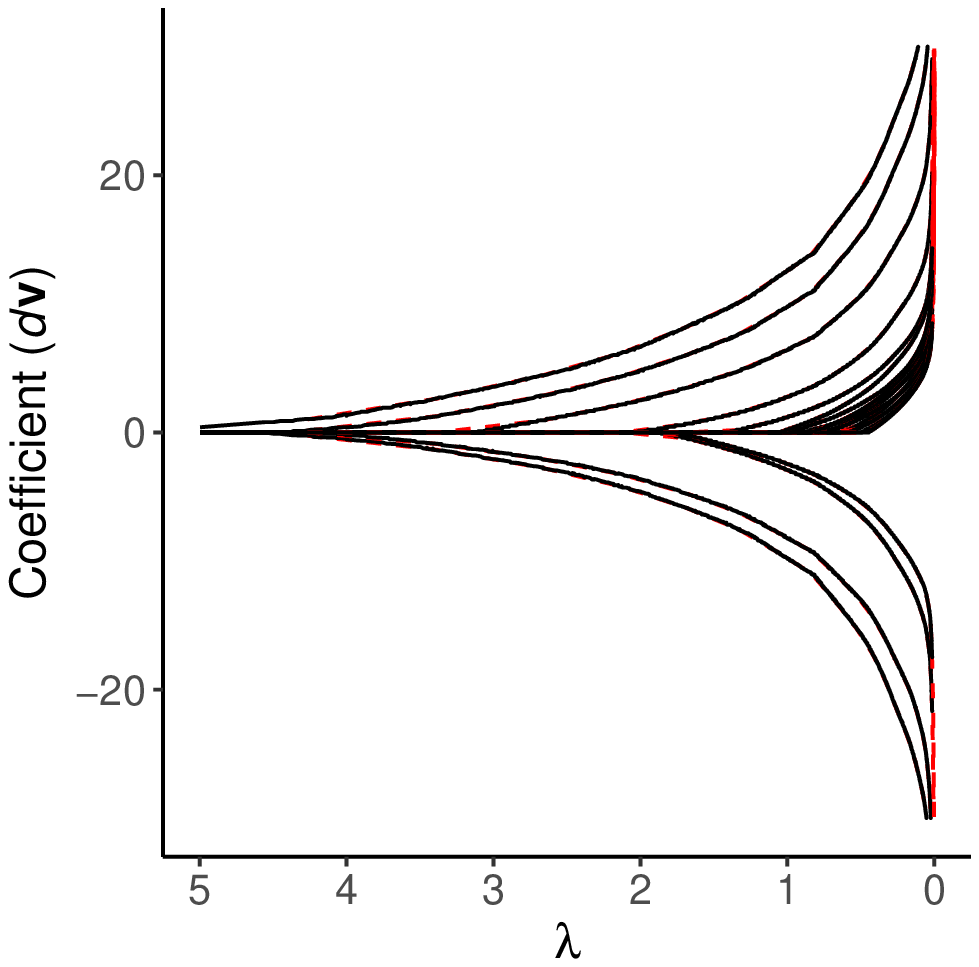}}
  \caption{Stagewsie paths of $\widehat{d}\widehat{\u}$ and $\widehat{d}\widehat{\v}$ for different step sizes. Stagewise paths are shown as black solid lines, and the exact solution paths from ACS are shown in red dashed lines.}\label{fig:duv-path}
\end{figure}

\subsection{Estimation Performance}

We compare the estimation accuracy of our proposed approaches to several competing regularized regression methods, including reduced-rank regression (RRR),
row-sparse reduced rank regression via adaptive group lasso (SRRR) \citep{chen2012sparse}, sparse orthogonal factor regression (SOFAR) \citep{uematsu2019sofar}.
For the two deflation based methods, we try both stagewise learning (STL) and ACS for solving CURE, resulting in sequential pursuit with stagewise learning (SeqSTL),
sequential pursuit with ACS (SeqACS), parallel pursuit with stagewise learning initialized by lasso (ParSTL(L)), parallel pursuit with stagewise learning initialized by
RRR (ParSTL(R)), parallel pursuit with ACS initialized by lasso (ParACS(L)), and parallel pursuit with ACS initialized by RRR (ParACS(R)). The ParACS(R) and
SeqACS are essentially corresponding to the reduced-rank regression with a sparse singular value decomposition (RSSVD) proposed by \citet{chen2012jrssb} and the sequential
co-sparse factor regression (SeFAR) proposed by \citet{MishraDeyChen2017}, respectively.

Models \upperroman{2} and \upperroman{3} are considered here with $n=q=100$, $p \in \{100, 200, 400\}$, $r^*\in \{3,6\}$, $\text{SNR}\in \{0.25, 0.5, 1\}$, $\rho=0.3$ and
$\epsilon=1$. The experiment under each setting is repeated 200 times. The estimation accuracy is measured by $\text{Er}(\widehat{\C}) = \norm{\widehat{\C} - \C^*}_{F}^2/(pq)$ and
$\text{Er}(\X\widehat{\C}) = \norm{\X(\widehat{\C} - \C^*)}_{F}^2/(nq)$.
The variable selection performance is characterized by the false positive rate (FPR) and false negative rate (FNR) in recovering the
sparsity patterns of the SVD structure, where $\text{FPR} = \text{FP}/(\text{TN}+\text{FP})$ and
$\text{FNR} = \text{FN}/(\text{TP}+\text{FN})$. Here, TP, FP, TN, and FN are the numbers of true
nonzeros, false nonzeros, true zeros, and false zeros of $\widehat{\U}$ and $\widehat{\V}$, respectively. Table \ref{tab:model-1} reports the results for Model \upperroman{2} with $\text{SNR} = 0.5$ and
$p \in \{ 200,400\}$, and Table \ref{tab:model-2} reports the results for Model \upperroman{3} under the same settings.



First of all, the stagewise methods are much faster than the other methods, and the efficiency gain in computation can be dramatic for models of large dimensions;
more results regarding computation time can be found in Section \ref{sec:sim:eff}. In general the estimation and prediction performance of stagewise method are
better or comparable to the other competing methods. RRR performs the worst, as it does not consider sparse estimation at all. SOFAR may perform unsatisfactorily
because its optimization is highly non-convex and it targets on the sparsity pattern in the SVD of $\C^*$ rather than its $\bf P-$orthogonal SVD. SRRR considers row-wise sparsity in $\C^*$ only.
In terms of the ACS solutions for CURE, with a more sparse $\C^*$, parallel pursuit usually performs better than sequential pursuit, which is consistent with our theoretical
results. When $\C^*$ becomes less sparse, the difference between parallel and sequential methods becomes negligible. It is also interesting to notice that stagewise methods may even slightly outperform their ACS counterparts. This may be due to the additional regularization effects of the stagewise approximation.


\singlespacing

\begin{table}[H]
  \caption{Results for Model II with $\text{SNR} = 0.5$, $p \in \{100,200,400\}$ and $r^* \in \{3,6\}$; here Er(C) and Er(XC) are rescaled by multiplying $10^3$.}\label{tab:model-1}
  \centering
  \resizebox{\columnwidth}{!}{
    \subfloat[$\text{SNR} = 0.5$, $r=3$]{
      \begin{tabular}{lccccc}
        \hline
        Method    & Er(C)                         & Er(XC) & FPR (\%) & FNR (\%) & Time (s) \\ \hline
                  & \multicolumn{5}{c}{$p = 100$}                                           \\ \hline
        RRR       & 4478.89                       & 246.22 & 100.00   & 0.00     & 0.08     \\
        SRRR      & 1.66                          & 139.64 & 57.32    & 0.00     & 5.28     \\
        SOFAR     & 6.44                          & 202.75 & 24.59    & 4.83     & 10.70    \\
        SeqACS    & 1.34                          & 85.83  & 0.96     & 4.54     & 0.50     \\
        ParACS(L) & 0.97                          & 70.50  & 2.10     & 3.56     & 8.08     \\
        ParACS(R) & 1.54                          & 92.22  & 3.24     & 5.92     & 7.00     \\
        SeqSTL    & 0.65                          & 42.58  & 0.85     & 4.52     & 0.38     \\
        ParSTL(L) & 0.70                          & 52.01  & 1.12     & 3.94     & 1.07     \\
        ParSTL(R) & 0.75                          & 46.16  & 1.06     & 3.96     & 0.15     \\ \hline
                  & \multicolumn{5}{c}{$p = 200$}                                           \\ \hline
        RRR       & 18.99                         & 247.89 & 100.00   & 0.00     & 0.12     \\
        SRRR      & 0.85                          & 160.50 & 64.31    & 0.00     & 6.98     \\
        SOFAR     & 3.45                          & 142.55 & 29.23    & 3.87     & 13.51    \\
        SeqACS    & 0.78                          & 99.79  & 0.67     & 4.46     & 1.68     \\
        ParACS(L) & 0.55                          & 80.38  & 1.58     & 3.33     & 8.84     \\
        ParACS(R) & 0.77                          & 95.32  & 2.03     & 8.71     & 8.08     \\
        SeqSTL    & 0.41                          & 50.72  & 0.59     & 4.12     & 0.47     \\
        ParSTL(L) & 0.42                          & 61.85  & 0.79     & 3.77     & 1.03     \\
        ParSTL(R) & 0.46                          & 54.22  & 0.75     & 3.85     & 0.17     \\ \hline
                  & \multicolumn{5}{c}{$p = 400$}                                           \\ \hline
        RRR       & 13.86                         & 243.25 & 100.00   & 0.00     & 0.47     \\
        SRRR      & 0.49                          & 194.07 & 50.02    & 0.00     & 21.94    \\
        SOFAR     & 5.61                          & 158.58 & 32.99    & 3.27     & 30.08    \\
        SeqACS    & 0.46                          & 115.37 & 0.44     & 5.17     & 12.47    \\
        ParACS(L) & 0.31                          & 87.26  & 1.13     & 4.12     & 12.41    \\
        ParACS(R) & 0.46                          & 106.39 & 1.22     & 6.40     & 12.83    \\
        SeqSTL    & 0.24                          & 57.48  & 0.37     & 4.87     & 1.03     \\
        ParSTL(L) & 0.24                          & 67.41  & 0.50     & 4.48     & 1.41     \\
        ParSTL(R) & 0.26                          & 61.04  & 0.45     & 4.67     & 0.51     \\ \hline
      \end{tabular}
    }
    \subfloat[$\text{SNR} = 0.5$, $r=6$]{
      \begin{tabular}{ccccc}
        \hline
        Er(C)   & Er(XC) & FPR (\%) & FNR (\%) & Time (s) \\ \hline
        \multicolumn{5}{c}{$p = 100$}                     \\ \hline
        4788.37 & 478.79 & 100.00   & 0.00     & 0.09     \\
        4.09    & 302.34 & 56.12    & 0.00     & 7.26     \\
        29.85   & 855.91 & 22.83    & 7.75     & 22.13    \\
        4.83    & 289.24 & 2.45     & 9.01     & 1.56     \\
        6.02    & 402.34 & 4.82     & 6.50     & 8.55     \\
        4.71    & 269.37 & 8.38     & 6.11     & 7.45     \\
        3.76    & 215.73 & 2.37     & 10.90    & 0.76     \\
        5.83    & 380.29 & 3.19     & 8.87     & 1.17     \\
        3.57    & 188.62 & 3.17     & 8.03     & 0.13     \\ \hline
        \multicolumn{5}{c}{$p = 200$}                     \\ \hline
        91.45   & 472.36 & 100.00   & 0.00     & 0.11     \\
        1.70    & 287.74 & 65.43    & 0.00     & 14.38    \\
        20.80   & 744.53 & 21.12    & 7.25     & 25.80    \\
        2.87    & 340.08 & 1.56     & 10.45    & 5.17     \\
        2.42    & 316.90 & 3.58     & 6.79     & 9.50     \\
        2.25    & 273.33 & 4.51     & 8.90     & 8.45     \\
        2.25    & 250.90 & 1.53     & 12.55    & 1.01     \\
        2.45    & 306.59 & 2.14     & 9.33     & 1.11     \\
        2.46    & 244.62 & 2.18     & 8.96     & 0.18     \\ \hline
        \multicolumn{5}{c}{$p = 400$}                     \\ \hline
        66.92   & 478.95 & 100.00   & 0.00     & 0.43     \\
        0.99    & 355.65 & 57.26    & 0.00     & 40.66    \\
        29.44   & 554.55 & 26.58    & 7.45     & 10.38    \\
        1.80    & 404.84 & 0.92     & 10.43    & 30.04    \\
        1.22    & 305.48 & 2.42     & 7.29     & 12.99    \\
        1.43    & 306.86 & 2.68     & 9.18     & 14.19    \\
        1.42    & 281.50 & 0.94     & 12.76    & 1.77     \\
        1.57    & 349.37 & 1.36     & 9.97     & 1.66     \\
        1.56    & 278.45 & 1.36     & 9.49     & 0.54     \\ \hline
      \end{tabular}
    }
  }
\end{table}

\begin{table}[H]
  \caption{Results for Model III with $\text{SNR} = 0.5$, $p \in \{ 100, 200,400\}$ and $r^* \in \{3,6\}$; here Er(C) and Er(XC) are rescaled by multiplying $10^3$.}\label{tab:model-2}
  \centering
  \resizebox{\columnwidth}{!}{
    \subfloat[$\text{SNR} = 0.5$, $r=3$]{
      \begin{tabular}{lccccc}
        \hline
        Method    & Er(C)                          & Er(XC) & FPR (\%) & FNR (\%) & Time (s) \\ \hline
                  & \multicolumn{5}{c}{$p  = 100$}                                           \\ \hline
        RRR       & 5628.05                        & 241.69 & 100.00   & 0.00     & 0.08     \\
        SRRR      & 1.86                           & 148.12 & 62.70    & 0.00     & 4.82     \\
        SOFAR     & 3.88                           & 150.65 & 23.48    & 0.00     & 11.27    \\
        SeqACS    & 1.15                           & 76.76  & 1.07     & 0.31     & 0.38     \\
        ParACS(L) & 1.17                           & 73.86  & 4.12     & 0.00     & 8.11     \\
        ParACS(R) & 2.05                           & 115.89 & 9.37     & 0.64     & 7.05     \\
        SeqSTL    & 0.64                           & 43.43  & 1.23     & 0.57     & 0.35     \\
        ParSTL(L) & 0.94                           & 61.82  & 2.65     & 0.26     & 1.13     \\
        ParSTL(R) & 0.90                           & 57.82  & 3.22     & 1.12     & 0.15     \\ \hline
                  & \multicolumn{5}{c}{$p = 200$}                                            \\ \hline
        RRR       & 18.05                          & 246.27 & 100.00   & 0.00     & 0.12     \\
        SRRR      & 0.98                           & 166.66 & 66.31    & 0.00     & 6.54     \\
        SOFAR     & 1.82                           & 114.63 & 27.10    & 0.00     & 12.91    \\
        SeqACS    & 0.62                           & 83.42  & 0.73     & 0.62     & 1.15     \\
        ParACS(L) & 0.64                           & 80.90  & 2.87     & 0.00     & 9.02     \\
        ParACS(R) & 0.62                           & 75.61  & 2.81     & 0.00     & 8.25     \\
        SeqSTL    & 0.36                           & 47.35  & 0.67     & 0.67     & 0.43     \\
        ParSTL(L) & 0.53                           & 66.74  & 1.71     & 1.10     & 0.94     \\
        ParSTL(R) & 0.56                           & 67.50  & 2.09     & 0.76     & 0.17     \\ \hline
                  & \multicolumn{5}{c}{$p = 400$}                                            \\ \hline
        RRR       & 13.53                          & 247.87 & 100.00   & 0.00     & 0.46     \\
        SRRR      & 0.58                           & 205.47 & 51.19    & 0.00     & 20.32    \\
        SOFAR     & 4.09                           & 158.56 & 31.15    & 0.26     & 31.04    \\
        SeqACS    & 0.35                           & 95.23  & 0.40     & 0.33     & 7.20     \\
        ParACS(L) & 0.37                           & 94.57  & 1.73     & 0.00     & 12.67    \\
        ParACS(R) & 0.37                           & 89.29  & 1.89     & 0.00     & 12.65    \\
        SeqSTL    & 0.20                           & 52.19  & 0.36     & 0.67     & 0.96     \\
        ParSTL(L) & 0.29                           & 75.08  & 0.92     & 0.50     & 1.33     \\
        ParSTL(R) & 0.32                           & 75.86  & 1.16     & 0.81     & 0.53     \\ \hline
      \end{tabular}
    }
    \subfloat[$\text{SNR} = 0.5$, $r=6$]{
      \begin{tabular}{ccccc}
        \hline
        Er(C)    & Er(XC) & FPR (\%) & FNR (\%) & Time (s) \\ \hline
        \multicolumn{5}{c}{$p = 100$}                      \\ \hline
        10574.61 & 472.34 & 100.00   & 0.00     & 0.08     \\
        5.76     & 365.04 & 66.86    & 0.00     & 6.63     \\
        12.38    & 538.21 & 16.28    & 0.25     & 20.74    \\
        5.13     & 334.67 & 2.60     & 1.85     & 0.88     \\
        5.36     & 331.75 & 12.05    & 0.27     & 8.61     \\
        7.77     & 380.04 & 23.52    & 2.74     & 7.64     \\
        3.87     & 246.22 & 3.42     & 17.39    & 0.88     \\
        5.33     & 336.18 & 8.73     & 5.58     & 1.15     \\
        4.34     & 230.55 & 10.15    & 4.70     & 0.13     \\ \hline
        \multicolumn{5}{c}{$p = 200$}                      \\ \hline
        83.35    & 479.98 & 100.00   & 0.00     & 0.11     \\
        2.42     & 328.58 & 71.09    & 0.00     & 10.14    \\
        5.92     & 401.01 & 15.58    & 0.00     & 27.03    \\
        2.87     & 365.76 & 1.72     & 2.81     & 2.64     \\
        2.96     & 317.90 & 8.30     & 0.06     & 9.25     \\
        2.23     & 224.24 & 8.33     & 0.00     & 8.37     \\
        2.09     & 258.97 & 2.09     & 17.87    & 1.12     \\
        3.00     & 319.34 & 5.64     & 7.87     & 1.17     \\
        2.85     & 272.97 & 6.75     & 5.62     & 0.20     \\ \hline
        \multicolumn{5}{c}{$p = 400$}                      \\ \hline
        64.31    & 478.36 & 100.00   & 0.00     & 0.45     \\
        1.48     & 384.30 & 59.62    & 0.00     & 31.72    \\
        15.87    & 504.32 & 22.85    & 5.14     & 10.49    \\
        1.54     & 396.73 & 1.01     & 2.49     & 19.71    \\
        1.73     & 337.57 & 5.16     & 0.05     & 12.67    \\
        1.54     & 270.46 & 6.17     & 0.00     & 13.25    \\
        1.12     & 276.18 & 1.12     & 15.96    & 1.80     \\
        1.84     & 358.94 & 3.25     & 8.52     & 1.66     \\
        1.87     & 316.16 & 4.02     & 6.83     & 0.59     \\ \hline
      \end{tabular}
    }
  }
\end{table}

\doublespacing

\subsection{Impact of Step Size}

We investigate the impact of the step size $\epsilon$ in stagewise learning on the performance of the proposed methods.
The data are simulated from Model \upperroman{3} with $p=600$, $q=n=200$, $r^*=3$ and $\rho=0.3$. We consider different step sizes,
i.e., $\epsilon$ in $\left\{0.5, 1, 1.5, 2, 2.5\right\}$. The experiment is replicated 200 times. Figure \ref{fig:step} shows the boxplots of the estimation error,
prediction error, and computation time. The performance of the stagewise methods is stabilized when $\epsilon$ is small enough, i.e., $\epsilon \leq 1$ in this 
example. But the computational cost will increase if $\epsilon$ is too small and it is clear that there is a tradeoff between stepsize and accuracy.
In practice, we suggest to conduct some pilot numerical analysis to identify a proper step size.


\begin{figure}[H]
    {\includegraphics[width=.32\textwidth,keepaspectratio=true]{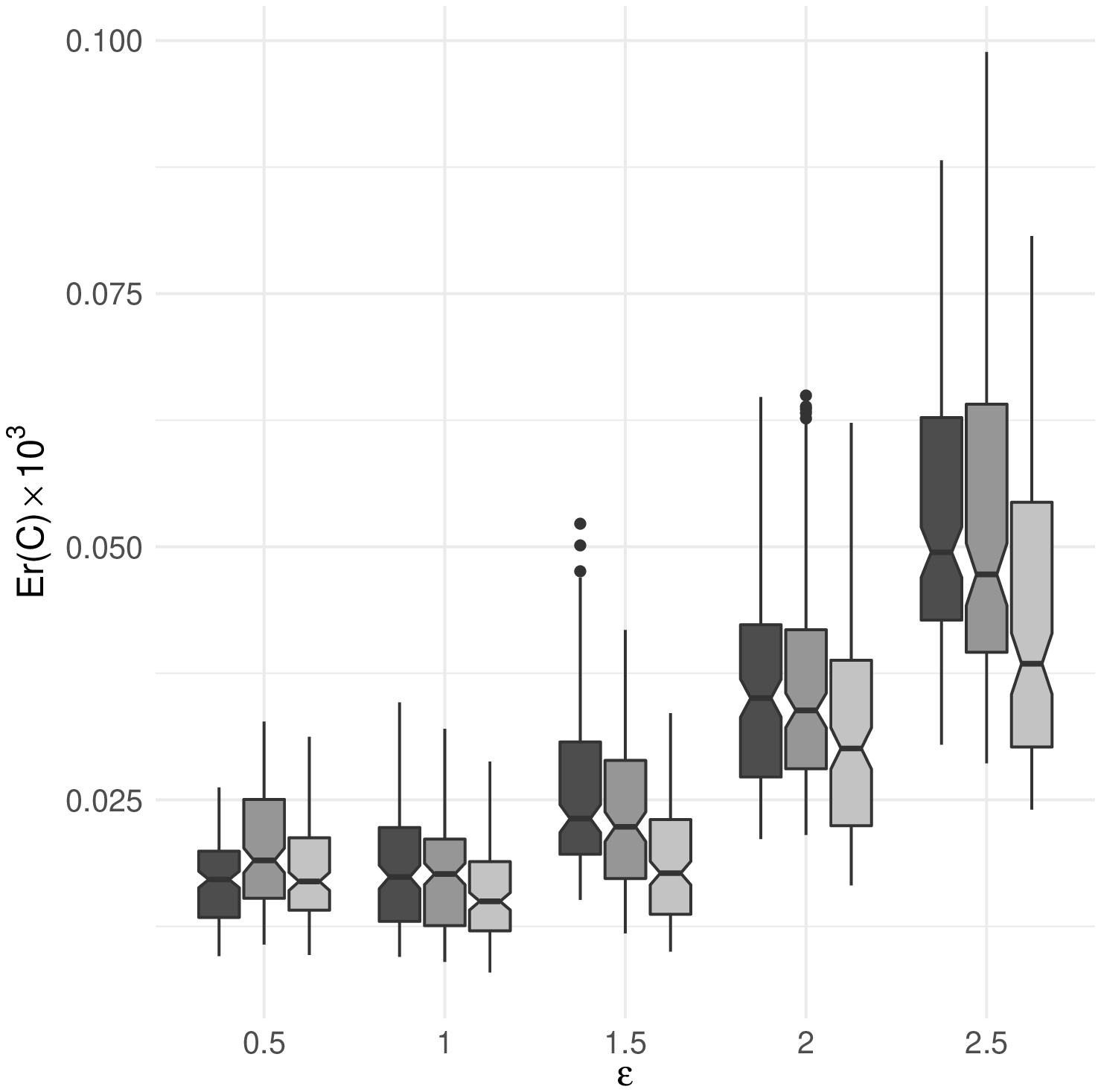}}
    {\includegraphics[width=.32\textwidth,keepaspectratio=true]{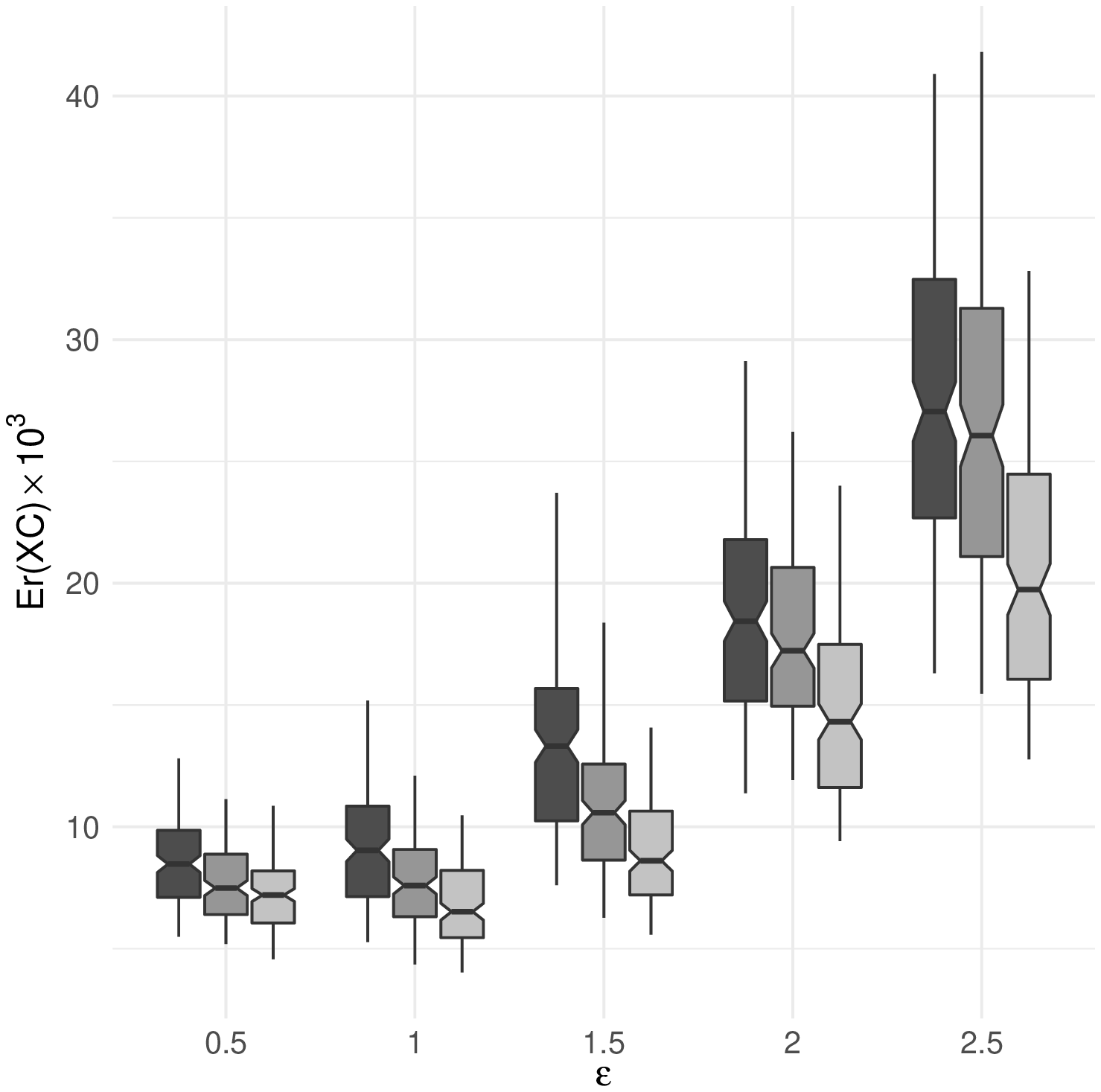}}
    {\includegraphics[width=.32\textwidth,keepaspectratio=true]{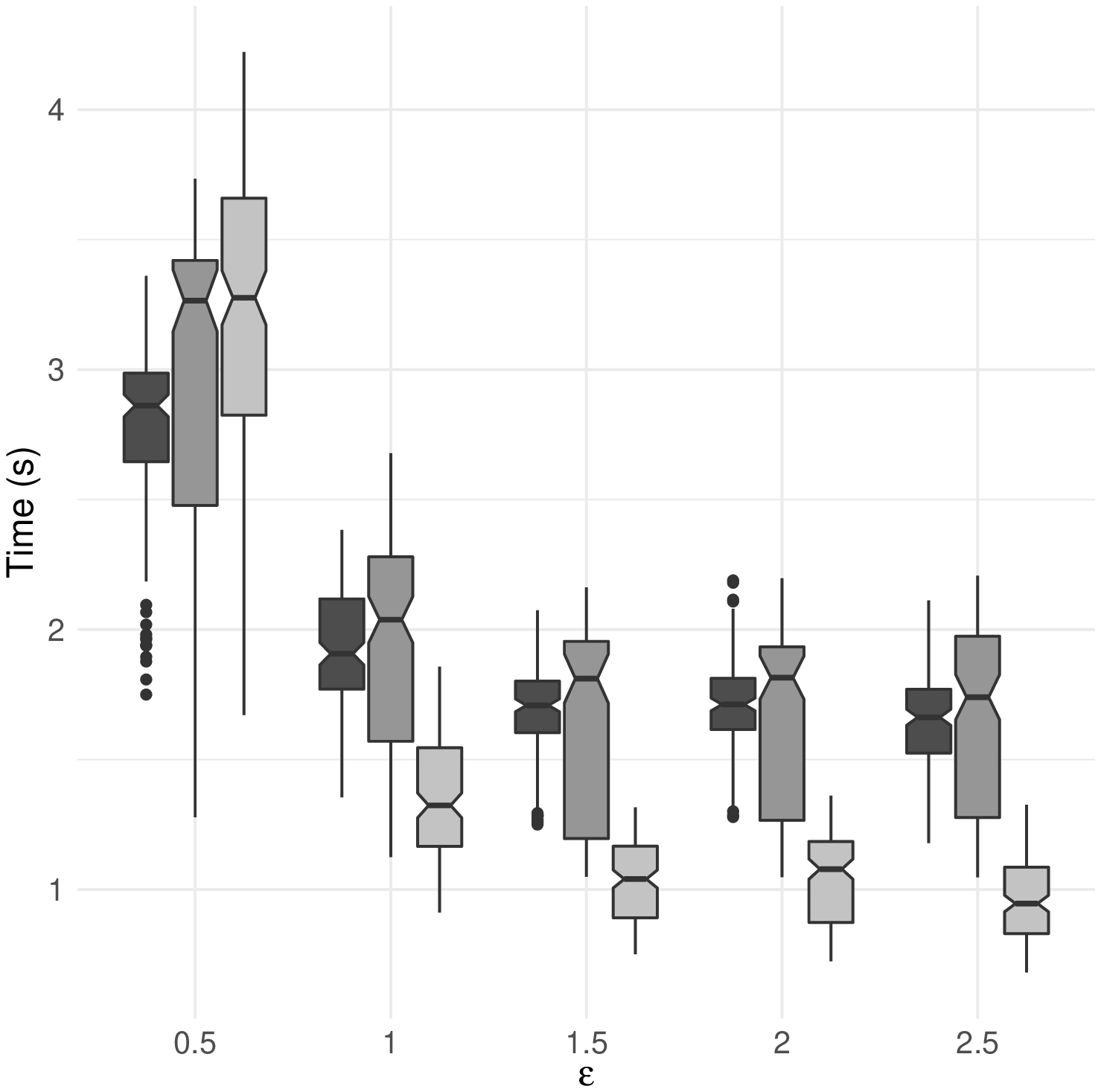}}
    \caption{Impact of the step size $\epsilon$ on the performance of the stagewise learning methods.The three colors, dark, medium and light grey correspond to ParSTL(L), ParSTL(R), and SeqSTL respectively.}\label{fig:step}
\end{figure}

\subsection{Computational Efficiency}\label{sec:sim:eff}

We report the computational time of our proposed methods as functions of $p$, $q$ and $n$ separately. The data are simulated from Model \upperroman{3} with $\text{SNR} = 0.25$, $r^*=3$ and $\rho=0.3$. In each setup, we let one of the three model dimensions, i.e., the sample size $n$, the number of predictors $p$, and the number of responses $q$, to vary from $200$ to $2000$, while holding the other two at a constant value of $200$. The experiment is repeated 100 times under each setting. Figure \ref{fig:time} reports the average computation times as functions of $p$, $q$ and $n$ in three panels from the left to the right, respectively. It is evident that our proposed approaches are scalable to large-scale problems, and the gain over the ACS-based approaches is dramatic.

\begin{figure}[H]
    \subfloat{\includegraphics[width=\textwidth,keepaspectratio=true]{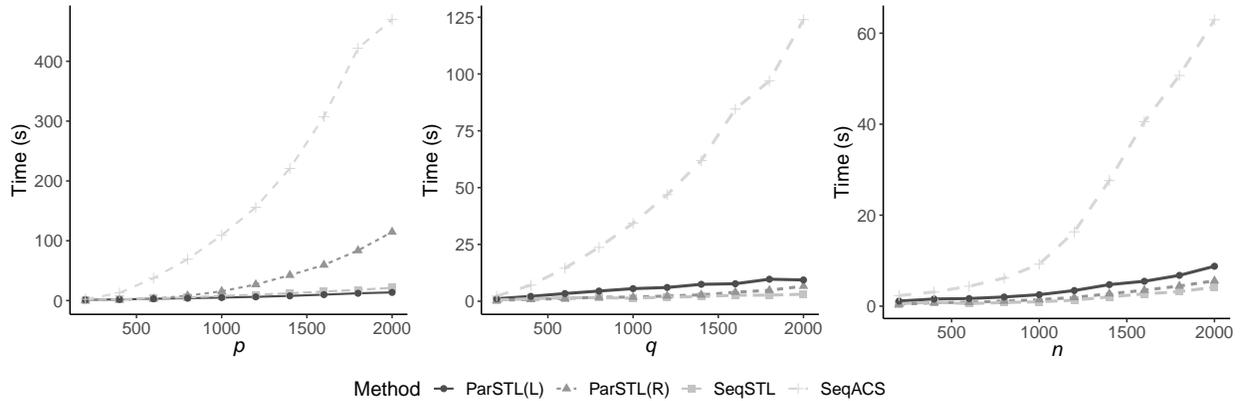}}
    \caption{Simulation: Computation cost with increasing $p$, $q$ and $n$.}\label{fig:time}
\end{figure}


\section{Yeast eQTL Mapping Analysis}\label{sec:yeast}

In an expression quantitative trait loci (eQTLs) mapping analysis, the main objective is to examine the association between the eQTLs, i.e., regions of the genome containing DNA sequence variants, and the expression levels of the genes in certain signaling pathways. Biochemical evidence often suggests that there exist a few functionally distinct signaling pathways of genes, each of which may involve only a subset of genes and correspondingly a subset of eQTLs. Therefore, the recovery of such association structure can be formulated as a sparse factor regression problem, with the gene expressions being the responses and the eQTLs being the predictors. Here, we analyze the yeast eQTL data set described by \citet{brem2005landscape} and \citet{storey2005multiple}, to illustrate the power and scalability of the proposed approaches for estimating the associations between $p=3244$ genetic markers and $q=54$ genes that belong to the yeast Mitogen-activated protein kinases (MAPKs) signaling pathway \citep{kanehisa2009kegg}, with data collected from $n=112$ yeast samples. 

\citet{uematsu2019sofar} used the same data set to showcase their SOFAR method. In their work, a marginal screening approach was first used to reduce the number of marker locations from $p = 3244$ to $p = 605$, which greatly alleviated the computational burden. However, since our proposed approaches are more scalable, it is worth trying to include all the markers for a joint regression. We thus try both approaches in this analysis.

Under either setting, i.e., with or without marginal screening, we perform a random splitting procedure to compare the performance of different methods. To make the comparison fair, all the methods have the same pre-specified rank, which is selected from RRR via 10-fold cross validation. Specifically, each time the data set is randomly split into $80\%$ for model fitting and $20\%$ for computing the out-sample mean squared error (MSE) of the fitted model. Also recorded are the computation time (in seconds), the number of nonzero entries in $\widehat{\U}$ ($\|\widehat{\U}\|_0$), the number of nonzero rows in $\widehat{\U}$ ($\|\widehat{\U}\|_{2,0}$), the number of nonzero entries in $\widehat{\V}$ ($\|\widehat{\V}\|_0$), and the number of nonzero rows in $\widehat{\V}$ ($\|\widehat{\V}\|_{2,0}$). The procedure is repeated 100 times, and to make a robust comparison we compute the $10\%$ trimmed means and standard deviations of the above performance measures.

Table \ref{tab:pre_trimmed} reports the results for the setting with marginal screening. All the methods that pursue sparse and low rank structures outperform the benchmark RRR in term of out-of-sample prediction performance; the ParSTL(R) method performs the best, although the improvement is not substantial comparing to other close competitors. Our proposed methods are no doubt the most computationally efficient among all the sparse and low-rank methods. The gain in computational efficiency is even more revealing and dramatic in the setting of $p = 3244$ without marginal screening, for which the results are reported in Table \ref{tab:full_trimmed}. For example, ParSTL can be more than 100 times faster than competitors such as RSSVD (similar to ParACS(R)) or SeqACS. It is also interesting to see that the predictive performance of all methods becomes slightly worse comparing to the setting with marginal screening. The performance of RRR deteriorates the most since it lacks the power of eliminating noise variables. The results suggest that in this particular application the potential benefit of jointly considering all the markers is exceeded by the loss due to noise accumulation. Nevertheless, we see that the proposed methods are still able to perform competitively in such a high-dimensional problem with excellent scalability.

\begin{table}[H]
\centering
\caption{Yeast eQTL mapping analysis: Results with marginal screening.}
\label{tab:pre_trimmed}
\scalebox{.75}{
\begin{tabular}{lcccccc}
\hline
Method    & $\|\widehat{\U}\|_0$ & $\|\widehat{\U}\|_{2,0}$ & $\|\widehat{\V}\|_0$ & $\|\widehat{\V}\|_{2,0}$ & MSE         & Time(s)       \\ \hline
RRR       & 1815 (0)             & 605 (0)                  & 162 (0)              & 54 (0)                   & 0.34 (0.03) & 0.52 (0.11)   \\
SOFAR     & 128.09 (23.85)       & 44.26 (7.64)             & 43.5 (7.38)          & 14.96 (1.41)             & 0.26 (0.03) & 85.35 (18.55) \\
RSSVD     & 61.39 (10.94)        & 58.49 (10.52)            & 6.74 (2.13)          & 6.2 (1.16)               & 0.3 (0.03)  & 38.14 (13.93) \\
SRRR      & 522.45 (18.28)       & 174.15 (6.09)            & 162 (0)              & 54 (0)                   & 0.24 (0.02) & 76.67 (13.78) \\
SeqACS    & 38.61 (3.34)         & 38.6 (3.35)              & 14.51 (1.44)         & 11.11 (0.93)             & 0.26 (0.02) & 29.84 (10.16) \\
SeqSTL    & 67.42 (8.03)         & 64.49 (7.52)             & 31.64 (4.28)         & 19.06 (1.9)              & 0.23 (0.02) & 3.26 (0.82)   \\
ParSTL(L) & 59.92 (9.39)         & 45.85 (6.21)             & 20.43 (2.61)         & 13.5 (1.66)              & 0.29 (0.1)  & 1.34 (0.33)   \\
ParSTL(R) & 72.71 (8.84)         & 69.62 (8.17)             & 34.50 (4.68)         & 21.25 (1.92)             & 0.21 (0.02) & 1.12 (0.21)   \\ \hline
\end{tabular}

}
\end{table}

\begin{table}[H]
\centering
\caption{Yeast eQTL mapping analysis: Results with full data (without marginal screening).}
\label{tab:full_trimmed}
\scalebox{.7}{
\begin{tabular}{lcccccc}
\hline
Method    & $\|\widehat{\U}\|_0$ & $\|\widehat{\U}\|_{2,0}$ & $\|\widehat{\V}\|_0$ & $\|\widehat{\V}\|_{2,0}$ & MSE         & Time(s)          \\ \hline
RRR       & 9732 (1788.02)       & 3244 (0)                 & 162 (29.76)          & 54 (0)                   & 0.44 (0.04) & 32.35 (2.60)   \\
SOFAR     & 156.62 (84.95)       & 61.1 (25.4)              & 40.09 (13.07)        & 16.02 (2.54)             & 0.28 (0.03) & 283.51 (102.01)  \\
RSSVD     & 52 (9.4)             & 51.09 (9.17)             & 5.4 (0.7)            & 5.35 (0.62)              & 0.3 (0.02)  & 1330.67 (541.91) \\
SRRR      & 1332.09 (359.29)     & 437.06 (44.94)           & 162 (29.76)          & 54 (0)                   & 0.24 (0.02) & 900.95 (248.2)   \\
SeqACS    & 38.74 (3.7)          & 38.67 (3.74)             & 13.29 (1.35)         & 10.31 (0.91)             & 0.26 (0.02) & 1378.98 (415.3)  \\
SeqSTL    & 57.56 (9.09)         & 56.46 (8.5)              & 24.18 (5.37)         & 15.65 (2.37)             & 0.24 (0.02) & 16 (5.04)        \\
ParSTL(L) & 57.69 (12.03)        & 43.67 (5.58)             & 19.21 (3.34)         & 12.36 (1.5)              & 0.31 (0.12) & 5.41 (1.7)       \\
ParSTL(R) & 67.46 (13.01)        & 64 (8.35)                & 30.91 (7.85)         & 19.25 (2.7)              & 0.22 (0.02) & 36.02 (3.16)     \\ \hline
\end{tabular}

}
\end{table}

Lastly, we examine the genes selected in the estimated pathways (low-rank components or layers) using our proposed methods with the full data set. Since each $\widehat{\v}_k$ vector is normalized to have unit $\ell_1$ norm, we define top genes to be those with entries larger than $1/q$ in magnitude. The selected genes are summarized in Table \ref{tab:gene_sum2}, in which the results in \citet{uematsu2019sofar} are reproduced for each of comparison. Figure \ref{fig:latent} shows the scatterplots of the latent responses $\Y\widehat{\v}_k$ versus the latent predictors $\X\widehat{\u}_k$ from the three stagewise methods fitted on the full data. The results from the three proposed methods are mostly consistent with \citet{uematsu2019sofar}. The patterns recovered by the three methods are similar, except that ParSTL(L) appears to reveal a slight different association comparing to the other two methods.
As explained in \citet{gustin1998map}, pheromone induces mating and nitrogen starvation induces filamentation are two pathways in yeast cells. The identified genes STE2, STE3 and GPA1 are receptors required for mating pheromone, which bind the cognate lipopeptide pheromones (MFA2, MFA1 etc). Besides, another top gene FUS3 is used to phosphorylate several downstream targets, including FAR1 which mediates various responses required for successful mating. Interestingly, our analysis identified an additional gene, TEC1, that was not reported in previous analysis. TEC1, as the transcription factor specific to the filamentation pathway, has been shown to have FUS3-dependent degradation induced by pheromone signaling \citep{bao2004pheromone}. 

\begin{table}[H]
\caption{Yeast eQTL mapping analysis: Top genes in the estimated pathways. The first row reproduces the results of \citet{uematsu2019sofar}. The first panel is for the setting with marginal screening, and the second panel is for the setting with full data.}\label{tab:gene_sum2}
\scalebox{0.65}{
\begin{tabular}{llll}
\hline
  Method    & Layer 1                                                                                   & Layer 2                                                                             & Layer 3
\\ \hline
 \multicolumn{4}{c}{With marginal screening ($p=605$)} 
  \\ \hline
\rowcolor[HTML]{EFEFEF} 
SOFAR     & STE3, STE2, MFA2, MFA1                                                                    & CTT1, SLN1, SLT2, MSN4, GLO1                                                        & FUS1, FAR1, STE2, STE3, GPA1, FUS3, STE12                                                                    \\
  SeqSTL    & \begin{tabular}[c]{@{}l@{}}STE3, STE2, MFA2, MFA1, \\ CTT1, FUS1\end{tabular}             & CTT1, FUS1, MSN4, GLO1, SLN1                                                        & \begin{tabular}[c]{@{}l@{}}FUS1, FAR1, STE2, GPA1, FUS3, STE3, TEC1, \\ SLN1, STE12, MFA2, CTT1\end{tabular} \\
\rowcolor[HTML]{EFEFEF} 
ParSTL(R) & \begin{tabular}[c]{@{}l@{}}STE3, STE2, MFA2, MFA1, \\ CTT1, FUS1, TEC1\end{tabular}       & \begin{tabular}[c]{@{}l@{}}CTT1, FUS1, MSN4, STE3, GLO1, \\ SLN1, MFA2\end{tabular} & \begin{tabular}[c]{@{}l@{}}FUS1, FAR1, STE2, GPA1, FUS3, STE3, TEC1, \\ CTT1, SLN1, STE12\end{tabular}       \\
  ParSTL(L) & \begin{tabular}[c]{@{}l@{}}STE3, STE2, MFA2, MFA1, \\ CTT1, FUS1\end{tabular}             & CTT1, FUS1, MSN4, GLO1                                                              & CTT1, MSN4, STE2                                                                                             
\\ \hline
 \multicolumn{4}{c}{Full data ($p=3244$)} 
\\ \hline  
SeqSTL    & \begin{tabular}[c]{@{}l@{}}STE3, STE2, MFA2, MFA1, \\ CTT1, FUS1\end{tabular}             & \begin{tabular}[c]{@{}l@{}}FUS1, CTT1, FAR1, SLN1, STE2, \\ MFA2, GPA1\end{tabular} & TEC1                                                                                                         \\
\rowcolor[HTML]{EFEFEF} 
ParSTL(R) & \begin{tabular}[c]{@{}l@{}}STE3, STE2, MFA2, MFA1, \\ CTT1, FUS1, TEC1, WSC2\end{tabular} & \begin{tabular}[c]{@{}l@{}}CTT1, FUS1, MSN4, GLO1, MFA2,\\  STE3, SLN1\end{tabular} & \begin{tabular}[c]{@{}l@{}}FUS1, FAR1, STE2, STE3, GPA1, FUS3, \\ TEC1, CTT1\end{tabular}                    \\
ParSTL(L) & \begin{tabular}[c]{@{}l@{}}STE3, STE2, MFA2, MFA1, \\ CTT1, FUS1\end{tabular}             & CTT1, FUS1, MSN4                                                                    & CTT1, MSN4, STE2                                                                                             \\ \hline
\end{tabular}
}
\end{table}

\begin{figure}[H]
\subfloat[Layer 1]{\includegraphics[width=.33\columnwidth]{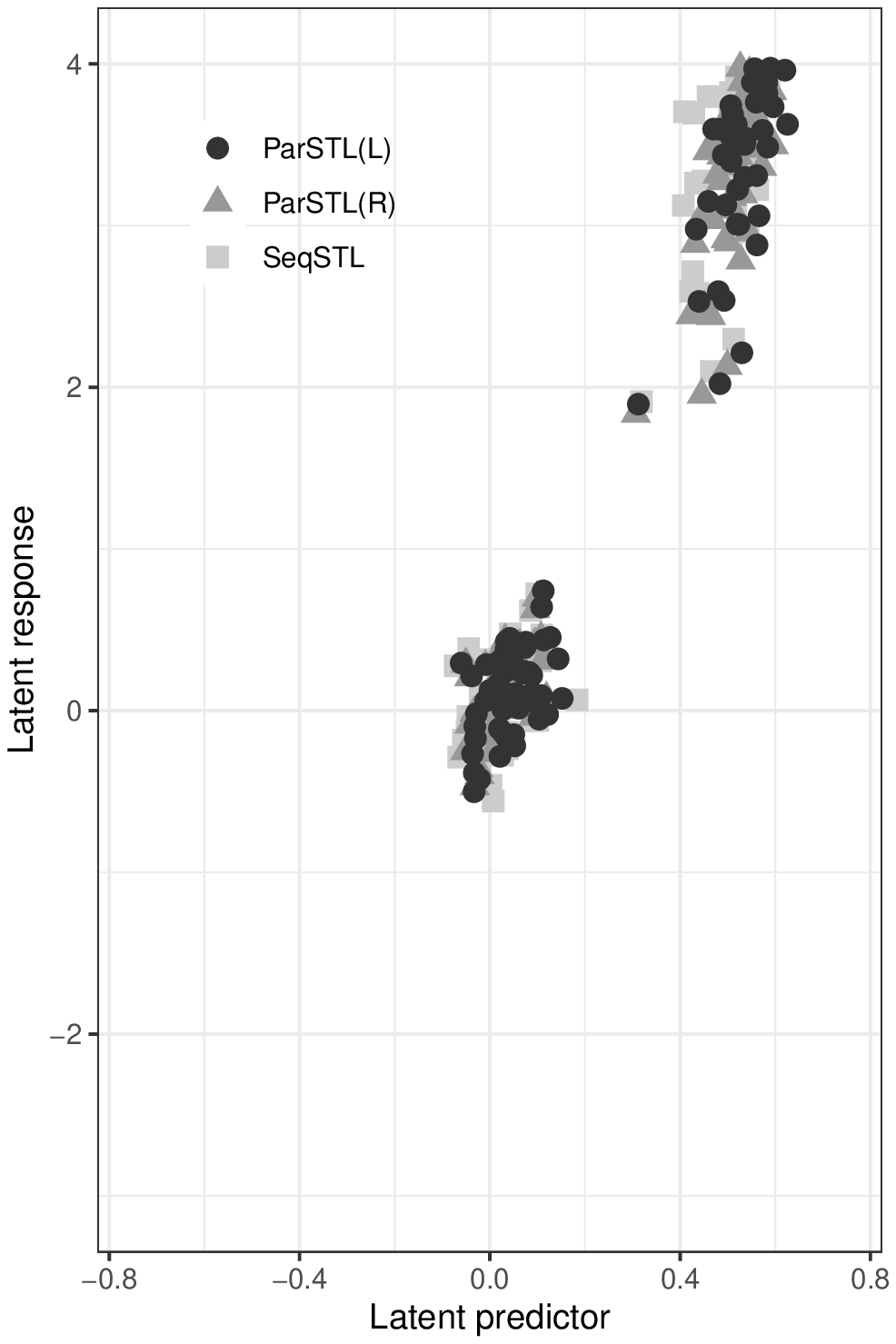}}
\subfloat[Layer 2]{\includegraphics[width=.33\columnwidth]{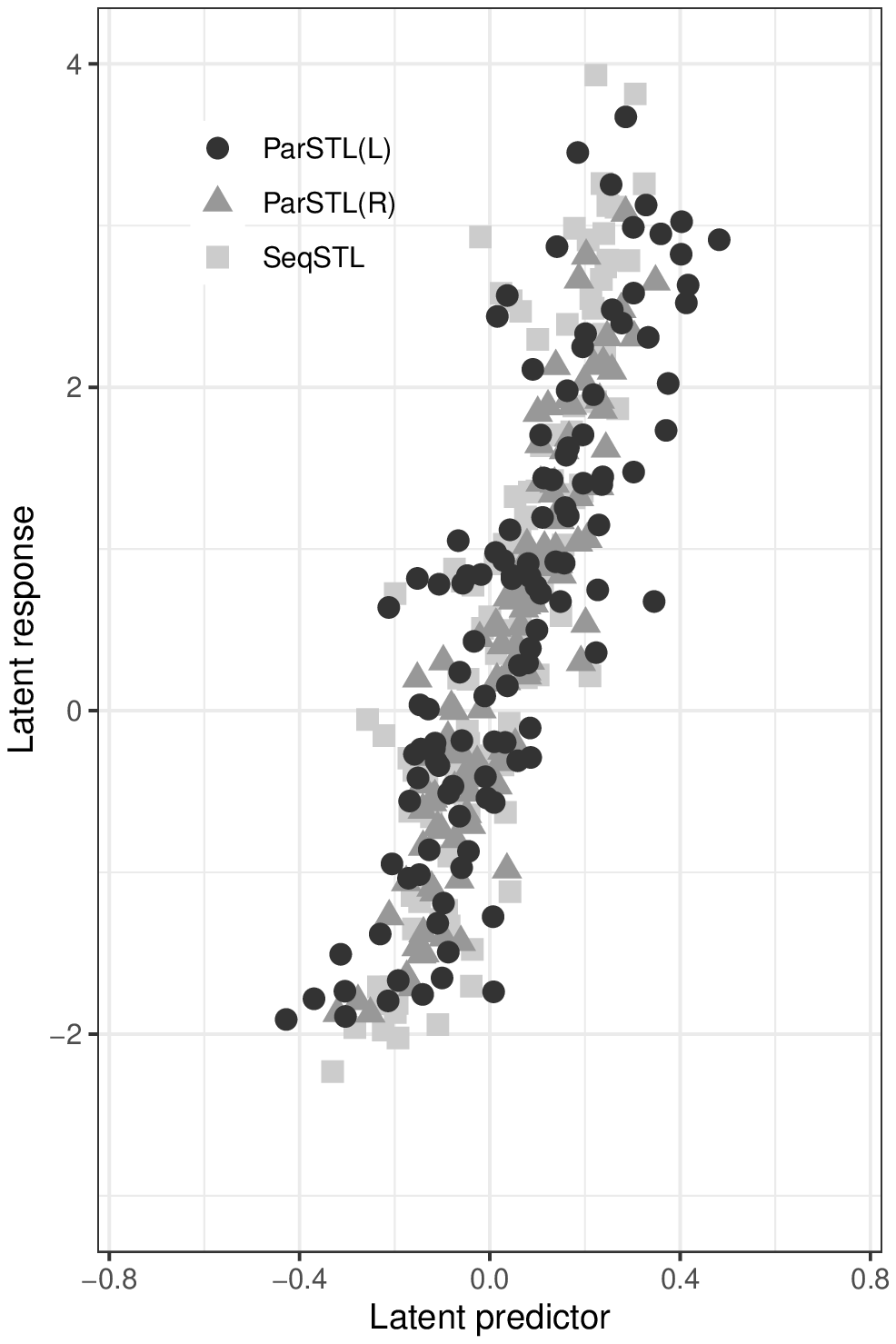}}
\subfloat[Layer 3]{\includegraphics[width=.33\columnwidth]{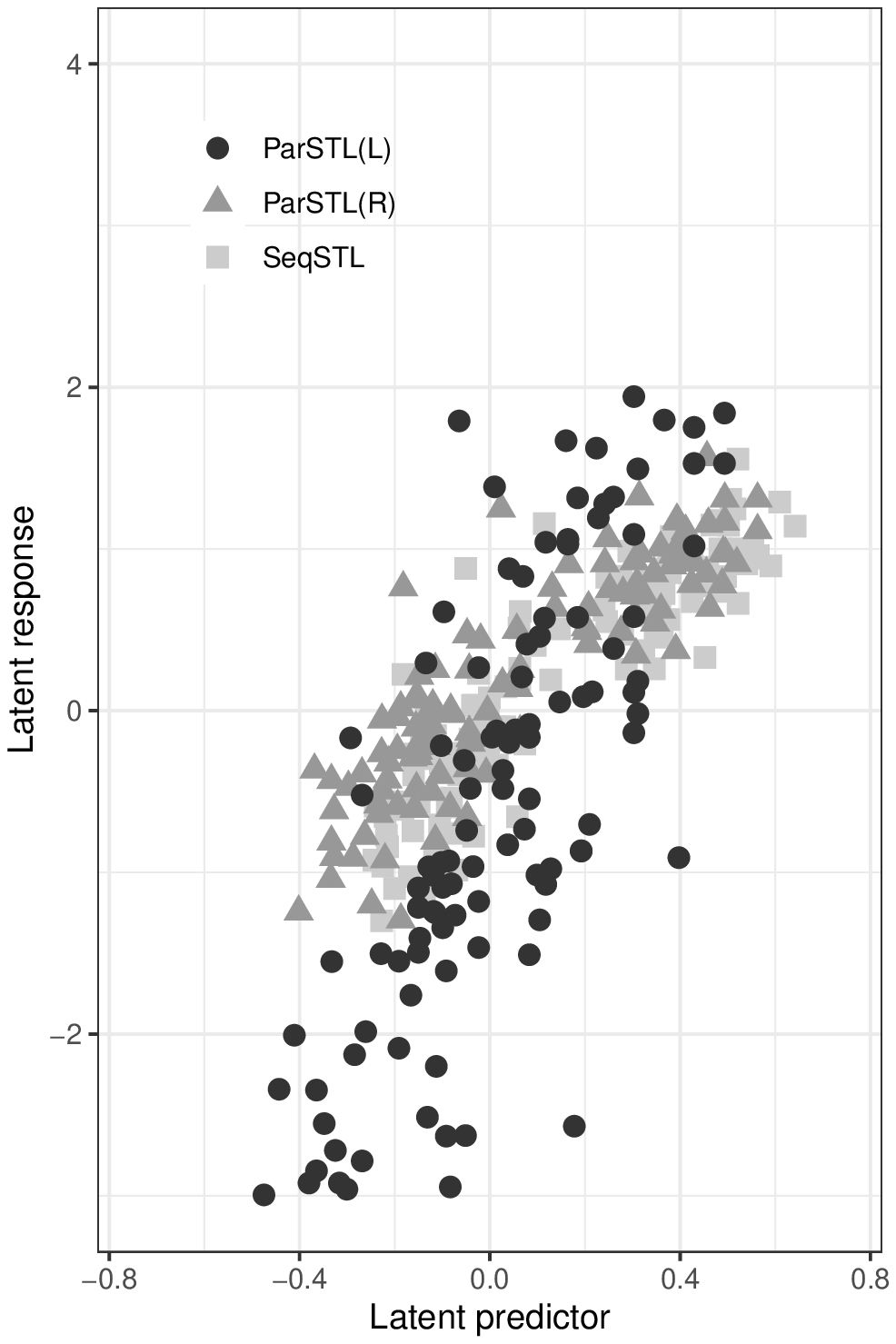}}
\caption{Yeast eQTL mapping analysis: Scatter plots of the estimated latent responses and latent predictors.}
\label{fig:latent}
\end{figure}

\section{Discussion} \label{sec6}
We have proposed a statistically guided divide-and-conquer approach for sparse factorization of large matrix. Both sequential and parallel deflation strategies are analyzed with corresponding statistical underpinnings. Moreover, a contended stagewise learning technique is developed to efficiently trace out the whole solution paths, which enjoys a much lower computational complexity than the alternating convex search. Our work is among the first to enable stagewise learning for non-convex problems, and extensive numerical studies demonstrate the effectiveness and scalability of our approach. There are several future research directions. Our approach can be applicable in many multi-convex problems including sparse factorization of a tensor \citep{HeChen2018}, and it can also be extended to more general model settings such as generalized linear models and mixture models. Building upon the architecture of the contended stagewise learning procedure, it is hopeful that a general framework of non-convex stagewise learning can be developed.  
\section*{Acknowledgments} \label{acknowledgments}
Chen's research is partially supported by NSF grants DMS-1613295 and IIS-1718798. Zheng's research is supported by National Natural Science Foundation of China grants 11601501, 11671374, and 71731010, and Fundamental Research Funds for the Central Universities grant WK2040160028. 
\bibliography{rui-bibtex}

\clearpage
\appendix

 \begin{center}
 {\Large\bf Supplementary Materials}
 \end{center}

  We present the proofs of the theoretical results in this section. For matrices $\bf A$ and $\bf B$, $\inner{{\bf A},{\bf B}}$ means the Frobenius inner product of ${\bf A}$ and ${\bf B}$, $\norm{{\bf A}}_F$ is the Forbenius norm of ${\bf A}$, and $\norm{{\bf A}}_1 = \sum_{i,j}\abs{A_{ij}}$. Moreover, denote by $J_k$ the support set of $\C_k^*$ and $\widehat{\bDelta}_k = \widehat{\C}_k - \C_k^*$ with $\widehat{\C}_{k}$ the $k$th layer estimator of $\C_k^*$ in either sequential or parallel pursuit.

  \section{Proof for the sequential pursuit}
  Recall that the sequential pursuit with lasso penalty sequentially solves
  \begin{align}\label{opt_se}
    (\widehat{d}_{k},\widehat{\u}_{k},\widehat{\v}_{k})
     & = \underset{(d,\u,\v)}{\arg\min}~(2n)^{-1}\norm{\Y_{k} - d\X\u\v\trans}_{F}^{2}
    + \lambda_{k}\norm{d\u\v\trans}_{1} \nonumber                                      \\
     & ~\mbox{s.t. } d > 0,n^{-1}\u\trans\X\trans\X\u=\v\trans\v=1,
  \end{align}
  where $\Y_{k} = \Y -\X\sum_{\ell=0}^{k-1}\widehat{\C}_{\ell}$, $\widehat{\C}_{0}=\0$ and
  $\widehat{\C}_{\ell} = \widehat{d}_{\ell}\widehat{\u}_{\ell}\widehat{\v}_{\ell}\trans$. We need two lemmas before showing the main theorem.

  \subsection{Lemma \ref{se_lemma1} and its proof}

  \begin{lemma}\label{se_lemma1}
    Under Condition \ref{cond3}, for any $k$, $1 \leq k \leq r^{*} - 1$, it holds that
    \begin{equation*}
      \frac{\norm{\X\widehat{\bDelta}_{k}}_{F}^{2}}
      {2\abs{\sum_{j=k+1}^{r^{*}}\inner{\X\C_{j}^{*}, \X\widehat{\C}_{k}}}} \geq \omega_{k}^{-1} > 1,
    \end{equation*}
    where $\omega_{k}=1-\delta_{k}^*/d_{k}^* > 0$.
  \end{lemma}

  \begin{proof}[Proof of Lemma \ref{se_lemma1}]
    Since the following argument applies to any $k$, $1 \leq k \leq r^{*}-1$, we consider a fixed $k$ here. First of all,
    there exists the following decomposition for $\X\widehat{\u}_{k}$ and $\widehat{\v}_{k}$ that
    \begin{align}\label{se_lemma1_eq1}
      \frac{1}{\sqrt{n}}\X\widehat{\u}_{k} = \sum_{j=1}^{r^{*} + 1}a_{j}\frac{1}{\sqrt{n}}\X\u_{j}^{*}, \ \
      \widehat{\v}_{k} = \sum_{j=1}^{r^{*} + 1}b_{j}\v_{j}^{*},
    \end{align}
    where $\frac{1}{\sqrt{n}}\X\u_{r^{*} + 1}^{*}$ and $\v_{r^{*} + 1}^{*}$ are some unit length vectors orthogonal to the sets of unit
    length orthogonal vectors $\{\frac{1}{\sqrt{n}}\X\u_{j}^{*}\}_{j=1}^{r^{*}}$ and $\{\v_{j}^{*}\}_{j=1}^{r^{*}}$, respectively,
    so that the coefficients satisfy $\sum_{j=1}^{r^{*}+1}a_{j}^{2} = \sum_{j=1}^{r^{*}+1}b_{j}^{2} = 1$.

    \smallskip

    If $a_k = 1$ or $b_k = 1$, it is a relatively trivial case. So we consider cases where $a_k < 1$ and $b_k < 1$. Based on decomposition \eqref{se_lemma1_eq1}, direct calculation yields that
    \begin{align*}
      n^{-1}\norm{\X\widehat{\bDelta}_{k}}_{F}^{2} = \widehat{d}_{k}^{2} - 2\widehat{d}_{k}d_{k}^{*}a_{k}b_{k} + d_{k}^{*2}, \\
      n^{-1}\sum_{j=k+1}^{r^{*}}\inner{\X\C_{j}^{*},\X\widehat{\C}_{k}} =
      \sum_{j=k+1}^{r^{*}}d_{j}^{*}\widehat{d}_{k}a_{j}b_{j}.
    \end{align*}
    It gives
    \begin{align*}
       & \frac{\norm{\X\widehat{\bDelta}_{k}}_{F}^{2}}{2\abs{\sum_{j=k+1}^{r^{*}}\inner{\X\C_{j}^{*},\X\widehat{\C}_{k}}}}
      = \frac{\widehat{\tau}_{k}^{2} - 2\widehat{\tau}_{k}a_{k}b_{k} + 1}
      {2\widehat{\tau}_{k}\abs{\sum_{j=k+1}^{r^{*}}(d_{j}^{*}/d_{k}^{*})a_{j}b_{j}}}                                       \\
       & = \frac{\widehat{\tau}_{k} - 2a_{k}b_{k} + \widehat{\tau}_k^{-1}}
      {2\abs{\sum_{j=k+1}^{r^{*}}(d_{j}^{*}/d_{k}^{*})a_{j}b_{j}}}
      \geq \frac{1-a_{k}b_{k}}{\abs{\sum_{j=k+1}^{r^{*}}(d_{j}^{*}/d_{k}^{*})a_{j}b_{j}}},
    \end{align*}
    where $\widehat{\tau}_{k} = \widehat{d}_{k} / d_{k}^{*} > 0$ and we make use of the fact that
    $\widehat{\tau}_{k} + \widehat{\tau}_k^{-1} \geq 2$.

    Then by Condition \ref{cond3}, we have
    \begin{equation*}
      \max_{k+1\leq j \leq r^*}d_{j}^{*}/d_{k}^{*} = d_{k+1}^{*} / d_{k}^{*}
      = (d_{k}^* - \delta_{k}^*) / d_{k}^* = 1 - \delta_{k}^* / d_{k}^* = \omega_{k} < 1.
    \end{equation*}
    Denote by $\widetilde{\omega}_{j} = d_j^* / d_k^*$. It follows that
    \begin{equation*}
      \begin{aligned}
        \frac{(1 - a_{k}b_{k})^{2}}{(\sum_{j=k+1}^{r^{*}}\widetilde{\omega}_ja_{j}b_{j})^{2}}
         & \geq \frac{(1 - a_{k}b_{k})^{2}}{(\sum_{j=k+1}^{r^{*}}\widetilde{\omega}_ja_{j}^{2})(\sum_{j=k+1}^{r^{*}}\widetilde{\omega}_jb_{j}^{2})} \\
         & \geq \frac{(1 - a_{k}b_{k})^{2}}{\omega_{k}^2(\sum_{j=k+1}^{r^{*}}a_{j}^{2})(\sum_{j=k+1}^{r^{*}}b_{j}^{2})}
        \geq \frac{(1-a_{k}b_{k})^{2}}{\omega_{k}^2(1-a_{k}^{2})(1-b_{k}^{2})} \geq \omega_{k}^{-2}.
      \end{aligned}
    \end{equation*}

    Thus, it yields
    \begin{equation*}
      \frac{\norm{\X\widehat{\bDelta}_{k}}_{F}^{2}}{2\abs{\sum_{\ell=k+1}^{r^{*}}\inner{\X\C_{\ell}^{*},\X\widehat{\C}_{k}}}} \geq \frac{1-a_{k}b_{k}}{\abs{\sum_{j=k+1}^{r^{*}}(d_{j}^{*}/d_{k}^{*})a_{j}b_{j}}}
      \geq \omega_{k}^{-1} > 1,
    \end{equation*}
    which concludes the results of Lemma \ref{se_lemma1}.
  \end{proof}

  \subsection{Lemma \ref{se_lemma2} and its proof}

  \begin{lemma}\label{se_lemma2}
    Under Conditions \ref{cond2}--\ref{cond1} with $s \geq s_1$ and $\lambda_{1} = 2\sigma_{\max} \sqrt{2\alpha\log(pq)/n}$ for some constant $\alpha > 1$, we have with probability at least
    $1 - (pq)^{1-\alpha}$,
    \begin{align*}
       & \norm{\widehat{\bDelta}_{1}}_{F} = O(\sqrt{s_{1}} \lambda_{1}) = O\left(\theta_1\sqrt{s_{1}\log(pq)/n}\right), \\
       & \norm{\widehat{\bDelta}_{1}}_{1} = O(s_{1} \lambda_{1}) = O\left(\theta_1s_{1}\sqrt{\log(pq)/n}\right),
    \end{align*}
    where $\theta_1 = (1-\omega_1)^{-1} = d_{1}^*/\delta_{1}^*$. 
  \end{lemma}

  \begin{proof}[Proof of Lemma \ref{se_lemma2}]
    Because $(\widehat{d}_{1},\widehat{\u}_{1}, \widehat{\v}_{1})$ is the optimal solution of \eqref{opt_se} for $k=1$, we have
    \begin{equation*}
      (2n)^{-1}\norm{\Y -\X\widehat{\C}_{1} }_{F}^{2} + \lambda_{1}\norm{\widehat{\C}_{1}}
      \leq (2n)^{-1}\norm{\Y -\X\C_{1}^{*}}_{F}^{2} + \lambda_{1}\norm{\C_{1}^{*}}_{1},
    \end{equation*}
    where $\widehat{\C}_{1} = \widehat{d}_{1}\widehat{\u}_{1}\widehat{\v}_{1}\trans $ and $\C_{1}^{*} = d_{1}^{*}\u_{1}^{*}\v_{1}\strans$. After some simplification, the above inequality gives that
    \begin{equation*}
      (2n)^{-1}\norm{\X\widehat{\bDelta}_{1}}_{F}^{2} + \lambda_{1}\norm{\widehat{\C}_{1}}_{1}
      \leq n^{-1}\inner{\X\trans \E,\widehat{\bDelta}_{1}} + n^{-1}\sum_{k=2}^{r^{*}}\inner{\X\C_{k}^{*},\X\widehat{\C}_{1}} +
      \lambda_{1}\norm{\C_{1}^{*}}_{1}.
    \end{equation*}

    Applying Lemma \ref{se_lemma1}, we get
    \begin{align}\label{proof_se_lemma2_inq1}
      (2n)^{-1}(1- \omega_1)\norm{\X\widehat{\bDelta}_{1}}_{F}^{2} + \lambda_{1}\norm{\widehat{\C}_{1}}_{1}
      \leq n^{-1}\inner{\X\trans \E,\widehat{\bDelta}_{1}} + \lambda_{1}\norm{\C_{1}^{*}}_{1} \nonumber \\
      \leq n^{-1}\norm{\X\trans \E}_{\max}\norm{\widehat{\bDelta}}_{1} + \lambda_{1}\norm{\C_{1}^{*}}_{1},
    \end{align}
    where the last inequality holds because of the H\"{o}lder's inequality.
    Then on event $\mathcal{A} = \{n^{-1}\norm{\X\trans \E}_{\max} \leq \lambda_{1}/2\}$, by inequality \eqref{proof_se_lemma2_inq1}, we have
    \begin{equation*}
      \lambda_{1}\norm{\widehat{\C}_{1}}_{1} \leq \frac{\lambda_{1}}{2}\norm{\widehat{\bDelta}_{1}}_{1} + \lambda_{1}\norm{\C_{1}^{*}}_{1}.
    \end{equation*}
    It yields that
    \begin{equation}\label{proof_se_lemma2_inq2}
      \norm{\widehat{\bDelta}_{J_1^c}}_{1} \leq 3\norm{\widehat{\bDelta}_{J_1}}_{1}.
    \end{equation}
    Therefore, by Condition \ref{cond1}, we can get
    \begin{equation}\label{proof_se_lemma2_inq3}
      \rho_{l}(\norm{\widehat{\bDelta}_{J_{1}}}_{F}^{2}\vee\norm{\widehat{\bDelta}_{J_{1,s_{1}}^{c}}}_{F}^{2})
      \leq n^{-1}\norm{\X\widehat{\bDelta}_{1}}_{F}^{2},
    \end{equation}
    where $\widehat{\bDelta}_{J_{1,s_{1}}^{c}}$ is the submatrix of
    $\widehat{\bDelta}_{J_{1}^{c}}$ as defined in Condition \ref{cond1}. Our discussion will be conditioning on the event $\mathcal{A}$ hereafter.

    On the other hand, combining inequalities \eqref{proof_se_lemma2_inq1} and \eqref{proof_se_lemma2_inq2} yields that
    \begin{equation*}
      (1 - \omega_1)(2n)^{-1}\norm{\X\widehat{\bDelta}_{1}}_{F}^{2} \leq (3\lambda_{1}/2)\norm{\widehat{\bDelta}_{1}}_{1} \leq
      6\lambda_{1}\norm{\widehat{\bDelta}_{J_1}}_{1}.
    \end{equation*}
    Together with \eqref{proof_se_lemma2_inq3}, we have
    \begin{equation}\label{proof_se_lemma2_inq4}
      \frac{1-\omega_1}{2}\rho_{l}(\norm{\widehat{\bDelta}_{J_{1}}}_{F}^{2}\vee\norm{\widehat{\bDelta}_{J_{1,s_{1}}^{c}}}_{F}^{2})
      \leq (1 - \omega_1)(2n)^{-1}\norm{\X\widehat{\bDelta}_{1}}_{F}^{2}
      \leq 6\lambda_{1}\norm{\widehat{\bDelta}_{J_1}}_{1}
      \leq 6\lambda_{1}\sqrt{s_{1}}\norm{\widehat{\bDelta}_{J_1}}_{F}.
    \end{equation}
    It follows that
    \begin{equation}\label{proof_se_lemma2_inq5}
      \norm{\widehat{\bDelta}_{J_{1}}}_{F} \leq 12(\rho_{l} - \rho_{l}\omega_1)^{-1}\lambda_{1}\sqrt{s_{1}}.
    \end{equation}

    Note that the $k$th largest component of $\widehat{\bDelta}_{J_{1}^{c}}$ in terms of absolute value is bounded from above by
    $\norm{\widehat{\bDelta}_{J_{1}^{c}}}_{1}/k$, then we have
    \begin{equation*}
      \norm{\widehat{\bDelta}_{\tilde{J}_{1,s_{1}}^{c}}}_{F}^{2} \leq
      \sum_{k=s_{1}+1}^{pq-s_{1}}\norm{\widehat{\bDelta}_{J_{1}^{c}}}_{1}^{2}/k^{2}
      \leq s_{1}^{-1}\norm{\widehat{\bDelta}_{J_{1}^{c}}}_{1}^{2},
    \end{equation*}
    where $\widehat{\bDelta}_{\tilde{J}_{1,s_{1}}^{c}}$ is a submatrix of $\widehat{\bDelta}_{J_{1}^{c}}$
    consisting of the components excluding those with the $s_1$ largest absolute values. Then by inequalities \eqref{proof_se_lemma2_inq2} and \eqref{proof_se_lemma2_inq5},
    we get
    \begin{equation}\label{proof_se_lemma2_inq6}
      \norm{\widehat{\bDelta}_{\tilde{J}_{1,s_{1}}^{c}}}_{F} \leq
      s_{1}^{-1/2}\norm{\widehat{\bDelta}_{J_{1}^{c}}}_{1} \leq
      3s_{1}^{-1/2}\norm{\widehat{\bDelta}_{J_{1}}}_{1} \leq
      3\norm{\widehat{\bDelta}_{J_{1}}}_{F} \leq
      36(\rho_{l} - \rho_{l}\omega_1)^{-1}\lambda_{1}\sqrt{s_{1}}.
    \end{equation}

In view of inequalities \eqref{proof_se_lemma2_inq4} and \eqref{proof_se_lemma2_inq5}, we can also get
    \begin{equation}\label{proof_se_lemma2_inq7}
      \norm{\widehat{\bDelta}_{J_{1,s_{1}}^{c}}}_{F} \leq
      \{12\lambda_{1}(\rho_{l}-\rho_{l}\omega_1)^{-1}\sqrt{s_{1}}\norm{\widehat{\bDelta}_{J_{1}}}_{F}\}^{1/2} \leq
      12(\rho_{l}-\rho_{l}\omega_1)^{-1}\lambda_{1}\sqrt{s_{1}}.
    \end{equation}
    Finally, combining \eqref{proof_se_lemma2_inq5}
    , \eqref{proof_se_lemma2_inq6} and \eqref{proof_se_lemma2_inq7}, we have
    \begin{equation*}
      \norm{\widehat{\bDelta}_{1}}_{F} \leq
      \norm{\widehat{\bDelta}_{J_{1}}}_{F} + \norm{\widehat{\bDelta}_{\tilde{J}_{1,s_{1}}^{c}}}_{F}
      + \norm{\widehat{\bDelta}_{J_{1,s_{1}}^{c}}}_{F} \leq 5M_1\sqrt{s_1}\lambda_1 =
      O( \theta_1\sqrt{s_{1}} \lambda_{1}) = O\left(\theta_1\sqrt{s_{1}\log(pq)/n}\right),
    \end{equation*}
    where $M_1 = 12(\rho_{l}-\rho_{l}\omega_1)^{-1}$ and $\theta_1 = (1-\omega_1)^{-1}$. For $\ell_{1}$ loss bound, from inequalities
    \eqref{proof_se_lemma2_inq2} and \eqref{proof_se_lemma2_inq5},
    we directly get
    \begin{equation*}
      \norm{\widehat{\bDelta}_{1}}_{1}  \leq 4\norm{\widehat{\bDelta}_{J_{1}}}_{1}
      \leq 4\sqrt{s_{1}}\norm{\widehat{\bDelta}_{J_{1}}}_{F} \leq 4 M_1 s_{1} \lambda_{1} =
      O(\theta_{1}s_{1} \lambda_{1}) = O\left(\theta_{1}s_{1}\sqrt{\log(pq)/n}\right).
    \end{equation*}

    At last, we will derive the probability of event $\mathcal{A}$. First, it follows from the union bound that
    \begin{equation*}
      P\{n^{-1}\norm{\X\trans \E}_{\max} > \lambda_{1}/2\} \leq
      \sum_{i=1}^{p}\sum_{j=1}^{q}P\{n^{-1}\abs{\x_{i}\trans \e_{j}} > \lambda_{1}/2\}.
    \end{equation*}
    Since $\text{Var}(\x_{i}\trans \e_{j}) = \sigma_{j}^{2}\norm{\x_{i}}_{2}^{2} \leq n\sigma_{\max}^{2}$, under Condition \ref{cond2}, applying the tail probability bound of Gaussian distribution, we get
    \begin{equation}\label{gaus}
      P\{n^{-1}\norm{\X\trans \E}_{\max} > \lambda_{1}/2\} \leq 2pq\exp\{-\frac{n\lambda_{1}^2}{8\sigma_{\max}^2}\}.
    \end{equation}
    Therefore, when $\lambda_{1}^{2} = 8\alpha\sigma_{\max}^{2}n^{-1}\log(pq)$, we know that event $\mathcal{A}$ holds with probability at least $1 - 2(pq)^{1-\alpha}$. It completes the proof of Lemma \ref{se_lemma2}. 
  \end{proof}

  \subsection{Proof of Theorem \ref{se_th}}

  When $k=1$, the results of Theorem \ref{se_th} are established in Lemma \ref{se_lemma2}. Now we do some preparation before showing the results for $k>1$.
  Note that
  \begin{equation*}
    \begin{aligned}
      \norm{\Y_{k} - d\X\u\v\trans }_{F}^{2}
       & = \norm{\Y - \X\sum_{\ell=0}^{k-1}\widehat{\C}_{\ell} - d\X\u\v\trans }_{F}^{2}
      = \norm{\widetilde{\Y}_{k} - \sum_{\ell=0}^{k-1}\X\widehat{\bDelta}_{\ell} - d\X\u\v\trans }_{F}^{2}                                 \\
       & = \norm{\widetilde{\Y}_{k} - d\X\u\v\trans }_{F}^{2} + 2\sum_{\ell=0}^{k-1}\inner{d\X\u\v\trans ,\X\widehat{\bDelta}_{\ell}} + T,
    \end{aligned}
  \end{equation*}
  where $\widetilde{\Y}_{k} = \X\sum_{\ell=k}^{r^{*}}\C_{\ell} + \E$, $\widetilde{\bDelta}_{\ell} = \widehat{\C}_{\ell} - \C_{\ell}^{*}$
  and $T =  -2\sum_{\ell = 0}^{k-1}\inner{\X\widehat{\bDelta}_{\ell},\widetilde{\Y}_k}$.
  Because $T$ doesn't change with the triple $(d,\u,\v)$, the optimization problem \eqref{opt_se} is equivalent to
  \begin{equation}\label{opt_se_equivalent}
    \begin{aligned}
      (\widehat{d}_{k},\widehat{\u}_{k},\widehat{\v}_{k}) = \underset{d,\u,\v}{\arg\min}~
       & (2n)^{-1}\norm{\widetilde{\Y}_{k} - d\X\u\v\trans }_{F}^{2} + n^{-1}\sum_{\ell=0}^{k-1}\inner{d\X\u\v\trans ,\X\widehat{\bDelta}_{\ell}} \\
       & +\lambda_{k}\norm{d\u\v\trans}_{1}.
    \end{aligned}
  \end{equation}

  Because $(\widehat{d}_{k},\widehat{\u}_{k},\widehat{\v}_{k})$ is the optimal solution of \eqref{opt_se_equivalent}, we have
  \begin{equation*}
    \begin{aligned}
      (2n)^{-1} & \norm{\widetilde{\Y}_{k} - \X\widehat{\C}_{k}}_{F}^{2} + n^{-1}\sum_{\ell=0}^{k-1}\inner{\X\widehat{\C}_{k},\X\widehat{\bDelta}_{\ell}}
      + \lambda_{k}\norm{\widehat{\C}_{k}}_{1}                                                                                                              \\
                & \leq (2n)^{-1}\norm{\widetilde{\Y}_{k} - \X\C_{k}^{*}}_{F}^{2} + n^{-1}\sum_{\ell=0}^{k-1}\inner{\X\C_{k}^{*},\X\widehat{\bDelta}_{\ell}}
      + \lambda_{k}\norm{\C_{k}^{*}}_{1}.
    \end{aligned}
  \end{equation*}
  Plugging in $\widetilde{\Y}_{k} - \X\widehat{\C}_{k} = \widetilde{\Y}_{k} - \X\C_{k}^{*} - \X\widehat{\bDelta}_{k}$, direct calculation yields that
  \begin{equation*}
    \begin{aligned}
      (2n)^{-1} & \norm{\X\widehat{\bDelta}_{k}}_{F}^{2} + n^{-1}\sum_{\ell=0}^{k-1}\inner{\X\widehat{\bDelta}_{k},\X\widehat{\bDelta}_{\ell}} + \lambda_{k}\norm{\widehat{\C}_{k}}_{1} \\
                & \leq n^{-1}\inner{\X\widehat{\bDelta}_{k},\E}
      + n^{-1}\sum_{\ell=k+1}^{r^{*}}\inner{\X\widehat{\bDelta}_{k},\X\C_{\ell}^{*}} +
      \lambda_{k}\norm{\C_{k}^{*}}_{1}.
    \end{aligned}
  \end{equation*}

  Utilizing Lemma \ref{se_lemma1}, we can derive that for $k = 2,\dots, r^{*}$,
  \begin{equation}\label{proof_se_th_ineq2}
    \begin{aligned}
      (2n)^{-1}(1 - \omega_k)\norm{\X\widehat{\bDelta}_{k}}_{F}^{2}
       & + n^{-1}\sum_{\ell=0}^{k-1}\inner{\X\widehat{\bDelta}_{k},\X\widehat{\bDelta}_{\ell}} + \lambda_{k}\norm{\widehat{\C}_{k}}_{1} \\
       & \leq n^{-1}\inner{\X\widehat{\bDelta}_{k},\E}
      + \lambda_{k}\norm{\C_{k}^{*}}_{1}.
    \end{aligned}
  \end{equation}
  Hereafter, our discussion will be conditioning on the following event
  $$\mathcal{A} = \{n^{-1}\norm{\X\trans \E}_{\max} \leq \lambda_{1}/2\},$$
  which has been shown to hold with a significant probability in the proof of Lemma \ref{se_lemma2}.

  \smallskip

  We then prove the results for $k > 1$ by mathematical induction. Assume that
  \begin{align*}
     & \norm{\widehat{\bDelta}_{J_{\ell}^c}}_1 \leq 3\norm{\widehat{\bDelta}_{J_{\ell}}}_{1}, \ \ \norm{\widehat{\bDelta}_{J_\ell}}_F \leq M_{\ell}\sqrt{s_{\ell}}\lambda_{\ell},                \\
     & \norm{\widehat{\bDelta}_{\tilde{J}_\ell^c,s_\ell}}_F \leq 3M_{\ell}\sqrt{s_\ell}\lambda_\ell, \ \ \norm{\widehat{\bDelta}_{J_\ell^c,s_\ell}}_F \leq M_{\ell}\sqrt{s_{\ell}}\lambda_{\ell}
  \end{align*}
  hold for $1 \leq \ell \leq k-1$, where
  $M_\ell = 12(\rho_l - \rho_l\omega_\ell)^{-1}$, $\widehat{\bDelta}_{J_\ell^c,s_\ell}$ is the submatrix of $\widehat{\bDelta}_{J_\ell^c}$
  as defined in Condition \ref{cond1}, and $\widehat{\bDelta}_{\tilde{J}_\ell^c,s_\ell}$ is the submatrix of
  $\widehat{\bDelta}_{J_\ell^c}$ consisting of the components excluding those with the $s_\ell$ largest absolute values.

  First of all, by Condition \ref{cond4}, we get
  \begin{equation*}
    n^{-1}\norm{\X\trans\X\widehat{\bDelta}_{\ell}}_{\max} \leq s_{\ell}^{-1/2}\phi_u\norm{\widehat{\bDelta}_{J_\ell}}_F
    \leq \phi_u M_\ell \lambda_\ell,
  \end{equation*}
  where the last inequality holds because of $\norm{\widehat{\bDelta}_{J_\ell}}_F \leq M_\ell\sqrt{s_{\ell}}\lambda_{\ell}$. Define the
  $k$th regularization parameter inductively as
  \begin{equation}\label{eq:lambda}
    \lambda_k = 2(2^{-1}\lambda_1 + \phi_u \sum_{\ell=1}^{k-1} M_\ell \lambda_\ell)
    = \lambda_1 + 2\phi_u M \sum_{\ell=1}^{k-1}\theta_\ell\lambda_\ell = \lambda_1 + \sum_{\ell=1}^{k-1}\eta_{\ell}\lambda_{\ell},
  \end{equation}
  where $M = 12\rho_{l}^{-1}$, $\theta_\ell = (1-\omega_\ell)^{-1} = d_{\ell}^*/\delta_{\ell}^*$, and $\eta_{\ell}=24\phi_u\rho_{l}^{-1}\theta_\ell = c \theta_\ell$ with $c = 24\phi_u\rho_{l}^{-1}$.
  Then by the same argument as inequalities \eqref{proof_se_lemma2_inq1} and \eqref{proof_se_lemma2_inq2} in the proof of Lemma \ref{se_lemma2}, we can get
  \begin{equation}\label{induction:1}
    \norm{\widehat{\bDelta}_{J_k^c}}_1 \leq 3\norm{\widehat{\bDelta}_{J_k}}_1.
  \end{equation}

  Further applying Condition \ref{cond1} gives
  \begin{equation*}
    \rho_l (\norm{\widehat{\bDelta}_{J_k}}_F^2 \vee \norm{\widehat{\bDelta}_{J_k^c,s_k}}_F^2) \leq n^{-1}\norm{\X\widehat{\bDelta}_k}_F^2.
  \end{equation*}
  Then by inequality \eqref{proof_se_th_ineq2} and the same argument as inequality \eqref{proof_se_lemma2_inq4}, we have
  \begin{equation*}
    (2n)^{-1}(1-\omega_k)\norm{\X\widehat{\bDelta}_{k}}_{F}^{2} \leq
    \frac{3}{2}\lambda_{k}\norm{\widehat{\bDelta}_{k}}_{1} \leq
    6\lambda_{k}\norm{\widehat{\bDelta}_{J_{k}}}_{1} \leq
    6\lambda_{k}\sqrt{s_{k}}\norm{\widehat{\bDelta}_{J_{k}}}_{F}.
  \end{equation*}
  Combining these two inequalities yields that
  \begin{equation}\label{induction:2}
    \norm{\widehat{\bDelta}_{J_k}}_F \leq 12(\rho_l - \rho_l\omega_k)^{-1}\sqrt{s_k}\lambda_k = M_k\sqrt{s_k}\lambda_k.
  \end{equation}

  Then by the same argument as inequalities \eqref{proof_se_lemma2_inq6} and \eqref{proof_se_lemma2_inq7}, we can get
  \begin{equation*}
    \begin{aligned}
       & \norm{\widehat{\bDelta}_{\tilde{J}_{k,s_k}^c}}_F \leq s_k^{-1/2}\norm{\widehat{\bDelta}_{J_k^c}}_1
      \leq 3s_k^{-1/2}\norm{\widehat{\bDelta}_{J_k}}_1 \leq 3\norm{\widehat{\bDelta}_{J_k}}_F,              \\
       & \norm{\widehat{\bDelta}_{J_{k,s_k}^c}}_F \leq
      \{12\lambda_{k}(\rho_{l}-\rho_{l}\omega_k)^{-1}\sqrt{s_{k}}\norm{\widehat{\bDelta}_{J_{k}}}_{F}\}^{1/2}.
    \end{aligned}
  \end{equation*}
  It yields that
  \begin{equation}\label{induction:3}
    \begin{aligned}
       & \norm{\widehat{\bDelta}_{\tilde{J}_{k,s_k}^c}}_F
      \leq 36(\rho_l - \rho_l\omega_k)^{-1}\sqrt{s_k}\lambda_k = 3M_k\sqrt{s_k}\lambda_k,                                                      \\
       & \norm{\widehat{\bDelta}_{J_{k,s_k}^c}}_F \leq 12(\rho_{l}-\rho_{l}\omega_k)^{-1}\sqrt{s_{k}}\lambda_{k} = M_k\sqrt{s_{k}}\lambda_{k}.
    \end{aligned}
  \end{equation}
  Therefore, we can derive that inequalities \eqref{induction:1}--\eqref{induction:3} hold for $2\leq k \leq r^*$ by mathematical induction.
  By the same argument as that in the proof of Lemma \ref{se_lemma2}, inequalities \eqref{induction:1}--\eqref{induction:3}
  give
  \begin{equation*}
    \begin{aligned}
       & \norm{\widehat{\bDelta}_k}_{F} \leq \norm{\widehat{\bDelta}_{J_k}}_{F} + \norm{\widehat{\bDelta}_{J_k^c,s_k}}_{F}
      + \norm{\widehat{\bDelta}_{\tilde{J}_k^c,s_k}}_{F} \leq 5M_k\sqrt{s_k}\lambda_k = O(\theta_k\sqrt{s_k}\lambda_k),         \\
       & \norm{\widehat{\bDelta}_k}_{1} \leq 4\norm{\widehat{\bDelta}_{J_k}}_1 \leq 4\sqrt{s_k}\norm{\widehat{\bDelta}_{J_k}}_F
      \leq 4M_ks_k\lambda_k = O(\theta_k s_k \lambda_k).
    \end{aligned}
  \end{equation*}

  Finally, we derive an explicit form of the regularization parameter $\lambda_k$ defined in \eqref{eq:lambda}.  By definition, we have
  \begin{equation*}
    \lambda_{k} = \lambda_1 + \sum_{\ell=1}^{k-1}\eta_{\ell}\lambda_{\ell}, \ \ \lambda_{k - 1} = \lambda_1 + \sum_{\ell=1}^{k-2}\eta_{\ell}\lambda_{\ell}.
  \end{equation*}
  It yields that
  \begin{equation*}
    \lambda_{k} - \lambda_{k - 1} = \eta_{k - 1}\lambda_{k - 1},
  \end{equation*}
  which gives
  \begin{equation*}
    \lambda_{k} = (1 + \eta_{k - 1})\lambda_{k - 1}.
  \end{equation*}
  Thus, we can derive $\lambda_k = \Pi_{\ell=1}^{k-1}(1+\eta_\ell)\lambda_1$ by mathematical induction, which concludes the proof of Theorem \ref{se_th}.

  \section{Proof for the parallel pursuit}

  Recall that the parallel pursuit is as follows
  \begin{equation}\label{para_each_layer_estimator}
    \begin{aligned}
      (\widehat{d}_{k},\widehat{\u}_{k},\widehat{\v}_{k})
       & = \underset{(d,\u,\v)}{\arg\min}~ (2n)^{-1}\norm{\Y_{k} - d\X\u\v\trans}_{F}^{2} + \lambda_{k}\norm{d\u\v\trans}_{1} \\
       & \mbox{s.t. } d > 0, n^{-1}\u\trans\X\trans\X\u = \v\trans\v = 1.
    \end{aligned}
  \end{equation}
  where $\Y_{k} = \Y - \X\sum_{\ell\neq k}\widetilde{\C}_{\ell}$ with the initial $\ell$th layer estimator $\widetilde{\C}_{\ell}$ that by definition takes the $s$ largest components of $\widetilde{\C}_{k}^0$ in terms of absolute values while sets the others to be zero. Here $\widetilde{\C}_{\ell}^0$ is $\ell$the unit rank matrix of the initial Lasso estimator $\widetilde{\C}$, which is generated from
  \begin{equation}\label{para_init_estimator_lasso}
    \begin{aligned}
      \widetilde{\C} = \underset{\C}{\arg\min}~ (2n)^{-1}\norm{\Y - \X\C}_{F}^{2} + \lambda_{0}\norm{\C}_{1}.
    \end{aligned}
  \end{equation}
  We need two lemmas before showing the main results of the parallel pursuit.

  \subsection{Lemma \ref{para_th_init} and its proof}

  \begin{lemma}\label{para_th_init}
    When Conditions \ref{cond2} and \ref{cond1} hold with the sparsity level $s \geq s_0$, $\lambda_{0} = 2\sigma_{\max} \sqrt{2\alpha\log(pq)/n}$ for some constant $\alpha > 1$, we have with probability at least $1 - (pq)^{1-\alpha}$, the initial Lasso estimator $\widetilde{\C}$ satisfies
    \begin{equation*}
      \begin{aligned}
        \norm{\widetilde{\bDelta}}_{F} = O(\sqrt{s_0\log(pq)/n}),~ n^{-1/2}\norm{\X\widetilde{\bDelta}}_F = O(\sqrt{s_0\log(pq)/n}),
      \end{aligned}
    \end{equation*}
    where $s_0 = \norm{\C^*}_{0}$ and $\widetilde{\bDelta} = \widetilde{\C} - \C^{*}$.
  \end{lemma}

  \begin{proof}[Proof of Lemma \ref{para_th_init}]
    By the optimality of $\widetilde{\C}$, we have
    \begin{equation*}
      (2n)^{-1}\norm{\X\widetilde{\bDelta}}_{F}^{2} + \lambda_{0}\norm{\widetilde{\C}}_{1} \leq \inner{\widetilde{\bDelta}, \X\trans \E} + \lambda_{0}\norm{\C^{*}}_{1}.
    \end{equation*}
    Applying the same argument as inequalities \eqref{proof_se_lemma2_inq1} and \eqref{proof_se_lemma2_inq2} in the proof of Lemma \ref{se_lemma2}, conditioning on the event $\mathcal{A} = \{\norm{\X\trans \E}_{\max} \leq \lambda_{0}/2\}$, we can get
    \begin{equation*}
      \norm{\widetilde{\bDelta}_{J^{c}}}_{1} \leq 3\norm{\widetilde{\bDelta}_{J}}_{1}.
    \end{equation*}
    Combining these two inequalities gives that
    \begin{equation}\label{ineq1_proof_para_th_init}
      n^{-1}\norm{\X\widetilde{\bDelta}}_{F}^{2} \leq 3\lambda_{0}\norm{\widetilde{\bDelta}}_{1}
      \leq 12\lambda_{0}\norm{\widetilde{\bDelta}_{J}}_{1} \leq 12\sqrt{s_0}\lambda_{0}\norm{\widetilde{\bDelta}_{J}}_{F}.
    \end{equation}

    On the other hand, by Condition \ref{cond1}, we have
    \begin{equation}\label{ineq2_proof_para_th_init}
      \rho_{l}(\norm{\widetilde{\bDelta}_{J}}_{F}^{2}\vee\norm{\widetilde{\bDelta}_{J_{s_0}^c}}_{F}^{2}) \leq n^{-1}\norm{\X\widetilde{\bDelta}}_{F}^{2}.
    \end{equation}
    Therefore, together with \eqref{ineq1_proof_para_th_init}, it yields
    \begin{equation*}
      \begin{aligned}
        \norm{\widetilde{\bDelta}_{J}}_{F} \leq 12\rho_{l}^{-1}\sqrt{s_0}\lambda_{0}, \ \
        \norm{\widetilde{\bDelta}_{J_{s_0}^c}}_{F} \leq 12\rho_{l}^{-1}\sqrt{s_0}\lambda_{0}.
      \end{aligned}
    \end{equation*}
    By the same argument as that in the proof of Lemma \ref{se_lemma2}, when $\lambda_{0} = 2\sigma_{\max} \sqrt{2\alpha\log(pq)/n}$, the following results hold with probability at least $1-(pq)^{1-\alpha}$,
    \begin{equation*}
      \norm{\widetilde{\bDelta}}_{F} = O(\sqrt{s_0}\lambda_0), \ \
      n^{-1/2}\norm{\X\widetilde{\bDelta}}_F = O(\sqrt{s_0}\lambda_0).
    \end{equation*}
    It concludes the proof of Lemma \ref{para_th_init}.
  \end{proof}

  \subsection{Lemma \ref{para_init_svd} and its proof}

  \begin{lemma}\label{para_init_svd}
    When $\lambda_{0} = 2\sigma_{\max} \sqrt{2\alpha\log(pq)/n}$ for some constant $\alpha > 1$, Conditions \ref{cond2}--\ref{cond1} hold with the sparsity level $s \geq s_0$ , then there exists some positive constant $M$ such that with probability at least $1 - (pq)^{1-\alpha}$, the following inequality holds uniformly over $1 \leq k \leq r^*$,
    \begin{equation*}
      \norm{\widetilde{\C}_k^0 - \C_k^*}_F \le M \psi_k \sqrt{s_0\log(pq)/n},
    \end{equation*}
    where $\psi_k = d_1^*d_c^*/(d_k^*\min[\delta_{k-1}^*, \delta_{k}^*])$ with $d_c^*$ the largest singular value of $\C^*$.
  \end{lemma}

  \begin{proof}[Proof of Lemma \ref{para_init_svd}]

    Note that $\widetilde{\v}_k$ and $\v_k^*$ are also the unit length right singular vectors of $\X\widetilde{\C}$ and $\X\C^*$, respectively, $\widetilde{\v}_k^T \v_k^* \geq 0$, and
    we have $\widetilde{d}_k\widetilde{\u}_k = \widetilde{\C}\widetilde{\v}_k$ and $d_k^*\u_k^* = \C^*\v_k^*$. It follows that
    \begin{equation*}
      \begin{aligned}
        \widetilde{\C}_k^0 - \C_k^* & = \widetilde{d}_k\widetilde{\u}_k\widetilde{\v}_k\trans - d_k^*\u_k^*\v_k^*\trans = d_k^*\u_k^*(\widetilde{\v}_k - \v_k^*)\trans + (\widetilde{d}_k\widetilde{\u}_k - d_k^*\u_k^*)\widetilde{\v}_k\trans \\
                                    & = d_k^*\u_k^*(\widetilde{\v}_k - \v_k^*)\trans +
        \left[\C^*(\widetilde{\v}_k - \v_k^*) + (\widetilde{\C}- \C^*)\widetilde{\v}_k\right]\widetilde{\v}_k\trans.
      \end{aligned}
    \end{equation*}
    Thus, we have
    \begin{equation}\label{para_init_svd_ineq1}
      \begin{aligned}
        \norm{\widetilde{\C}_k^0 - \C_k^*}_F & \leq d_k^*\norm{\u_k^*}_2\norm{\widetilde{\v}_k - \v_k^*}_2
        + \norm{\C^*(\widetilde{\v}_k - \v_k^*)}_2 + \norm{(\widetilde{\C}- \C^*)\widetilde{\v}_k}_2                          \\
                                             & \leq d_k^*\norm{\u_k^*}_2\norm{\widetilde{\v}_k - \v_k^*}_2
        + d_c^*\norm{\widetilde{\v}_k - \v_k^*}_2 + \norm{\widetilde{\C}- \C^*}_F                                             \\
                                             & \leq 2d_c^*\norm{\widetilde{\v}_k - \v_k^*}_2 + \norm{\widetilde{\C}- \C^*}_F,
      \end{aligned}
    \end{equation}
    where $d_c^*$ is the largest singular value of $\C^*$ and the last inequality makes use of the fact that $d_k^*\norm{\u_k^*}_2 \leq d_c^*$ for any $k$, $1\leq k \leq r^*$, due to the following inequality
    \begin{equation*}
      d_k^*\norm{\u_k^*}_2 = \frac{\u_k^*\C^*\v_k*}{\norm{\u_k^*}_2} \leq \max_{\norm{\u}_2=\norm{\v}_2=1}\u\trans\C^*\v
      =d_c^*.
    \end{equation*}

    Moreover, applying \citet[Theorem 3]{yu2014useful}, we get
    \begin{equation*}
      \norm{\widetilde{\v}_k - \v_k^*}_2 \leq
      \frac{2^{3/2}(2 d_1^* + n^{-1/2}\norm{\X\widehat{\bDelta}}_{F})n^{-1/2}\norm{\X\widehat{\bDelta}}_{F}}
      {\min(d_{k-1}^{*2} - d_k^{*2}, d_{k}^{*2} - d_{k+1}^{*2})}
    \end{equation*}
    for any $k$, $1 \leq k \leq r^*$ with $d_0^* = +\infty$.
    By Condition \ref{cond3}, we have
    \begin{align*}
      d_{k - 1}^{*2} - d_{k}^{*2} & = (d_{k - 1}^* + d_{k}^*)(d_{k - 1}^* - d_{k}^*) = (d_{k - 1}^* + d_{k}^*)\delta_{k - 1}^* > d_{k}^* \delta_{k - 1}^*, \\
      d_{k}^{*2} - d_{k+1}^{*2}   & = (d_{k}^* + d_{k+1}^*)(d_{k}^* - d_{k+1}^*) = (d_{k}^* + d_{k+1}^*)\delta_{k}^* > d_{k}^* \delta_{k}^*,
    \end{align*}
    which yields
    \begin{equation*}
      \min(d_{k-1}^{*2} - d_k^{*2}, d_{k}^{*2} - d_{k+1}^{*2}) > d_k^*\min[\delta_{k-1}^*, \delta_{k}^*].
    \end{equation*}

    Then together with the results of Lemma \ref{para_th_init}, it follows that
    \begin{equation}\label{para_init_svd_ineq2}
      \begin{aligned}
        \norm{\widetilde{\v}_k - \v_k^*}_2 & = O\Big\{
        \frac{d_1^* n^{-1/2}\norm{\X\widehat{\bDelta}}_{F}}
        {d_k^*\min[\delta_{k-1}^*, \delta_{k}^*]}\Big\} = O(\widetilde{\psi}_{k}\sqrt{s_0\log(pq)/n}),
      \end{aligned}
    \end{equation}
    uniformly over $1 \leq k \leq r^*$, where $\widetilde{\psi}_k = d_{1}^*/(d_k^*\min[\delta_{k-1}^*, \delta_{k}^*])$.
    Combining inequalities \eqref{para_init_svd_ineq1} and \eqref{para_init_svd_ineq2} and Lemma \ref{para_th_init} entails
    \begin{equation*}
      \norm{\widetilde{\C}_k^0 - \C_k^*}_F \leq 2d_c^*\norm{\widetilde{\v}_k - \v_k^*}_2 + \norm{\widetilde{\C} - \C^*}_F
      \leq M \psi_k\sqrt{s_0\log(pq)/n}
    \end{equation*}
    for some positive constant $M$, uniformly over $1 \leq k \leq r^*$, where $\psi_k = d_1^*d_c^*/(d_k^*\min[\delta_{k-1}^*, \delta_{k}^*])$.
    It concludes the proof of Lemma \ref{para_init_svd}.
  \end{proof}

  \subsection{Proof of Theorem \ref{para_th}}

  \begin{proof}
    Since $(\widehat{d}_{k},\widehat{\u}_{k},\widehat{\v}_{k})$ is the optimal solution of \eqref{para_each_layer_estimator}, we have
    \begin{equation}\label{ineq1_proof_para_th}
      (2n)^{-1}\norm{\X\widehat{\bDelta}_{k}}_{F} + \lambda_{k}\norm{\widehat{\C}_{k}}_{1} \leq n^{-1}\inner{\widehat{\bDelta}_{k}, \X\trans (\Y_{k} - \X\C_{k}^{*})} + \lambda_{k}\norm{\C_{k}^{*}}_{1}
    \end{equation}
    with $\Y_{k} = \Y - \X\widetilde{\C} + \X\widetilde{\C}_{k}$. Our proof will be conditioning on the event
    $\mathcal{A}=\{n^{-1}\norm{\X\trans \E}_{\max}\leq \lambda_{0}/2\}$, which holds with probability at least $1-(pq)^{1-\alpha}$ as demonstrated in Lemma \ref{para_th_init}.

    The key part of this proof is to derive the upper bound on $n^{-1}\norm{\X\trans (\Y_{k} - \X\C_{k}^{*})}_{\max}$. On the one hand, by the KKT condition of optimization \eqref{para_init_estimator_lasso}, we have
    \begin{equation}\label{kkt}
      n^{-1}\norm{\X\trans (\Y-\X\widetilde{\C})}_{\max} \leq \lambda_{0} = 2\sigma_{\max} \sqrt{2\alpha\log(pq)/n}.
    \end{equation}
    On the other hand, Lemma \ref{para_init_svd} gives
    \begin{equation*}
      \norm{\widetilde{\C}_{k}^0 - \C_k^*}_F \leq M\psi_k \sqrt{s_0\log(pq)/n}.
    \end{equation*}
    Denote by $S_k$ the index set consisting of the indices of the $s$ largest components of $\widetilde{\C}_{k}^0$ in terms of absolute values. By definition, our initial $k$th layer estimator $\widetilde{\C}_{k}$ takes the components of $\widetilde{\C}_{k}^0$ on $S_k$ while sets the other components to be zero.

    First of all, we claim that
    \begin{equation}\label{contra}
      \norm{\C_{k,J_k \cap S_k^C}^{*}}_F \leq 2M\psi_k\sqrt{s_0\log(pq)/n}.
    \end{equation}
    If this is not true, then $\norm{\C_{k,J_k \cap S_k^C}^{*}}_F > 2M\psi_k\sqrt{s_0\log(pq)/n}$. We can deduce that
    \begin{align*}
      \norm{\widetilde{\C}_{k,J_k \cap S_k^C}^{0}}_F & = \norm{\widetilde{\C}_{k,J_k \cap S_k^C}^{0} - \C_{k,J_k \cap S_k^C}^{*} + \C_{k,J_k \cap S_k^C}^{*}}_F              \\
                                                     & \geq \norm{\C_{k,J_k \cap S_k^C}^{*}}_F - \norm{\widetilde{\C}_{k,J_k \cap S_k^C}^{0} - \C_{k,J_k \cap S_k^C}^{*}}_F  \\
                                                     & \geq \norm{\C_{k,J_k \cap S_k^C}^{*}}_F - \norm{\widetilde{\C}_{k}^{0} - \C_{k}^{*}}_F > M\psi_k\sqrt{s_0\log(pq)/n}.
    \end{align*}
    On the other hand, since $J_k$ is the support of $\C_{k}^{*}$, we have
    \begin{align*}
      \norm{\widetilde{\C}_{k,S_k \cap J_k^C}^{0}}_F & = \norm{\widetilde{\C}_{k,J_k \cap S_k^C}^{0} - \C_{k,S_k \cap J_k^C}^{*} + \C_{k,S_k \cap J_k^C}^{*}}_F                  \\
                                                     & \leq \norm{\C_{k,S_k \cap J_k^C}^{*}}_F + \norm{\widetilde{\C}_{k,S_k \cap J_k^C}^{0} - \C_{k,S_k \cap J_k^C}^{*}}_F      \\
                                                     & \leq \norm{\C_{k,S_k \cap J_k^C}^{*}}_F + \norm{\widetilde{\C}_{k}^{0} - \C_{k}^{*}}_F \leq  M\psi_k\sqrt{s_0\log(pq)/n}. 
    \end{align*}

    In view of these two inequalities, we get
    \begin{align*}
      \norm{\widetilde{\C}_{k,J_k \cap S_k^C}^{0}}_F > \norm{\widetilde{\C}_{k,S_k \cap J_k^C}^{0}}_F.
    \end{align*}
    This is a contraction since by the definition of $S_k$, the set $J_k \cap S_k^C$ contains $s_k^* - t$ elements with $t = |J_k \cap S_k|$, each of which in $\widetilde{\C}_{k}^{0}$ should be no larger than any of the $s - t$ components indexed by $S_k \cap J_k^C$, and $s - t$ is no less than $s_k^* - t$. Therefore, we know that inequality \eqref{contra} should be true.
    It follows that
    \begin{equation}\label{para_th_thresh_ineq1}
      \begin{aligned}
        \norm{\widetilde{\C}_k - \C_k^*}_{F} & \leq \norm{\widetilde{\C}_{k,S_k} - \C_{k,S_k}^*}_{F}
        + \norm{\C_{k,J_k \cap S_k^C}^{*}}_F                                                                                                                           \\
                                             & \leq \norm{\widetilde{\C}_{k}^0 - \C_{k}^*}_{F} + \norm{\C_{k,J_k \cap S_k^C}^{*}}_F \leq 3M\psi_k\sqrt{s_0\log(pq)/n}.
      \end{aligned}
    \end{equation}

    Moreover, since $\|\widetilde{\C}_k\|_0 = s$ and $\|\C_k^*\|_0 = s_k^* \leq s$, we have
    \begin{equation*}
      \norm{\widetilde{\bDelta}_k}_0 = \|\widetilde{\C}_k - \C_k^*\|_0 \leq \|\widetilde{\C}_k\|_0 + \|\C_k^*\|_0 = s + s_k^* \leq 2s.
    \end{equation*}
    Denote by $\widetilde{J}_k$ the index set consisting of the indices of the $s$ largest components of $\widetilde{\bDelta}_k$ in terms of absolute values.
    It is clear that $\norm{\widetilde{\bDelta}_{{\widetilde{J}_k}^c}}_1 \le \norm{\widetilde{\bDelta}_{\widetilde{J}_k}}_1$. Thus, by Condition \ref{cond4}, we get
    \begin{equation}\label{para_th_thresh_ineq2}
      n^{-1}\norm{\X\trans\X\widetilde{\bDelta}_{k}}_{\max} \leq \phi_u\norm{\widetilde{\bDelta}_{\widetilde{J}_k}}_F/\sqrt{s}
      \le \phi_u\norm{\widetilde{\bDelta}_k}_F/\sqrt{s}.
    \end{equation}
    As $s \geq s_0$, combining inequalities \eqref{para_th_thresh_ineq1} and \eqref{para_th_thresh_ineq2} gives
    \begin{equation*}
      n^{-1}\norm{\X\trans\X\widetilde{\bDelta}_{k}}_{\max} \leq 3 M\phi_u\psi_k\sqrt{\log(pq)/n}.
    \end{equation*}
    Together with \eqref{kkt}, we can derive that
    \begin{equation*}
      n^{-1}\norm{\X\trans (\Y_{k} - \X\C_{k}^{*})}_{\max} \leq n^{-1}\norm{\X\trans (\Y-\X\widetilde{\C})}_{\max}
      + n^{-1}\norm{\X\trans\X\widetilde{\bDelta}_{k}}_{\max} \le C\psi_k\sqrt{\log(pq)/n}
    \end{equation*}
    for some positive constant $C$ independent of $k$.

    \smallskip

    Therefore, when $\lambda_{k} = 2C\psi_k\sqrt{\log(pq)/n}$, by the same argument as inequality \eqref{induction:1}, we get
    \begin{equation}\label{ineq2_proof_para_th}
      \norm{\widehat{\bDelta}_{J_{k}^{c}}}_{1} \leq 3\norm{\widehat{\bDelta}_{J_{k}}}_{1}.
    \end{equation}
    Similarly, it follows from inequality \eqref{ineq1_proof_para_th} that
    \begin{equation*}
      n^{-1}\norm{\X\widehat{\bDelta}_{k}}_{F}^{2} \leq 3\lambda_{k}\norm{\widehat{\bDelta}_{k}}_{1}
      \leq 12\lambda_{k}\norm{\widehat{\bDelta}_{J_{k}}}_{1} \leq 12\lambda_{k}\sqrt{s_{k}}\norm{\widehat{\bDelta}_{J_{k}}}_{F}.
    \end{equation*}
    Applying Condition \ref{cond1}, we have
    \begin{equation}\label{ineq3_proof_para_th}
      \rho_{l}(\norm{\widehat{\bDelta}_{J_k}}_{F}^{2}\vee\norm{\widehat{\bDelta}_{J_k^c,s_k}}_{F}^{2}) \leq
      n^{-1}\norm{\X\widehat{\bDelta}_{k}}_{F}^{2} \leq
      12\lambda_{k}\sqrt{s_{k}}\norm{\widehat{\bDelta}_{J_{k}}}_{F}.
    \end{equation}

    Finally, by the same argument as that in the proof of Lemma \ref{se_lemma2}, we can get
    \begin{equation}\label{uni}
      \begin{aligned}
         & \norm{\widehat{\bDelta}_{k}}_{F} \leq \norm{\widehat{\bDelta}_{J_k}}_{F} +
        \norm{\widehat{\bDelta}_{J_k^c,s_{k}}}_{F} + \norm{\widehat{\bDelta}_{\tilde{J}_k^c,s_{k}}}_{F}
        = O(\lambda_{k}\sqrt{s_k}),                                                                                     \\
         & \norm{\widehat{\bDelta}_{k}}_{1} = \norm{\widehat{\bDelta}_{J_k}}_{1} + \norm{\widehat{\bDelta}_{J_k^c}}_{1}
        \leq 4\norm{\widehat{\bDelta}_{J_k}}_{1} \leq 4\sqrt{s_{k}}\norm{\widehat{\bDelta}_{J_k}}_{F} = O(\lambda_{k}s_{k}).
      \end{aligned}
    \end{equation}
    Since the above argument applies to any fixed $k$, $1 \leq k \leq r^*$, and the corresponding constants are independent of $k$, we know that the estimation error bounds in \eqref{uni} hold uniformly over $k$. It concludes the proof of Theorem \ref{para_th}.
  \end{proof}

  \section{Proof of Theorem \ref{partial_vs_global_optimum_th}}

  Note that $\widehat{\C}_{k}$ is the global minimizer of the corresponding CURE problem and $\widehat{\C}_{k}^L$ is a local minimizer that satisfies the conditions in Theorem \ref{partial_vs_global_optimum_th}. Since the following argument applies to any fixed $k$, $1 \leq k \leq r^*$, we omit the subscript $k$ and write them as $\widehat{\C}$ and $\widehat{\C}^L$ for simplicity. We will show that $\norm{\bDelta}_F = O(s_k^{1/2}\lambda_k)$ and $\norm{\bDelta}_1=O(s_k \lambda_k)$ with $\bDelta = \widehat{\C}^L - \widehat{\C}$ such that $\widehat{\C}^L$ enjoys the same asymptotic properties as $\widehat{\C}$ in terms of estimation error bounds.

  Our proof will be conditioning on the event $\mathcal{A}$ (defined in Lemma \ref{se_lemma2} for the sequential pursuit and defined in Lemma \ref{para_th_init} for the parallel pursuit), which holds with probability at least $1 - (pq)^{1-\alpha}$. First of all, by the KKT condition, the global minimizer $\widehat{\C}$ satisfies
  \begin{equation*}
    n^{-1}\norm{\X\trans (\Y_{k}-\X\widehat{\C})}_{\max} \leq \lambda_{k}.
  \end{equation*}
  Together with $n^{-1}\norm{\X\trans (\Y_{k}-\X\widehat{\C}^L)}_{\max} = O(\lambda_{k})$ imposed in Theorem \ref{partial_vs_global_optimum_th}, we get
  \begin{equation*}
    n^{-1}\norm{\X\trans\X\bDelta}_{\max} \leq n^{-1}\norm{\X\trans (\Y_{k}-\X\widehat{\C})}_{\max} + n^{-1}\norm{\X\trans (\Y_{k}-\X\widehat{\C}^L)}_{\max} = O(\lambda_k).
  \end{equation*}
  Denote by $A_j = \text{supp}(\bDelta_j)$ with $\bDelta_j$ being the $j$th column of $\bDelta$. It follows that
  \begin{align*}
    n^{-1}\norm{\X_{A_j}\trans\X_{A_j}\bDelta_{A_j}}_2 & \leq n^{-1} \sqrt{\abs{A_j}}\norm{\X_{A_j}\trans\X_{A_j}\bDelta_{A_j}}_{\max} \leq n^{-1} \sqrt{\abs{A_j}}\norm{\X\trans\X_{A_j}\bDelta_{A_j}}_{\max}     \\
                                                       & = n^{-1} \sqrt{\abs{A_j}}\norm{\X\trans\X \bDelta_{j}}_{\max} \leq n^{-1} \sqrt{\abs{A_j}}\norm{\X\trans\X \bDelta}_{\max} = O(\abs{A_j}^{1/2}\lambda_k),
  \end{align*}
  where $\X_{A_j}$ ($\bDelta_{A_j}$) is the submatrix (subvector) of $\X$ ($\bDelta_{j}$) consisting of columns (components) in $A_j$.

  Moreover, by the assumptions $\norm{\widehat{\C}}_0 = O(s_k)$ and $\norm{\widehat{\C}^L}_0 = O(s_k)$, there exists some positive constant $C$ such that
  \begin{equation*}
    \sum_{j=1}^{q}\abs{A_j} \leq \norm{\widehat{\C}}_0 + \norm{\widehat{\C}^L}_0 \leq C s_k.
  \end{equation*}
  Therefore, it follows from the assumption $\min_{\norm{\boldsymbol{\gamma}}_{2}=1,\norm{\boldsymbol{\gamma}}_{0}\le Cs_k} n^{-1/2}\norm{\X\boldsymbol{\gamma}}_{2} \ge \kappa_0$ in Theorem \ref{partial_vs_global_optimum_th} that the smallest singular value of $n^{-1/2} \X_{A_j}$ is bounded from below by $\kappa_0$. It yields
  \begin{equation*}
    \kappa_0^2\norm{\bDelta_{A_j}}_2 \leq n^{-1}\norm{\X_{A_j}\trans\X_{A_j}\bDelta_{A_j}}_2 = O(\abs{A_j}^{1/2} \lambda_k).
  \end{equation*}
  Since $\norm{\bDelta}_F^2 = \sum_{j=1}^{q}\norm{\bDelta_j}_2^2 = \sum_{j=1}^{q}\norm{\bDelta_{A_j}}_2^2$, we finally get
  \begin{equation*}
    \norm{\bDelta}_F^2 = O(\lambda_k^2\sum_{j=1}^{q}\abs{A_j}) = O(s_k\lambda_k^2).
  \end{equation*}
  It follows immediately that $\norm{\bDelta}_F = O(s_k^{1/2}\lambda_k)$ and $\norm{\bDelta}_1=O(s_k \lambda_k)$, which concludes the proof.

  \section{Proof of Stagewise CURE}\label{sec:app:cure}

  \subsection{Derivations of contended stagewise learning}\label{sec:derivation-algorithm}
 
Before presenting the derivations, we first recall some notations for the sake of clarity. Denote 
$\X=\left[\widetilde{\x}_{1},\dots,\widetilde{\x}_p\right]=\left[\x_1,\dots,\x_{n}\right]\trans\in\mathbb{R}^{n\times p}$ and 
$\Y=\left[\widetilde{\y}_{1},\dots,\widetilde{\y}_{q}\right]=\left[\y_1,\dots,\y_n\right]\trans\in\mathbb{R}^{n\times q}$ then let 
$\widetilde{\x}_{j}$ and $\widetilde{\y}_k$ be the $j$th and $k$th columns of $\X$ and $\Y$, respectively. In addition, let 
$\E^t$ be equal to $\Y - d^t\X\u^t\v^t\trans$ and $\widetilde{\e}_k$ denotes the $k$th column of $\E^t$. 

\noindent \textbf{(I) Initialization}. Recall that the optimization in the initialization is as follow
\begin{equation*}
    (\widehat{j},\widehat{k},\widehat{s})
    = \underset{(j,k);s=\pm\epsilon}{\arg\min}~ L(s\1_j\1_k\trans).
\end{equation*}
By the expansion of the loss function, we have
\begin{equation*}
    \begin{aligned}
        L(s\1_j\1_k\trans) & = (2n)^{-1}\norm{\Y - s\X\1_{j}\1_{k}\trans}_F^2 + \frac{\mu}{2}\norm{s\1_{j}\1_{k}\trans}_F^2                                              \\
                           & = (2n)^{-1}\norm{\Y}_F^2 + (2n)^{-1}\epsilon^2\1_{j}\trans\X\trans\X\1_{j} - n^{-1}s\1_{j}\trans\X\trans\Y\1_{k} + \frac{\mu}{2}\epsilon^2.
    \end{aligned}
\end{equation*}
Thus the original problem is equivalent to
\begin{equation*}
    \begin{aligned}
        (\what{j},\what{k})
        = \underset{j,k}{\arg\min}~(2n)^{-1}\epsilon\norm{\widetilde{\x}_j}_2^2 - n^{-1}\abs{\widetilde{\x}_{j}\trans\widetilde{\y}_{k}},~
        \widehat{s}
        = \mbox{sgn}(\widetilde{\x}_{j}\trans\widetilde{\y}_{k})\epsilon. 
    \end{aligned}
\end{equation*}

\noindent \textbf{(II) Backward update}. At the $(t+1)$th step, the two updating options are as follows,
\begin{equation*}
    \begin{aligned}
        \what{j} & = \arg\min_{j\in\mathcal{A}^t} L(d\u\v\trans) \mbox{ s.t. } (d\u) = (d\u)^t-\mbox{sgn}(u_j^t)\epsilon\1_j, \v = \v^{t}, \\
        \what{k} & = \arg\min_{k\in\mathcal{B}^t} L(d\u\v\trans) \mbox{ s.t. } (d\v) = (d\v)^t-\mbox{sgn}(v_k^t)\epsilon\1_k, \u = \u^{t}.
    \end{aligned}
\end{equation*}
Note that, either updating $d\u$ or $d\v$ will decrease the penalty term by a fixed amount $\lambda^t\epsilon$. For example, assuming that we update $d\u$, then 
\begin{equation*}
    \begin{aligned}
        \rho(d^{t+1}\u^{t+1}\v^{t+1}\trans) - \rho(d^t\u^t\v^t\trans)
         & = \norm{(d\u)^t - \mbox{sgn}(u_{\what{j}}^t)\epsilon\1_{\what{j}}}_1\norm{\v^t}_1 - d^t\norm{\u^t}_1\norm{\v^t}_1 \\
         & = \norm{(d\u)^t - \mbox{sgn}(u_{\what{j}}^t)\epsilon\1_{\what{j}}}_1 - \norm{(d\u)^t}_1 = -\epsilon.
    \end{aligned}
\end{equation*}
It follows that we only need to compare the value of the loss function to
decide whether to update $d\u$ or $d\v$. 

Focusing on $d\u$, the loss function can be written as
\begin{equation}
    \begin{aligned}
        L\left(\left[(d\u)^t-\mbox{sgn}(u_{j}^t)\epsilon\1_{j}\right]\v^t\trans\right)
         & = (2n)^{-1}\norm{\Y - \X\left[(d\u)^t-\mbox{sgn}(u_{j}^t)\epsilon\1_{j}\right]\v^t\trans}_{F}^2 + \\
         & \quad \frac{\mu}{2}\norm{\v^t}_2^2\norm{(d\u)^t-\mbox{sgn}(u_{j}^t)\epsilon\1_{j}}_2^2\\  
         & = (2n)^{-1}\epsilon^{2}\norm{\widetilde{\x}_j}_2^2\norm{\v^{t}}_{2}^{2}
        + n^{-1}\mbox{sgn}(u_{j}^t)\epsilon\widetilde{\x}_{j}\trans\E^t\v^t                                                   \\
         & \quad - \mu\epsilon d^t\abs{u_j^t}\norm{\v^t}_2^2 + \frac{\mu}{2}\epsilon^2\norm{\v^t}_2^2 + L(d^t\u^t\v^t\trans),
    \end{aligned}
\end{equation}
where $\E^t=\Y - d^t\X\u^t\v^t\trans$, and $\widetilde{\x}_j$ is the $j$th column of $\X$. For $d\v$, with similar argument, we have 
\begin{equation}\label{proof:loss:v}
    \begin{aligned}
        L\left(\u^t\left[(d\v)^t - \mbox{sgn}(v_{k}^t)\epsilon\1_{k}\right]\trans\right)
         & = n^{-1}\mbox{sgn}(v_k^t)\epsilon\u^t\trans\X\trans\widetilde{\e}_k^t - \mu\epsilon d^t\abs{v_k^t}\norm{\u^t}_2^2 + \frac{\mu}{2}\epsilon^2\norm{\u^t}_2^2 \\
         & \quad + (2n)^{-1}\epsilon^2\norm{\X\u^t}_2^2 + L(d^t\u^t\v^t\trans),
    \end{aligned}
\end{equation}
where $\widetilde{\e}_k^t$ is the $k$th column of $\E^t$. Therefore, 
the two proposals in the backward step are equivalent to
\begin{equation}\label{proof:backward:opt}
    \begin{aligned}
        \what{j} & = \underset{j\in\mathcal{A}^t}{\arg\min}~ (2n)^{-1}\epsilon\norm{\widetilde{\x}_j}_2^2\norm{\v^{t}}_{2}^{2} +
        n^{-1}\mbox{sgn}(u_{j}^t)\widetilde{\x}_{j}\trans\E^t\v^t - \mu d^t\abs{u_j^t}\norm{\v^t}_2^2,                           \\
        \what{k} & = \underset{k\in\mathcal{B}^t}{\arg\min}~ n^{-1}\mbox{sgn}(v_k^t)\u^t\trans\X\trans\widetilde{\e}_k^t -
        \mu d^t\abs{v_k^t}\norm{\u^t}_2^2.
    \end{aligned}
\end{equation}

\noindent \textbf{(III) Forward update}. When the backward update can no longer proceed, a forward
update is carried out. Let's focus on $d\u$; the problem is
\begin{equation*}
    \underset{j,s=\pm\epsilon}{\arg\min}~L\left(\left[(d\u)^t+s\1_j\right]\v^t\trans\right) = (2n)^{-1}\norm{\Y - \X\left[(d\u)^t+s\1_j\right]\v^t\trans}_F^2 +
    \frac{\mu}{2}\norm{(d\u)^t+s\1_j}_2^2\norm{\v^t}_2^2.
\end{equation*}
We have that
\begin{equation*}
    \begin{aligned}
        L\left(\left[(d\u)^t+s\1_j\right]\v^t\trans\right)
         & = (2n)^{-1}\epsilon^2\norm{\widetilde{\x}_j}_2^2\norm{\v^t}_2^2 - n^{-1}s\widetilde{\x}_j\trans\E^t\v^t\\
         & \quad + \frac{\mu}{2}\epsilon^2\norm{\v^t}_2^2 + \mu sd^tu_j^t\norm{\v^t}_2^2 + L(d^t\u^t\v^t\trans).
    \end{aligned}
\end{equation*}
Similarly, for $d\v$, we have 
\begin{equation*}
    \begin{aligned}
        L\left(\u^t\left[(d\v)^t+h\1_k\right]\trans\right) = \frac{\epsilon^2}{2n}\norm{\X\u^t}_2^2 - \frac{h}{n}\u^t\trans\X\trans\widetilde{\e}_k^t
        + \frac{\mu\epsilon^2}{2}\norm{\u^t}_2^2 + \mu hd^t v_k^t\norm{\u^t}_2^2 + L(d^t\u^t\v^t\trans).
    \end{aligned}
\end{equation*}
Therefore, minimizing the loss function is equivalent to
\begin{equation}\label{proof:forward:opt}
    \begin{aligned}
         & \underset{j,s=\pm\epsilon}{\arg\min}~ (2n)^{-1}\epsilon^2\norm{\widetilde{\x}_j}_2^2\norm{\v^t}_2^2
        - s(n^{-1}\widetilde{\x}_j\trans\E^t\v^t - \mu d^tu_j^t\norm{\v^t}_2^2),                                               \\
         & \underset{k,h=\pm\epsilon}{\arg\min}~ h(\mu d^t v_k^t\norm{\u^t}_2^2 - n^{-1}\u^t\trans\X\trans\widetilde{\e}_k^t).
    \end{aligned}
\end{equation}
Then it follows  that
\begin{equation*}
    \begin{aligned}
        \what{j} & = \underset{j}{\arg\min}~ (2n)^{-1}\epsilon\norm{\widetilde{\x}_j}_2^2\norm{\v^t}_2^2
        - \abs{n^{-1}\widetilde{\x}_j\trans\E^t\v^t - \mu d^tu_j^t\norm{\v^t}_2^2},                                            \\
        \what{k} & = \underset{k}{\arg\min}~ -\abs{n^{-1}\u^t\trans\X\trans\widetilde{\e}_k^t - \mu d^t v_k^t\norm{\u^t}_2^2},
    \end{aligned}
\end{equation*}
and
\begin{equation*}
    \begin{aligned}
        \widehat{s} & = \mbox{sgn}(n^{-1}\widetilde{\x}_j\trans\E^t\v^t - \mu d^tu_j^t\norm{\v^t}_2^2)\epsilon ~\text{for updating $d\u$};        \\
        \widehat{h} & = \mbox{sgn}(n^{-1}\u^t\trans\X\trans\widetilde{\e}_k^t - \mu d^t v_k^t\norm{\u^t}_2^2)\epsilon ~\text{for updating $d\v$}.
    \end{aligned}
\end{equation*}

\noindent \textbf{(IV) Computational complexity}. 
Assume that the $t$th step is completed, and the cardinalities of active sets $\mA^t$ and $\mB^t$ are $a^t$ and $b^t$, respectively.
In the $(t+1)$th step, we need to update the current residual matrix $\E^t = \Y - d^t\X\u^t\v^t\trans$.
The complexity for updating $\E^t$ is $O(a^t n + b^t n)$ because we only need to calculate the coordinates in the active sets.

With the given $\E^t$, updating $\widetilde{\x}_j\trans\E^t\v^t$ requires $O(b^tn)$ by only calculating the coordinates corresponding to the active set.
Similarly, the complexity of updating $\u^t\trans\X\trans\widetilde{\e}_k^t$ is $O(a^tn)$. The complexity of computing $\norm{\u^t}_2^2$ or
$\norm{\v^t}_2^2$ is $O(a^t)$ or $O(b^t)$, respectively. Comparing \eqref{proof:backward:opt} with \eqref{proof:forward:opt}, the complexity of \eqref{proof:backward:opt} is negligible because it only searches
$j$ and $k$ in the active sets. By the expansion of loss function, the complexity for comparing the value of loss function
is also negligible in terms of \eqref{proof:forward:opt}. Therefore, with $\E^t$, the complexity of backward and forward step in the $(t+1)$th step is $O(a^tnq + b^tnp)$. Combing the complexity for updating $\E^t$, the complexity in the $(t+1)$th step is still of the order $O(a^tnq + b^tnp)$.

  \subsection{Proof of Lemma \ref{lemma:algorithm:converge}}


\begin{lemma}\label{stagewise_initial_lemma}
    Let $Q(\cdot)$ be the objective function  defined in \eqref{elas_cure} of the main paper.\\
    \noindent 1. For any $\lambda$, if there exist $j$ and $k$ such that $Q(h\1_{j}\1_{k}\trans;\lambda) \leq Q(0;\lambda)$ where $\abs{h}=\epsilon$, it must be true that $\lambda \leq \lambda^{0}$.

    \noindent 2. For any $t$, we have
            $Q(d^{t+1}\u^{t+1}\v^{t+1}\trans;\lambda^{t+1}) \leq
            Q(d^{t}\u^{t}\v^{t}\trans;\lambda^{t+1}) - \xi$.

    \noindent 3. For $\xi \geq 0$ and any $t$ such that $\lambda^{t+1} < \lambda^{t}$, we have
    \begin{equation*}
        \begin{aligned}
             & Q(d^{t}\u^{t}\v^{t}\trans;\lambda^{t}) - \xi
            \leq Q\left((d^{t}\u^{t}\pm \epsilon\1_{j})\v^{t}\trans;\lambda^{t}\right), \\
             & Q(d^{t}\u^{t}\v^{t}\trans;\lambda^{t}) - \xi \leq
            Q\left(\u^{t}(d^{t}\v^{t}\pm \epsilon\1_{k})\trans;\lambda^{t}\right),
        \end{aligned}
    \end{equation*}
    where $1\leq j \leq p$ and $1\leq k \leq q$.
\end{lemma}
The first statement in Lemma \ref{stagewise_initial_lemma} ensures the validity of the initialization step. The second statement in Lemma \ref{stagewise_initial_lemma} ensures that, for each $\lambda^{t}$, the algorithm performs coordinate descent whenever $\lambda^{t}=\lambda^{t+1}$. The third statement implies that $\lambda^t$ gets reduced only when the penalized loss at $\lambda^t$ can not be further reduced even by searching over all possible coordinate descent directions. 


\begin{proof}[Proof of Lemma \ref{stagewise_initial_lemma}]
    \noindent 1. By the assumption, we have $L(0) - L(h\1_{j}\1_{k}\trans ) \geq \lambda \rho(h\1_{j}\1_{k}\trans ) = \lambda \epsilon$, which yields
    \begin{equation*}
        \lambda \leq \frac{1}{\epsilon}[L(0) - L(h\1_{j}\1_{k}\trans )]
        \leq \frac{1}{\epsilon}[L(0) - \underset{j,k,s=\pm\epsilon}{\min}~L(h\1_{j}\1_{k}\trans )] = \lambda^{0}.
    \end{equation*}

    \noindent 2.
    In the backward step, we have that $\lambda^{t+1} = \lambda^{t}$,  
    $L(d^{t+1}\u^{t+1}\v^{t+1}\trans) \leq L(d^{t}\u^{t}\v^{t}\trans) + \lambda^{t}\epsilon - \xi$, and the penalty term is always decreased by a fixed amount $\lambda^t\epsilon$. So $Q(d^{t+1}\u^{t+1}\v^{t+1}\trans;\lambda^{t+1}) \leq Q(d^{t}\u^{t}\v^{t}\trans;\lambda^{t+1})- \xi$ holds.
    
    It remains to consider the forward step when $\lambda^{t+1} = \lambda^{t}$. If $Q(d^{t+1}\u^{t+1}\v^{t+1}\trans;\lambda^{t+1}) > Q(d^{t}\u^{t}\v^{t}\trans;\lambda^{t+1})- \xi$, we have
    \begin{equation*}
        \begin{aligned}
            L(d^{t}\u^{t}\v^{t}\trans) - L(d^{t+1}\u^{t+1}\v^{t+1}\trans) - \xi
             & < \lambda^{t+1}[\rho(d^{t+1}\u^{t+1}\v^{t+1}) - \rho(d^{t}\u^{t}\v^{t}\trans)] = \lambda^{t+1}\epsilon,
        \end{aligned}
    \end{equation*}
    which yields that
    \begin{equation*}
        \lambda^{t+1} > \frac{L(d^{t}\u^{t}\v^{t}\trans) - L(d^{t+1}\u^{t+1}\v^{t+1}\trans) - \xi}{\epsilon}.
    \end{equation*}
    This contradicts with the formula for computing $\lambda^{t+1}$.
  
    \noindent 3.
    The $\lambda^{t}$ can only be reduced in the forward step. So when $\lambda^{t+1}<\lambda^{t}$, we have
    \begin{equation*}
        \lambda^{t+1} = \frac{L(d^{t}\u^{t}\v^{t}\trans) - L(d^{t+1}\u^{t+1}\v^{t+1}\trans)-\xi}{\epsilon},
    \end{equation*}
and  $\epsilon=\rho(d^{t+1}\u^{t+1}\v^{t+1}\trans) - \rho(d^{t}\u^{t}\v^{t}\trans)$. It follows that 
    \begin{equation*}
        \begin{aligned}
             \lambda^{t+1}[\rho(d^{t+1}\u^{t+1}\v^{t+1}\trans) - \rho(d^{t}\u^{t}\v^{t}\trans)]
              = L(d^{t}\u^{t}\v^{t}\trans) - L(d^{t+1}\u^{t+1}\v^{t+1}\trans) - \xi.
        \end{aligned}
    \end{equation*}
    That is, $Q(d^{t+1}\u^{t+1}\v^{t+1}\trans;\lambda^{t+1}) = Q(d^{t}\u^{t}\v^{t}\trans;\lambda^{t+1}) - \xi$.

    We then have that 
    \begin{equation}\label{lemma_stage_3_eq1}
        \begin{aligned}
            Q(d^{t+1}\u^{t+1}\v^{t+1}\trans;\lambda^{t+1}) +
            (\lambda^{t} - \lambda^{t+1})\rho(d^{t}\u^{t}\v^{t}\trans) =
            Q(d^{t}\u^{t}\v^{t}\trans;\lambda^{t}) - \xi,
        \end{aligned}
    \end{equation}
   and 
    \begin{equation}\label{lemma_stage_3_ineq1}
        \begin{aligned}
             & \quad~ \lambda^{t+1}\rho(d^{t+1}\u^{t+1}\v^{t+1}\trans) + (\lambda^{t} - \lambda^{t+1})\rho(d^{t}\u^{t}\v^{t}\trans)               \\
             & = \lambda^{t+1}\rho(d^{t+1}\u^{t+1}\v^{t+1}\trans) + (\lambda^{t} - \lambda^{t+1})[\rho(d^{t+1}\u^{t+1}\v^{t+1}\trans) - \epsilon] \\
             & = \lambda^{t}\rho(d^{t+1}\u^{t+1}\v^{t+1}\trans) + \epsilon(\lambda^{t+1}-\lambda^{t}) \leq
            \lambda^{t}\rho(d^{t+1}\u^{t+1}\v^{t+1}\trans).
        \end{aligned}
    \end{equation}
    Combining \eqref{lemma_stage_3_eq1} and \eqref{lemma_stage_3_ineq1}, we have
    \begin{equation*}
        \begin{aligned}
            Q(d^{t}\u^{t}\v^{t}\trans;\lambda^{t}) - \xi
             & \leq Q(d^{t+1}\u^{t+1}\v^{t+1}\trans;\lambda^{t+1}) + (\lambda^{t} - \lambda^{t+1})\rho(d^{t+1}\u^{t+1}\v^{t+1}\trans) \\
             & = Q(d^{t+1}\u^{t+1}\v^{t+1}\trans;\lambda^{t})\\
             & = \min\{Q((d^{t}\u^{t}\pm \epsilon\1_j )\v^{t}\trans;\lambda^{t}),Q((\u^{t}(d^{t}\v^{t}\pm \epsilon\1_k )\trans;\lambda^{t})\}.
        \end{aligned}
      \end{equation*}
This completes the proof.
  \end{proof}

  Now we prove Lemma \ref{lemma:algorithm:converge}. Recall that for $\lambda^t$, the corresponding objective function is
  \begin{equation*}
    Q(d\u\v\trans) = L(\d\u\v\trans) + \lambda^td\norm{\u}_1\norm{\v}_1,
  \end{equation*}
  where $L(d\u\v\trans) = (2n)^{-1}\norm{\Y-\d\X\u\v\trans}_F^2+\mu\norm{d\u\v\trans}_F^2/2$. In the $(t+1)$th step, the objective function is strongly convex when either $\u^t$ or $\v^t$ is fixed. Thus, with any given $\X$ and $\Y$, there exists a constant $M$ such that  
  \begin{equation*}
    \begin{aligned}
      \mu\norm{\v^t}_2^2\I \preceq \nabla_{d\u}^2L \preceq M\I ,~
      \mu\norm{\u^t}_2^2\I \preceq \nabla_{d\v}^2L \preceq M\I,
    \end{aligned}
  \end{equation*}
where $\I$ is the identity matrix, and $\nabla_{d\u}^2L$ and $\nabla_{d\v}^2L$ are the second derivatives of the loss function with respect to $d\u$ and $d\v$, respectively. Define
  \begin{equation*}
    \begin{aligned}
      (d\u)^{t*} & = \underset{d\u}{\arg\min}~Q(d\u\v^t\trans)~\text{with fixed $\v^t$};   \\
      (d\v)^{t*} & = \underset{d\v}{\arg\min}~Q(\u^t(d\v)\trans)~\text{with fixed $\u^t$}.
    \end{aligned}
  \end{equation*}

We now derive the upper bound of $\norm{d^t\u^t - (d\u)^{t*}}_2$. (With similar argument, we can get the upper bound of
  $\norm{d^t\v^t-(d\v)^{t*}}_2$.) With fixed $\v^t$, Taylor expansion of the objective function gives
  \begin{equation}\label{proof:th_converge:eq1}
    \begin{aligned}
      Q\left((d\u)^{t*}\right) & = Q\left((d\u)^t\right) + [\nabla_{d\u}L\left((d\u)^t\right) + \lambda^t\bdelta]\trans\left((d\u)^{t*} - (d\u)^t\right) \\
                               & \quad + \frac{1}{2}\left((d\u)^{t*} - (d\u)^t\right)\trans\nabla_{d\u}^2L\left((d\u)^{t*} - (d\u)^t\right),
    \end{aligned}
  \end{equation}
where $\bdelta$ is a $p$-dimensional vector and we use $\norm{\v^t}=1$. The $j$th entry of $\bdelta$, $\delta_j$, is subject to $\abs{\delta_j}\leq 1$ and
  $\delta_j=\mbox{sgn}(u_j)$ if $u_j\neq 0$. Here for convenience, we have rewritten $d^t\u^t$ as $(d\u)^t$ and omitted $\v^t$ in the expressions. With the boundedness of the Hessian matrix and
  the fact that $Q\left((d\u)^{t*}\right)\leq Q\left((d\u)^t\right)$, it follows that
  \begin{equation*}
    \begin{aligned}
      \frac{\mu\norm{\v^t}_2^2}{2}\norm{(d\u)^{t*} - (d\u)^t}_2^2
       & \leq -[\nabla_{d\u}L\left((d\u)^t\right) + \lambda^t\bdelta]\trans\left((d\u)^{t*} - (d\u)^t\right) \\
       & \leq \norm{\nabla_{d\u}L\left((d\u)^t\right) + \lambda^t\bdelta}_2\norm{(d\u)^{t*} - (d\u)^t}_2.
    \end{aligned}
  \end{equation*}
Therefore,
  \begin{equation*}
    \norm{(d\u)^{t*} - (d\u)^t}_2 \leq \frac{2}{\mu\norm{\v^t}_2^2}\norm{\nabla_{d\u}L\left((d\u)^t\right) + \lambda^t\bdelta}_2.
  \end{equation*}

It remains to derive the upper bound of $\norm{\nabla_{d\u}L\left((d\u)^t\right) + \lambda^t\bdelta}_2$. To do this, let's examine the upper bound of each entry of
  $\nabla_{d\u}L\left((d\u)^t\right)+\lambda^t\bdelta$. By the third statement in Lemma \ref{stagewise_initial_lemma}, for $u_j^t\neq 0$, we have
  \begin{equation}\label{proof:th_converge:ineq1}
    L\left((d\u)^t\pm\mbox{sgn}(u_j^t)\epsilon\1_j\right) \pm\lambda^t\epsilon \geq L\left((d\u)^t\right) - \xi,
  \end{equation}
  when $\lambda^{t+1} < \lambda^{t}$. On the other hand, with Taylor expansion, we have
  \begin{equation}\label{proof:th_converge:ineq2}
    \begin{aligned}
      L\left((d\u)^t\pm\mbox{sgn}(u_j^t)\epsilon\1_j\right)
       & = L\left((d\u)^t\right) \pm\mbox{sgn}(u_j^t)\epsilon\nabla_{d\u}L\left((d\u)^t\right)\trans\1_j \\
       & \quad + \frac{\epsilon^2}{2}\1_{j}\trans\nabla_{d\u}^{2}L\1_j.
    \end{aligned}
  \end{equation}
  Combing \eqref{proof:th_converge:ineq1} and \eqref{proof:th_converge:ineq2} yields
  \begin{equation*}
    \mp\left[\mbox{sgn}(u_j^t)\epsilon\nabla_{d\u}L\left((d\u)^t\right)\trans\1_j + \lambda^t\epsilon\right]
    \leq \frac{\epsilon^2}{2}\1_{j}\trans\nabla_{d\u}^{2}L\1_j + \xi
    \leq \frac{M\epsilon^2}{2} + \xi,
  \end{equation*}
  where the last inequality is due to the boundedness of the Hessian matrix. We have that 
  \begin{equation}\label{proof:th_converge:ineq4}
    \begin{aligned}
      \abs{\nabla_{d\u}L\left((d\u)^t\right)\trans\1_j + \mbox{sgn}(u_j^t)\lambda^t}
       & = \abs{\mbox{sgn}(u_j^t)\nabla_{d\u}L\left((d\u)^t\right)\trans\1_j + \lambda^t} \\
       & \leq \frac{M\epsilon}{2} + \frac{\xi}{\epsilon}~~\text{for $u_j^t\neq 0$}.
    \end{aligned}
  \end{equation}
When $u_j^t = 0$, by the third statement of Lemma \ref{stagewise_initial_lemma}, we have that 
  \begin{equation}\label{proof:th_converge:ineq3}
    L\left((d\u)^t\pm\mbox{sgn}(u_j^t)\epsilon\1_j\right) + \lambda^t\epsilon \geq L\left((d\u)^t\right) - \xi.
  \end{equation}
  Similarly, combining \eqref{proof:th_converge:ineq2} and \eqref{proof:th_converge:ineq3}, we have that 
  \begin{equation*}
    \begin{aligned}
      \mp\mbox{sgn}(u_j^t)\epsilon\nabla_{d\u}L\left((d\u)^t\right)\trans\1_j
       & \leq \xi + \lambda^t\epsilon + \frac{\epsilon^2}{2}\1_{j}\trans\nabla_{d\u}^{2}L\1_j \\
       & \leq \xi + \lambda^t\epsilon + \frac{M\epsilon^2}{2},
    \end{aligned}
  \end{equation*}
  which yields
  \begin{equation}\label{proof:th_converge:ineq5}
    \abs{\nabla_{d\u}L\left((d\u)^t\right)\trans\1_j} - \lambda^t \leq \frac{M\epsilon}{2} + \frac{\xi}{\epsilon}~~\text{for $u_j^t=0$}.
  \end{equation}
  Then, combining \eqref{proof:th_converge:ineq4} and \eqref{proof:th_converge:ineq5}, we get
  \begin{equation*}
    \abs{\nabla_{d\u}L\left((d\u)^t\right)\trans\1_j + \delta_j\lambda^t}
    \leq \frac{M\epsilon}{2} + \frac{\xi}{\epsilon}~~\text{for $j=1,\dots,p$}.
  \end{equation*}
Note that we can choose $\delta_j$ appropriately from $-1$ to $1$ when $u_j^t=0$. Therefore, we get the bound that 
  \begin{equation}\label{proof:th_converge:ineq6}
    \begin{aligned}
      \norm{(d\u)^{t*} - (d\u)^t}_2
       & \leq \frac{2}{\mu\norm{\v^t}_2^2}\norm{\nabla_{d\u}L\left((d\u)^t\right) + \lambda^t\bdelta}_2    \\
       & \leq \frac{2\sqrt{p}}{\norm{\v^t}_2^2}\left(\frac{M\epsilon}{2\mu} + \frac{\xi}{\epsilon\mu}\right).
    \end{aligned}
  \end{equation}
  By similar argument, we can get the bound that 
  \begin{equation}\label{proof:th_converge:ineq7}
    \norm{(d\v)^{t*} - (d\v)^{t}}_2 \leq \frac{2\sqrt{q}}{\norm{\u^t}_2^2}\left(\frac{M\epsilon}{2\mu} + \frac{\xi}{\epsilon\mu}\right).
  \end{equation} 
  Then, with $1=\norm{\u^t}_1 \leq \sqrt{p}\norm{\u^t}_2$ and $1=\norm{\v^t}_1\leq \sqrt{q}\norm{\v^t}_2$, combining \eqref{proof:th_converge:ineq6} 
  and \eqref{proof:th_converge:ineq7} yields 
  \begin{equation*}
    \max\left[\norm{(d\u)^{t*}\v^t\trans - d^t\u^t\v^t\trans}_F,~\norm{\u^t(d\v)^{t*}\trans - d^t\u^t\v^t\trans}_F\right] 
    \leq 2\sqrt{pq}\left(\frac{M\epsilon}{2\mu}+\frac{\xi}{\epsilon\mu}\right),
  \end{equation*}
  where we make use of the fact that $\norm{\x\y\trans}_F=\norm{\x}_{2}\norm{\y}_2$ for any column vectors $\x$ and $\y$. This completes the proof.



\bibliographystyle{rss}

\end{document}